\documentclass[colorlinks=true,allcolors=darkblue]{article}

\PassOptionsToPackage{numbers, compress}{natbib}

\usepackage[utf8]{inputenc} 
\usepackage[T1]{fontenc}    
\usepackage{hyperref}       
\usepackage{url}            
\usepackage{booktabs}       
\usepackage{amsfonts}       
\usepackage{nicefrac}       
\usepackage{microtype}      
\usepackage{xcolor}         
\usepackage{mathrsfs}

\usepackage{algorithm}
\usepackage{algorithmic}

\usepackage{jcd}
\usemedgeometry
\usepackage{natbib}
\usepackage{cleveref}
\usepackage{overpic}
\usepackage{enumitem}
\usepackage{tikz}
\usepackage{float}
\usepackage{olm_macros}

\usepackage{xspace}

\title{Federated Asymptotics: a model to compare federated learning algorithms}

\newcommand{\footremember}[2]{%
   \footnote{#2}
    \newcounter{#1}
    \setcounter{#1}{\value{footnote}}%
}
\newcommand{\footrecall}[1]{%
    \footnotemark[\value{#1}]%
}

\author{%
  Gary Cheng\footremember{authorship}{Equal contribution, author order random }\footremember{eedept}{Electrical Engineering Department, Stanford University} \\
  \texttt{chenggar@stanford.edu} 
  \and
  Karan Chadha\footrecall{authorship} \footrecall{eedept}\\
  \texttt{knchadha@stanford.edu} 
    \and
   John Duchi\footrecall{eedept}
   \footnote{Statistics Department, Stanford University}\\
  \texttt{jduchi@stanford.edu} 
}
\crefname{assumption}{Assumption}{Assumption}

\renewcommand{\citet}{\cite}

\begin{document}
\maketitle

\addtocontents{toc}{\protect\setcounter{tocdepth}{0}}
\begin{abstract}
  We propose an asymptotic framework to analyze the performance of
  (personalized) federated learning algorithms. In this new framework, we
  formulate federated learning as a multi-criterion objective, where the
  goal is to minimize each client's loss using information from all of the
  clients. We analyze a linear regression model where, for a given client,
  we may theoretically compare the performance of various algorithms in the
  high-dimensional asymptotic limit. This asymptotic multi-criterion
  approach naturally models the high-dimensional, many-device nature of
  federated learning. These tools make fairly precise predictions about the
  benefits of personalization and information sharing in federated
  scenarios—at least in our (stylized) model—including that Federated
  Averaging with simple client fine-tuning achieves the same asymptotic risk
  as the more intricate meta-learning and proximal-regularized approaches
  and outperforming Federated Averaging without personalization. We evaluate
  these predictions on federated versions of the EMNIST, CIFAR-100,
  Shakespeare, and Stack Overflow datasets, where the experiments
  corroborate the theoretical predictions, suggesting such frameworks may
  provide a useful guide to practical algorithmic development.
\end{abstract}



 \section{Introduction}\label{sec:intro}

In Federated learning (FL), a collection of client machines, or devices,
collect data and coordinate with a central server to fit machine-learned
models, where communication and availability constraints add
challenges~\cite{Kairouz2019advances}.
A natural formulation here, assuming a supervised learning setting, is to
assume that among $m$ distinct clients, each client $i$ has
distribution $P_i$, draws observations $Z \sim P_i$, and wishes to fit a
model---which we represent abstractly as a parameter vector $\theta \in
\Theta$---to minimize a risk, or expected loss, $\risk_i(\theta) \defeq
\E_{P_i}[\ell(\theta; \statrv)]$, where the loss $\ell(\theta; z)$ measures the
performance of $\theta$ on example $\statval$. Thus, at the most abstract level,
the federated learning problem is to solve the multi-criterion problem
\begin{equation}
  \label{eqn:multicriterion}
  \minimize_{\theta_1, \ldots, \theta_m}
  \left(\risk_1(\param_1), \ldots, \risk_m(\param_m)\right).
\end{equation}

At this level, problem~\eqref{eqn:multicriterion} is both trivial---one
should simply minimize each risk $\risk_i$ individually---and impossible, as
no individual machine has enough data locally to effectively minimize
$\risk_i$. Consequently, methods in federated learning typically take
various departures from the multicriterion
objective~\eqref{eqn:multicriterion} to provide more tractable
problems.  Many approaches build off of the empirical risk minimization
principal~\cite{Vapnik92, Vapnik95, HastieTiFr09}, where we seek a single
parameter $\theta$ that does well across all machines and data,
minimizing a (weighted) average loss
\begin{equation}
  \label{eqn:avg-loss}
  \sum_{i = 1}^m p_i
  \risk_i(\theta)
\end{equation}
over $\theta \in \Theta$, where $p \in \R^m_+$ satisfies $p^T \ones =
1$. This ``zero personalization'' approach has the advantage that data is
(relatively) plentiful, and has led to a substantial literature. Much of
this work focuses on developing efficient methods that limit possibly
expensive and unreliable communication between centralized servers and
distributed
devices~\citep{HaddadpourMa19,ReddiChZaGaRuKoKuMc21,McMahanMoRaHaAr17,
  KarimireddyKaMoReStSu20, MohriSiSu19, LiSaZaSaTaSm20}.  Given (i) the
challenges of engineering such large-scale systems, (ii) the success of
large-scale machine learning models without personalization, and (iii) the
plausibility that individual devices have similar distributions $P_i$, the
zero personalization approach is natural.
However, as distributions across individual devices are typically non-identical, it is of interest to develop methods more
closely targeting problem~\eqref{eqn:multicriterion}.

One natural assumption to make is that the optimal client parameters are
``close'' to one another and thus must be ``close'' to the minimizer of the
zero-personalization objective~\eqref{eqn:avg-loss}.  Approaches leveraging
this assumption \citep{DinhTrNg20,SmithChiSaTa17,
  WangMaKiEiBeRa19,FallahMoOz20} regularize client parameters towards the
global parameter. While these methods are intuitive, most focus on showing
convergence rates to local minima. While convergence is important, these
analyses do little to characterize the performance of solutions
attained---to what the methods actually converge.  These issues motivate
our paper.\\


\noindent\textbf{Contributions:}

\begin{enumerate}[leftmargin=*]
\item \textbf{New model} (Sec.~\ref{sec:model}): We propose and analyze a
  (stylized) high-dimensional linear regression model, where, for a given
  client, we can characterize the performance of
  collaborate-then-personalize algorithms in the high-dimensional asymptotic
  limit.
\item \textbf{Precise risk characterization}: We use
  our stylized model to evaluate the asymptotic test loss of several
  procedures.
  These include simple fine-tuned variants of Federated
  Averaging~\cite{McMahanMoRaHaAr17},
  where one learns an average global model~\eqref{eqn:avg-loss}
  and updates once using local data;
  meta-learning variants of federated learning; and
  proximal-regularized personalization in federated learning.
\item \textbf{Precise predictions and experiments}: Our theory makes several
  percise predictions, including that fairly naive methods---fine-tuning
  variants---should perform as well as more sophisticated methods, as well
  as conditions under which federated methods improve upon
  zero-personalization~\eqref{eqn:avg-loss} or zero-collaboration
  methods. To test these predicted behaviors, we perform several experiments
  on federated versions of the EMNIST, CIFAR-100, Shakespeare, and Stack
  Overflow datasets. Perhaps surprisingly, the experiments are
  quite consistent with the behavior the theory predicts.
\end{enumerate}

Our choice to study linear models in the high-dimensional asymptotic setting
(when dimension and samples scale proportionally) takes as motivation a
growing phenomenological approach to research in machine learning, where one
develops simple models that predict (perhaps unexpected) complex behavior in
model-fitting. Such an approach has advantages: by developing simpler
models, one can isolate causative aspects of behavior and make precise
predictions of performance, leveraging these to provide insights in more
complex models.  Consider, for instance, \citet{HastieMoRoTi19}, who show
that the ``double-descent'' phenomenon, where (neural-network) models show
decreasing test loss as model size grows, exists even in linear
regression. In a robust (adversarial) learning setting,
\citet{CarmonRaScLiDu19} use a two-class Gaussian linear discriminant model
to suggest ways that self-supervised training can circumvent hardness
results, using the predictions (on the simplified model) to inform a full
deep training pipeline substantially outperforming (then) state-of-the-art.
\citet{Feldman19} develops clustering models where memorization of data is
necessary for good learning performance, suggesting new models for
understanding generalization. We view our contributions in this intellectual
tradition: using a stylized high-dimensional asymptotics to develop
statistical insights underpinning Federated Learning (FL).  This allows
direct comparison between different FL methods---not between upper bounds,
but actual losses---and serving as a basic framework to motivate new
methodologies in Federated Learning.

\paragraph{Related Work.} 
The tried-and-true method to adapt to new data distributions is fine-tuning
\cite{HowardRu18}. In Federated Learning (FL), this broadly corresponds to
fine-tuning a global model (e.g., from FedAvg) on a user's local data
\cite{WangMaKiEiBeRa19,YuBaSh20,LiHuBeSm21}. While fine-tuning's simplicity
and practical efficacy recommend it, we know of little theoretical analysis.

A major direction in FL is to design personalization-incentivizing
objectives.  \citet{SmithChiSaTa17}, for example, build out of the
literature on multitask and transfer learning~\cite{Caruana97,PanYa09} to
formulate a multi-task learning objective, treating each machine as an
individual task; this and other papers~\cite{FallahMoOz20, MansourMoRoSu20,
  DinhTrNg20} show rates of convergence for optimization methods on these
surrogates.  These methods use the heuristic that personalized, local models
should lie ``close'' to one another, and the authors provide empirical
evidence for their improved performance.  Yet it is not always clear what
conditions are necessary (or sufficient) for these specialized
personalization methods to outperform naive zero collaboration---fully local
training on available data on each individual device---and zero
personalization (averaged) methods.
In a related vein, meta-learning
approaches \cite{FinnAbLe17,FallahMoOz20,JiangKoRuKa19} seek a
global model that can quickly ``adapt'' to local distributions $P_i$, typically
by using a few gradient steps. This generally yields a sophisticated non-convex objective,
making it hard to give even heuristic guarantees, leading authors instead to emphasize worst-case convergence rates
of iterative algorithms to (near) stationary points. 

Other methods of personalization have also been proposed. In contrast to
using a global model to help train the local model, \citet{MansourMoRoSu20}
and~\citet{ZecMoSuLeGi20} use a mixture of global and local models to
incorporate personalized features. \citet{ChenZhLoSu21}, like we
do, propose
evaluating federated algorithms via the
formulation~\eqref{eqn:multicriterion}; they give minimax bounds to
distinguish situations in which zero collaboration and zero personalization
(averaged) methods~\eqref{eqn:avg-loss} are, respectively, worst-case
optimal.




\section{The Linear Model}
\label{sec:model}


We consider a high-dimensional asymptotic model,
where clients solve statistically related linear regression problems,
and each client $i \in [m]$ has a local dataset size $n_i$ smaller than (but
comparable to) the dimension $d$ of the problem. This choice models the
empirical fact that the data on a single client is typically small relative
to model dimension (e.g., even training the last layer of a deep neural network). 

More concretely, each client $i \in [m]$ will use an overparameterized linear regression
problem to recover an unknown parameter $\param_i\opt \in \R^d$.
Client $i$ has $n_i$ i.i.d.\ observations $(\bx_{i, k},
y_{i,k}) \in \R^d \times \R$,
\begin{align*}
  y_{i, k} = \bx_{i, k}^T \param_i\opt + \xi_{i, k},
  ~~~~
  \bx_{i, k} \simiid P^i_{\bx}
  ~~ \mbox{and} ~~
  \xi_{i,k} \simiid P^i_{\xi}.
\end{align*}
We make the routine assumption that the features are  centered,
with $\E[\bx_{i,k}] = 0$ and $\cov(\bx_{i,k}) = \Sigma_i$.
We also assume that the
noise is centered with finite variance, i.e., $\E[\xi_{i,k}] = 0$ and
$\var(\xi_{i,k}) = \sigma_i^2$. For convenience, we let $X_i \in
\R^{n_i \times d}$ and $\by_i \in \R^{n_i}$ denote client $i$'s data, and
$X \coloneqq [X_1^T,\dots,X_m^T]^T$. We also let $N \coloneqq \summach
n_j$.

A prior $P^i_\param$ on the parameter $\param_i\opt$
relates tasks on each client, where conditional on
$\param_0\opt$,
$\param_i\opt$ is supported on $r_i\sphere^{d-1} + \param_0\opt$---the sphere
of radius $r_i$ (bounded by a constant for all $i\in[m]$) centered at $\param_0\opt$---with $\E[\param_i\opt] = \param_0\opt$. 
The variation between clients is captured by differences in $r_i$ (label shift) and $\Sigma_i$ (covariate shift), while the similarity is captured by the shared center $\param_0\opt$.
Intuitively, data from client $j$ is useful to client $i$
as it provides information on the possible location of
$\param_0\opt$. Lastly, we assume that the distributions of
$\bx,\param\opt$, and $\xi$ are independent of each other and across
clients.

Every client $i$ seeks to minimize its local population loss---the
squared prediction error of a new sample $\bx_{i, 0}$ independent of the
training set---conditioned on $X$. The
sample loss is $\risksamp(\param;(\bx,y)) = (\bx^T\param - y)^2 -
\sigma_i^2$, giving per client test loss
\begin{align*}
    \risk_i(\hat{\param}_i \mid X) &:= \E[(\bx_{i, 0}^T \hat{\param}_i - \bx_{i, 0}^T \param_i\opt)^2\mid X] \\
    & = \E\Big[\normbig{\hat{\param}_i - \param_i\opt}_{\Sigma_i}^2 \mid X\Big]
\end{align*}
where the expectation is taken over $(\bx_{i, 0},\param_i\opt, \xi_i) \sim
P_{\bx}^i \times P_{\param\opt}^i \times P_{\xi}^i$ and $\norm{x}_\Sigma^2 = x^T\Sigma x$.  It is essential here
that we focus on per client performance: the ultimate goal is to improve
performance on any given client, as per \cref{eqn:multicriterion}.  For analysis
purposes, we often consider the equivalent bias-variance decomposition
\begin{align}\label{eqn:risk}
    \risk_i(\hat{\param}_i | X) = \underbrace{\norm{\E[\hat{\param}_i|X] - \param\opt}_{\Sigma_i}^2}_{B_i(\hat{\param}_i|X)} + \underbrace{\tr(\cov(\hat{\param}_i | X)\Sigma_i)}_{V_i(\hat{\param}_i|X)}.
\end{align}
Our main asymptotic assumption, which captures the high-dimensional and many-device nature
central to modern federated learning problems, follows:
\begin{assumption} \label{ass:asym}
  As $m \to \infty$, both $d = d(m)\to \infty$ and $n_j = n_j(m) \to \infty$ for $j
  \in [m]$, and $\lim_{m}\frac{d}{n_j} =
  \gamma_j $. Moreover,
  $1 < \gammamin \le \lim_m \inf_{j \in [m]}
  \frac{d}{n_j} \le \lim_m \sup_{j \in [m]} \frac{d}{n_j}
  \le \gammamax < \infty$.
\end{assumption}
\noindent
Importantly, individual devices are overparameterized: we always have
$\gamma_j > 1$, as is common, when the dimension of models is large relative
to local sample sizes, but may be smaller than the (full) sample. Intuitively, $\gamma_j$ captures the degree of overparameterization of the network for user $j$. 
We also require control of the eigenspectrum of our data~\citep[cf.][Assumption
  1]{HastieMoRoTi19}.

\begin{definition}\label{def:emp-dis}
  The \emph{empirical distribution of the eigenvalues}
    of $\Sigma$ is the function $\mu(\cdot;\Sigma) : \R \to \R_+$ with
\begin{align}\label{eqn:emp-dis}
    \mu(s; \Sigma) \coloneqq \frac{1}{d} \sum_{j=1}^d \bindic{s \geq s_i},
\end{align}
where $s_1 \ge s_2 \ge \dots \ge s_d$ are the eigenvalues of $\Sigma$.
\end{definition}

\begin{assumption} \label{ass:cov}
  For each user $i$, data $\bx \sim P^i_\bx$ have the form $\bx
  = \Sigma_i^\half \bz$. For some $q > 2$, $\moment{q} < \infty$, and $M
  < \infty$,
  \begin{enumerate}[label = (\alph*)]
  \item The vector $\bz \in \R^d$ has independent entries with $\E[z_i] =
    0$, $\E[z_i^2] = 1$, and $\E[|z_i|^{2q}] \le \moment{q} < \infty$
  \item $s_1 = \opnorm{\Sigma_i} \le \covop$, $s_d  = \lambda_{\min}(\Sigma_i) \ge 1/\covop$, and $\int s^{-1} d\mu(s;\Sigma_i) < M$. 
      \item $\mu(\cdot; \Sigma_i)$ converges weakly to $\nu_i$ 
  \end{enumerate}
\end{assumption}
\noindent
These conditions are standard, guaranteeing sufficient
moments for convergence of covariance estimates,
that the eigenvalues of $\Sigma_i$ do not accumulate near 0, and a mode of
convergence for the spectrum of $\Sigma_i$.

 \section{Locally fine-tuning a global solution}
  
In this section, we describe and analyze fine-tuning algorithms that use the
FedAvg solution (a minimizer of the objective~\eqref{eqn:avg-loss}) as a
warm start to find personalized models. We compare the test loss of these
algorithms with naive, zero personalization and zero collaboration
approaches. Among other things, we show that a ridge-regularized locally
fine-tuned method outperforms the other methods.


\label{sec:pers-fed-avg}
\subsection{Fine-tuned Federted Averaging (\ftfa)}
\label{sec:ftfa}

Fine-tuned Federated Averaging (\ftfa) approximates minimizing the
multi-criterion loss~\eqref{eqn:multicriterion} using a two-step procedure
in \Cref{alg:ftfa} (see \Cref{app:algs} for detailed pseudocode). Let
$\sampset_i$ denote client $i$'s sample. The idea is to replace
the local expected risks $\risk_i$ in~\eqref{eqn:avg-loss} with
the local empirical risks
\begin{equation*}
  \riskhat_i(\param) \coloneqq
  \frac{1}{n_i}\sum_{\statval \in \sampset_i} \risksamp(\param; \statval),
\end{equation*}
using the FedAvg solution $\hparamsolfa{0}$ as a warm-start for local
training in the second step. Intuitively, \ftfa\ interpolates between zero
collaboration and zero personalization algorithms.  Each client $i$ can run
this local training phase independently of (and in parallel with) all
others, as the data is fully local; this separation makes \ftfa\ essentially
no more expensive than Federated Averaging.

\begin{algorithm}[ht]
  \caption{\ftfa \& \rtfa (details in appendix)\label{alg:ftfa}}
\begin{enumerate}
\item The server coordinates (e.g.\ using FedAvg)
  to find a global model using
  data from all clients, solving
  \begin{align}\label{eqn:ftfa-obj-gen}
    \hparamsolfa{0} = \argmin_{\param} \summach p_j\riskhat_j(\param),
  \end{align}
  where $p \in \R^m_+$ are weights satisfying $\summach p_j = 1$. The server
  broadcasts $\hparamsolfa{0}$ to all clients.  \stepcounter{enumi}
\begin{enumerate}[leftmargin=0cm]
\item \textbf{FTFA:} Client $i$ optimizes
  its risk $\riskhat_i$ using 
  a first-order method initialized at $\hparamsolfa{0}$,
  returning model $\hparamsolfa{i}$.
      \setlength\itemsep{1em}
    \item \textbf{RTFA:} Client $i$ minimizes a regularized empirical
      risk to return model
    \begin{align*}
        \hparamsolrt{i} = \argmin_{\param} \riskhat_i(\param) + \frac{\lambda}{2}\ltwo{\param - \hparamsolfa{0}}^2.
    \end{align*}
\end{enumerate}
\end{enumerate}
\end{algorithm}


For the linear model in \Cref{sec:model}, \ftfa\ first minimizes the average
weighted loss $\summach p_j\frac{1}{2n_j}\ltwos{X_j\param - \by_j}^2$ over
all clients.  As the local linear regression problem to minimize $\ltwo{X_i
  \param - \by_i}^2$ is overparameterized, first-order methods on it
(including the stochastic gradient method) correspond to solving minimum
$\ell_2$-norm interpolation problems~\citep[Thm.~1]{GunasekarLeSoSr18}.
Thus, when performed to convergence (no matter the first-order
method), \ftfa\ is equivalent to the
two step procedure
\begin{align}
    \hparamsolfa{0} &= \argmin_{\param}\summach p_j\frac{1}{2n_j}\ltwo{X_j\param - \by_j}^2\label{eqn:param-fedavg}\\
    \hparamsolfa{i} &= \argmin_{\param}
    \left\{\ltwobig{\hparamsolfa{0} - \param}
    ~~ \mbox{s.t.}~~ X_i \param = \by_i\right\}, \label{eqn:param-perfedavg}
\end{align}
outputting model $\hparamsolfa{i}$ for client $i$. Before
giving our result, we make an additional assumption regarding the
asymptotics of the number of clients and the dimension of the data.

\begin{assumption}\label{ass:client-growth}
  For a constant $c$ and $q > 2$, $(\log d)^{cq}\summach p_j^{q/2 + 1} n_j
  \rightarrow 0$ as $m,d,n_j \rightarrow \infty$.
\end{assumption}
In Assumption~\ref{ass:client-growth}, $p_j$ is the weight associated with
the loss of $j^{\rm th}$ client when finding the global model using
federated averaging. To ground the assumption, consider two particular cases
of interest: (i) $p_j = 1/m$, when every client is weighted equally, and
(ii) $p_j = n_j / N$, when each data point is weighted equally.  When $p_j =
1/m$, we have $(\log d)^{cq}\summach p_j^{q/2 + 1} n_j =
\frac{N}{m}\frac{(\log d)^{cq}}{m^{q/2}}.$ When $p_j = n_j/N$, using
\Cref{ass:asym} we have $(\log d)^{cq}\summach p_j^{q/2 + 1} n_j = (\log
d)^{cq} \summach \frac{n_j^{q/2 + 2}}{N^{q/2 + 1}} \le
(\frac{\gammamax}{\gammamin})^{q/2 + 1}\frac{N}{m}\frac{(\log
  d)^{cq}}{m^{q/2}}.$ Thus, in both the cases, ignoring the polylog factors,
if we have $\frac{N}{m}\frac{(\log d)^{cq}}{m^{q/2}} \rightarrow 0$ i.e.,
$m^{q/2}$ grows faster than the average client sample size, $N/m$, then
\Cref{ass:client-growth} holds.
With these assumptions defined, we are able to compute the asymptotic test loss of \ftfa.

\begin{theorem}\label{thm:perfedavg}
  Consider the observation model in Sec.~\ref{sec:model} and the estimator
  $\hparamsolfa{i}$ in~\eqref{eqn:param-perfedavg}. Let \Cref{ass:asym}
  hold, and let Assumptions~\ref{ass:cov} and~\ref{ass:client-growth} hold
  with $c = 2$ and $q > 2$. Additionally, assume that for each $m$ and
  $j \in [m]$, $\opnorms{\E[\scov_j^2]} \leq \matmoment{2}$, where
  $\matmoment{2} < \infty$. Then for client $i$, the asymptotic
  prediction bias and variance of \ftfa\ are
\begin{align*}
    \lim_{m \rightarrow\infty} B_i(\hparamsolfa{i}|X) & \stackrel{p}= \lim_{m \rightarrow\infty} \norm{\Pi_i [\param_0\opt - \param_i\opt]}_{\Sigma_i}^2  \\
    \lim_{m \rightarrow\infty} V_i(\hparamsolfa{i}|X)  & \stackrel{p}= \lim_{m \rightarrow\infty} \frac{\sigma_i^2}{n_i} \tr(\scov_i\pinv\Sigma_i),
\end{align*}
where $\Pi_i \coloneqq I - \scov_i\pinv \scov_i$ and $\lsigi{z}^2 \coloneqq z^T \Sigma_i z$. The exact expressions of these limits for general choice $\Sigma$ in the implicit form can be found in the appendix. For the special case when $\Sigma_i = I$, the closed form limits are
\begin{align*}
    B_i(\hparamsolfa{i}|X)  \cp r_i^2\left(1 - \frac{1}{\gamma_i}\right) \qquad
    V_i(\hparamsolfa{i}|X)  \cp \frac{\sigma_i^2}{\gamma_i - 1},
\end{align*}
where $\cp$ denotes convergence in probability.
\end{theorem}



\subsection{Ridge-tuned FedAvg (\rtfa)}

Minimum-norm results provide insight into the behavior of popular algorithms
including SGM and mirror descent. Having said that, we can also analyze
ridge penalized versions of \ftfa. In this algorithm, the server finds the
same global model as \ftfa, but each client uses a regularized objective to
find a local personalized model as in 2b of \Cref{alg:ftfa}. More
concretely, in the linear regression setup, for appropriately chosen step
size and as the number of steps taken goes to infinity, this corresponds to
the two step procedure with the first step~\eqref{eqn:param-fedavg}
and second step
\begin{equation}
  \hparamsolrt{i}
  = \argmin_{\param} \frac{1}{2n_i} \ltwo{ X_i \param - \by_i}^2   +\frac{\lambda}{2}\ltwobig{\hparamsolfa{0} - \param}^2,  \label{eqn:param-perridge}
\end{equation}
where \rtfa\ outputs the model $\hparamsolrt{i}$ for client $i$.
Under the same assumptions as
\Cref{thm:perfedavg}, we can again calculate the asymptotic test loss.

\begin{theorem}\label{thm:perridge}
  Let the conditions of Theorem~\ref{thm:perfedavg}
  hold. Then for client $i$,
  the asymptotic prediction bias and variance of \rtfa\ are
\begin{align*}
    \lim_{m \rightarrow\infty} B_i(\hparamsolrt{i}|X)& \stackrel{p}= \lim_{m \rightarrow\infty} \lambda^2 \lsigi{(\scov_i + \lambda I)\inv (\param_0\opt - \param_i\opt)}^2 \\
    \lim_{m \rightarrow\infty} V_i(\hparamsolrt{i}|X) & \stackrel{p}= \lim_{m \rightarrow\infty} \frac{\sigma_i^2}{n_i} \tr(\Sigma_i\scov_i(\lambda I + \scov_i)^{-2}),
\end{align*}
The exact expressions of these limits for general choice $\Sigma$ in the implicit form can be found in the appendix. For the special case when $\Sigma_i = I$, the closed form limits are
\begin{align*}
    B_i(\hparamsolrt{i}|X ) &\cp r_i^2\lambda^2 m_i'(-\lambda) 
    \\
    V_i(\hparamsolrt{i}|X)  &\cp \sigma_i^2 \gamma_i (m_i(-\lambda) - \lambda m_i'(-\lambda)),
\end{align*}
where $m_i(z)$ is the Stieltjes transform of the limiting spectral measure
$\nu_i$ of the covariance matrix $\Sigma_i$. When $\cov(\bx_{j,k}) = I$,
$m_i(z)$ has the closed form expression
$m_i(z) = (1 - \gamma_i - z - \sqrt{(1 - \gamma_i - z)^2 - 4\gamma_i z})/(2\gamma_i z)$.
For each client $i\in[m]$, when $\lambda$ is set to be the minimizing value  $\lambda_i\opt = \sigma_i^2 \gamma_i/r_i^2$, the expression simplifies to $\risk_i(\hparamsol{i}{R}{\lambda_i\opt}|X) \rightarrow \sigma_i^2 \gamma_i m_i(-\lambda_i\opt)$.
\end{theorem}

With the optimal choice of hyperparameter $\lambda$, \rtfa\ has lower test
loss than \ftfa; indeed, in overparameterized linear regression, the ridge
solution with regularization $\lambda \rightarrow 0$ converges to the
minimum $\ell_2$-norm interpolant~\eqref{eqn:param-perfedavg}.

\subsection{Comparison to Naive Estimators}

Three natural baselines to which we may compare \ftfa\ and \rtfa\ are the
zero personalization estimator $\hparamsolfa{0}$, the zero collaboration
estimator
\begin{align}
    \hparam_{i}^{N} = \argmin_{\param} \ltwo{\param} \quad s.t.\quad X_i \param = \by_i \label{eqn:param-naive},
\end{align}
and the ridge-penalized, zero-collaboration estimator
\begin{align}
    \hparam_i^N(\lambda) = \argmin_{\param} \frac{1}{2 n_i}\ltwo{ X_i \param - \by_i}^2 + \frac{\lambda}{2} \ltwo{\param}^2. \label{eqn:param-ridge-naive}
\end{align}
As with \ftfa\ and \rtfa, we can compute the asymptotic test loss explicitly
for each. We provide expressions only for identity covariance
case for clarity, similar results and comparisons hold for
general covariance matrices.

\begin{corollary}\label{cor:naive}
  Let the conditions of Theorem~\ref{thm:perfedavg} hold and $\Sigma_i = I$
  for $i$. Then for client $i$,
  \begin{align*}
    B_i(\hparamsolfa{0}|X )  \cp r_i^2 ~~~ \mbox{and} ~~~
    V_i(\hparamsolfa{0}|X)  \cp 0.
  \end{align*}
  Consider the estimator $\hparam_i^N$ defined by \cref{eqn:param-naive}. In
  addition to the above conditions, suppose that $\param_i\opt$ is drawn
  such that $\ltwo{\param_i\opt} = \rho_i$ is constant with respect to
  $m$. Further assume that for some $q > 2$, for $j\in[m]$ and $k \in
  [n_j]$, and $l \in [d]$, $\E[(\bx_{j,k})_l^{2q}] \le \moment{q} <
  \infty$. Then
  \begin{align*}
    B_i(\hparam_i^N|X )  \cp \rho_i^2\left(1 - \frac{1}{\gamma_i} \right)
    ~~ \mbox{and} ~~
    V_i(\hparam_i^N|X)  \cp \frac{\sigma_i^2}{\gamma_i - 1}.
  \end{align*}
  Consider the estimator $\hparam_i^N(\lambda)$ defined by
  \cref{eqn:param-ridge-naive}. Then Under the preceding conditions,
\begin{align*}
    B_i(\hparam_i^N(\lambda)|X )  &\cp \rho_i^2 \lambda^2 m_i'(-\lambda)
    \\
    V_i(\hparam_i^N(\lambda)|X)  &\cp \sigma_i^2 \gamma_i(m_i(-\lambda) - \lambda m_i'(-\lambda)).
\end{align*}
Moreover, if $\lambda$ is set to be the loss-minimizer $\lambda_i\opt = \sigma_i^2\gamma_i/\rho_i^2$, then $\risk_i(\hparam_i^N(\lambda_i\opt); \param\opt|X) \cp \sigma_i^2 \gamma_i m_i(-\lambda_i\opt)$.

\end{corollary}

Key to these results are the differences between the radii $r_i^2 =
\ltwo{\theta_i\opt - \theta_0\opt}^2$ and $\rho_i = \ltwo{\theta_i\opt}^2$,
where $r_i^2 \le \rho_i^2$, and their relationship to the other problem
parameters.  First, it is straightforward to see that \ftfa\ outperforms
FedAvg, $\hparamsolfa{0}$, if and only if $\sigma_i^2 < r_i^2(\gamma_i -
1)/\gamma_i$. This makes intuitive sense: if the noise is too large, then
local tuning is fitting mostly to noise. \ftfa\ always outperforms the
zero-collaboration estimator $\hparam_i^N$, as $\rho_i \geq r_i$, and
the difference $\rho_i^2 - r_i^2$ governs the gap between
collaborative and non-collaborative solutions. This remains true for the
ridge-based solutions: a first-order expansion comparing
Theorem~\ref{thm:perridge} and Corollary~\ref{cor:naive} shows that for
$\rho_i$ near $r_i$, we have
\begin{align*}
  \lefteqn{\risk_i(\hparam_i^R(\sigma_i^2 \gamma_i / r_i^2) \mid X)
    - \risk_i(\hparam_i^N(\sigma_i^2 \gamma_i/\rho_i^2) \mid X)} \\
  & \qquad\qquad\qquad \qquad ~ = C \cdot (\rho_i^2 - r_i^2) + o(\rho_i^2 - r_i^2),
\end{align*}
where $C$ depends on all the problem parameters.

With appropriate regularization $\lambda$, \rtfa mitigates the weaknesses of
\ftfa. Thus, formally, we may show that \rtfa\ with the optimal
hyperparameter always outperforms the zero-personalization estimator
$\hparamsolfa{0}$ (see the appendices). Furthermore, since $\rho_i \geq
r_i$, its straightforward to see that \rtfa\ outperforms ridgeless
zero-collaboration estimator $\hparam_i^N$, and the ridge-regularized
zero-collaboration estimator $\hparam_i^N(\lambda\opt)$ as well.

 \section{Meta learning and Proximal Regularized Algorithms}
\label{sec:maml-prox}

The fine-tuning procedures in the previous section provide a (perhaps) naive
baseline, so we consider a few alternative federated learning procedures,
both of which highlight the advantages of the high-dimensional asymptotics
in its ability to predict performance. While we cannot survey the numerous
procedures in FL, we pick two we consider representative: the first adapting
meta learning~\cite{FallahMoOz20} and the second using a
proximal-regularized approach~\cite{DinhTrNg20}.  In both cases, the
researchers develop convergence rates for their methods (in the former case,
to stationary points), but no results on the predictive performance or their
\emph{statistical} behavior exists.  We develop these in this section,
showing that these more sophisticated approaches perform no better, in
our asymptotic framework, than the \ftfa\ and \rtfa\ algorithms we outline in
\Cref{sec:pers-fed-avg}.

\subsection{Model-Agnostic Meta-Learning}

Model-Agnostic Meta-Learning (MAML) \cite{FinnAbLe17} was learns models that
adapt to related tasks by minimizing an empirical loss augmented evaluated
not at a given parameter $\param$ but at a ``one-step-updated'' parameter
$\param - \alpha \nabla \risk(\param)$, representing one-shot learning.
\citet{FallahMoOz20}, contrasting this MAML approach to the standard
averaging objectives~\eqref{eqn:avg-loss}, adapt MAML to the federated
setting, developing a method we term \mamlfl. We describe their two step
procedure in \Cref{alg:maml} (see \Cref{app:algs} for detailed
pseudocode). \Cref{alg:maml} has two variants \cite{FallahMoOz20}; in one,
the Hessian term is ignored, and in the other, the Hessian is approximated
using finite differences. \cite{FallahMoOz20} showed that these these
algorithms converge to a stationary point of \cref{eqn:maml-obj-gen} (with
$p_j = 1/m$) for general non-convex smooth functions.

\begin{algorithm}[ht]
  \caption{\mamlfl\ (details in Appendix~\ref{app:algs})\label{alg:maml}}
\begin{enumerate}[leftmargin=0.5cm]
    \item Server and clients coordinate to (approximately) solve
    \begin{align}\label{eqn:maml-obj-gen}
    \hparamsol{0}{M}{\alpha} = \argmin_{\param} \summach p_j\riskhat_j(\param - \alpha \grad \riskhat_j(\param)),
    \end{align}
    where $p_j \in (0,1)$ are weights such that $\summach p_j = 1$ and
    $\alpha$ denotes stepsize. Server broadcasts global model
    $\hparamsol{0}{M}{\alpha}$ to clients.
    \item Client $i$ learns a model $\hparamsol{i}{M}{\alpha}$ by optimizing
      its empirical risk, $\riskhat_i(\cdot)$, using SGM initialized at
      $\hparamsol{0}{M}{\alpha}$
    \end{enumerate}
\end{algorithm}

 In our linear model, for appropriately chosen hyperparameters and as the number of steps taken goes to infinity, this personalization method corresponds to the following two step procedure:
\begin{align}
  &\hparamsol{0}{M}{\alpha} \label{eqn:param-maml} \\
    &= \argmin_\param \summach  \frac{p_j}{2n_j} \ltwo{X_j\left[\param - \frac{\alpha}{n_j} X_j^T(X_j \param - \by_j)\right] -\by_j}^2 \nonumber \\
    &\hparamsol{i}{M}{\alpha} = \argmin_{\param}  \ltwo{\hparamsol{0}{M}{\alpha} - \param} \; \; s.t. \; \; X_i \param = \by_i \label{eqn:param-permaml}
\end{align}
As in Section~\ref{sec:ftfa}, we assume that the client
model in step~2 of Alg.~\ref{alg:maml} has fully converged;
any convergent first-order method converges to the
minimum norm interpolant~\eqref{eqn:param-permaml}.
The representations~\eqref{eqn:param-maml} and~\eqref{eqn:param-permaml}
allow us to analyze the
test loss of the \mamlfl personalization scheme in our asymptotic
framework.
\begin{theorem}\label{thm:permaml}
  Consider the observation model in Section~\ref{sec:model} and the
  estimator $\hparamsolmaml{i}$ in~\eqref{eqn:param-permaml}. Let
  \Cref{ass:asym} hold, \Cref{ass:cov} hold with $q = 3v$ where $v > 2$, and
  \Cref{ass:client-growth} hold with $c = 5$ and $q = v$. Additionally,
  assume that for each $m$ and all $j \in [m]$,
  $\lambda_{\min{}}(\E[\scov_j(I - \alpha\scov_j)^2]) \ge \lambda_0 > 0$ and
  $\opnorms{\E[\scov_j^6]} \leq \matmoment{6} < \infty$. Then
  for client $i$, the asymptotic
  prediction bias and variance of \mamlfl\ are
  \begin{align*}
    \lim_{m \rightarrow\infty} B_i(\hparamsolmaml{i}|X)& \stackrel{p}= \lim_{m \rightarrow\infty} \norm{\Pi_i [\param_0\opt - \param_i\opt]}_{\Sigma_i}^2  \\
    \lim_{m \rightarrow\infty} V_i(\hparamsolmaml{i}|X) & \stackrel{p}= \lim_{m \rightarrow\infty} \frac{\sigma_i^2}{n_i} \tr(\scov_i\pinv\Sigma_i),
  \end{align*}
  where $\Pi_i \coloneqq I - \scov_i\pinv \scov_i$ and $\lsigi{z}^2
  \coloneqq z^T \Sigma_i z$. The exact expressions of these limits for
  general choice $\Sigma$ in the implicit form are  in the
  appendices. For the special case that $\Sigma_i = I$, the closed form limits
  are
  \begin{align*}
    B_i(\hparamsolmaml{i}|X) & \cp r_i^2\left(1 - \frac{1}{\gamma_i}\right)  \\
    V_i(\hparamsolmaml{i}|X) & \cp \frac{\sigma_i^2}{\gamma_i - 1}.
  \end{align*}
\end{theorem}

In short, the asymptotic test risk of \mamlfl\ matches that of
\ftfa\ (\Cref{thm:perfedavg}). In general, the
\mamlfl\ objective~\eqref{eqn:maml-obj-gen} is typically non-convex even
when $\riskhat_j$ is convex, making convergence subtle. Even ignoring
convexity, the inclusion of a derivative term in the objective can make the
standard smoothness conditions~\citep{Nesterov04} upon which convergence
analyses (and algorithms) repose fail to hold.  Additionally, computing
gradients of the \mamlfl objective~\eqref{eqn:maml-obj-gen} requires
potentially expensive second-order derivative computations or careful
approximations to these, making optimization more challenging and expensive
irrespective of convexity.  We provide more discussion in the
appendices. Theorem~\ref{thm:permaml} thus suggests that one might be
circumspect about choosing \mamlfl or similar algorithms over simpler
baselines that do not require such complexity in optimization.


\begin{remark}
  The algorithm \citet{FallahMoOz20} propose performs only a single
  stochastic gradient step for personalization, which is distinct from our
  analyzed procedure~\eqref{eqn:param-permaml}, as it is essentially
  equivalent to running SGM until convergence from the initialization
  $\hparamsolmaml{0}{M}{\alpha}$ (see step 2 of \Cref{alg:maml}).  We find
  two main justifications for this choice: first, experimental work of
  \cite{JiangKoRuKa19}, in addition to our own experiments (see
  Figures~\ref{fig:pers-femnist} and~\ref{fig:pers-cifar}) empirically
  suggest that the more (stochastic gradient) steps of personalization, the
  better performance we expect. Furthermore, as we mention above earlier,
  performing personalization SGM steps locally, in parallel, and
  asynchronously is no more expensive than running the first step of
  \Cref{alg:maml}.  This guaranteed of convergence also presents a fair
  point of comparison between the algorithms we consider.
\end{remark}

\subsection{Proximal-Regularized Approach}

Instead of a sequential, fine-tuning approach, an alternative approach to
personalization involves jointly optimizing global and
local parameters. In this vein,
\citet{DinhTrNg20} propose the pFedMe algorithm (whose details
we provide in the appendices) to
solve the following coupled optimization problem to find personalized models
for each client:
\begin{align*}
    & \left(\hparamsol{0}{P}{\lambda},\hparamsol{1}{P}{\lambda},\hdots,\hparamsol{m}{P}{\lambda}\right) = \\ 
    & \argmin_{\param_0, \param_1,\dots, \param_m} \summach p_j\left( \riskhat_i(\param_j) + \frac{\lambda}{2}\ltwo{\param_j - \param_0}^2\right).
\end{align*}
The proximal penalty encourages the local models $\param_i$ to be close to one another. In our linear model, for appropriately chosen hyperparameters and as the number of steps taken goes to infinity, the proposed optimization problem simplifies to
\begin{align}
  \label{eqn:param-prox}
  &
  \left(\hparamsol{0}{P}{\lambda},\hparamsol{1}{P}{\lambda},\hdots,\hparamsol{m}{P}{\lambda}\right) = \\
  &\argmin_{\param_0, \param_1,\dots, \param_m} \summach p_j \left(\frac{1}{2n_j}\ltwo{X_j \param_j - \by_j}^2 + \frac{\lambda}{2}\ltwo{\param_j - \param_0}^2\right),
  \nonumber
\end{align}
where $\hparamsol{0}{P}{\lambda}$ denotes the global model and
$\hparamsol{i}{P}{\lambda}$ denote the local models. We can again use our
asymptotic framework to analyze the test loss of this scheme. For
this result, we use an additional condition on
$\sup_{j\in[m]}\P(\lambda_{\max}(\scov_j) > R)$ that gives us uniform
control over the eigenvalues of all the users.
\begin{theorem}\label{thm:prox}
  Consider the observation model n \cref{sec:model} and the estimator
  $\hparamsol{i}{P}{\lambda}$ in \eqref{eqn:param-prox}. Let \Cref{ass:asym}
  hold, and let Assumptions~\ref{ass:client-growth} and~\ref{ass:cov} hold
  with $c = 2$ and the same $q > 2$. Additionally, assume that
  $\E[\opnorms{\scov_j^2}] \leq \tau_3 < \infty$. Further suppose that there
  exists $R \geq M$ such that $\limsup_{m \rightarrow \infty} \sup_{j \in
    [m]} \P(\lambda_{\max{}}(\scov_j) > R) \leq \frac{1}{16M^2
    \tau_3}$. Then for client $i$, the asymptotic prediction bias and
  variance of \pfedme\ are
  \begin{align*}
    \lim_{m \rightarrow\infty} B_i(\hparamsol{i}{P}{\lambda}|X)& \stackrel{p}= \lim_{m \rightarrow\infty} \lambda^2 \lsigi{(\scov_i + \lambda I)\inv (\param_0\opt - \param_i\opt)}^2 \\
    \lim_{m \rightarrow\infty} V_i(\hparamsol{i}{P}{\lambda}|X)&  \stackrel{p}= \lim_{m \rightarrow\infty} \frac{\sigma_i^2}{n_i} \tr(\Sigma_i\scov_i(\lambda I + \scov_i)^{-2}),
  \end{align*}
  See the appendices for exact expressions of these limits for general
  $\Sigma$. For the special case that $\Sigma_i = I$, the limits
  are
  \begin{align*}
    B_i(\hparamsol{i}{P}{\lambda}|X ) &\cp r_i^2\lambda^2 m_i'(-\lambda) \\
    V_i(\hparamsol{i}{P}{\lambda}|X) &\cp \sigma_i^2 \gamma (m_i(-\lambda) - \lambda m_i'(-\lambda)),
  \end{align*}
  where $m_i(z)$ is as in \Cref{thm:perridge}.  For each client
  $i\in[m]$, the minimizing $\lambda$ is $\lambda_i\opt =
  \sigma_i^2 \gamma_i/r_i^2$ and
  $\risk_i(\hparamsol{i}{P}{\lambda_i\opt}|X) \rightarrow \sigma_i^2 \gamma_i
  m_i(-\lambda_i\opt)$.
\end{theorem}

The asymptotic test loss of the proximal-regularized approach is thus
identical to the locally ridge-regularized (\rtfa) solution; see
\Cref{thm:perridge}. \citet{DinhTrNg20}'s algorithm to optimize
\cref{eqn:param-prox} is sensitive to hyperparameter choice, meaning
significant hyperparameter tuning may be needed for good performance and
even convergence of the method (of course, both methods do require tuning
$\lambda$). Moreover, a local update step in \pfedme requires approximately
solving a proximal-regularized optimization problem, as opposed to taking a
single stochastic gradient step. This can make \pfedme\ much more
computationally expensive depending on the properties of $\riskhat$. This is
not to dismiss more complex proximal-type algorithms, but only to say that,
at least in our analytical framework, simpler and embarassingly
parallelizable procedures (\rtfa in this case) may suffice to capture the
advantages of a proximal-regularized scheme.


 \section{Experiments}
\label{sec:expts}
\begin{figure*}[t]
\begin{minipage}[c]{0.3\linewidth}
\begin{overpic}[width=\linewidth]{
      		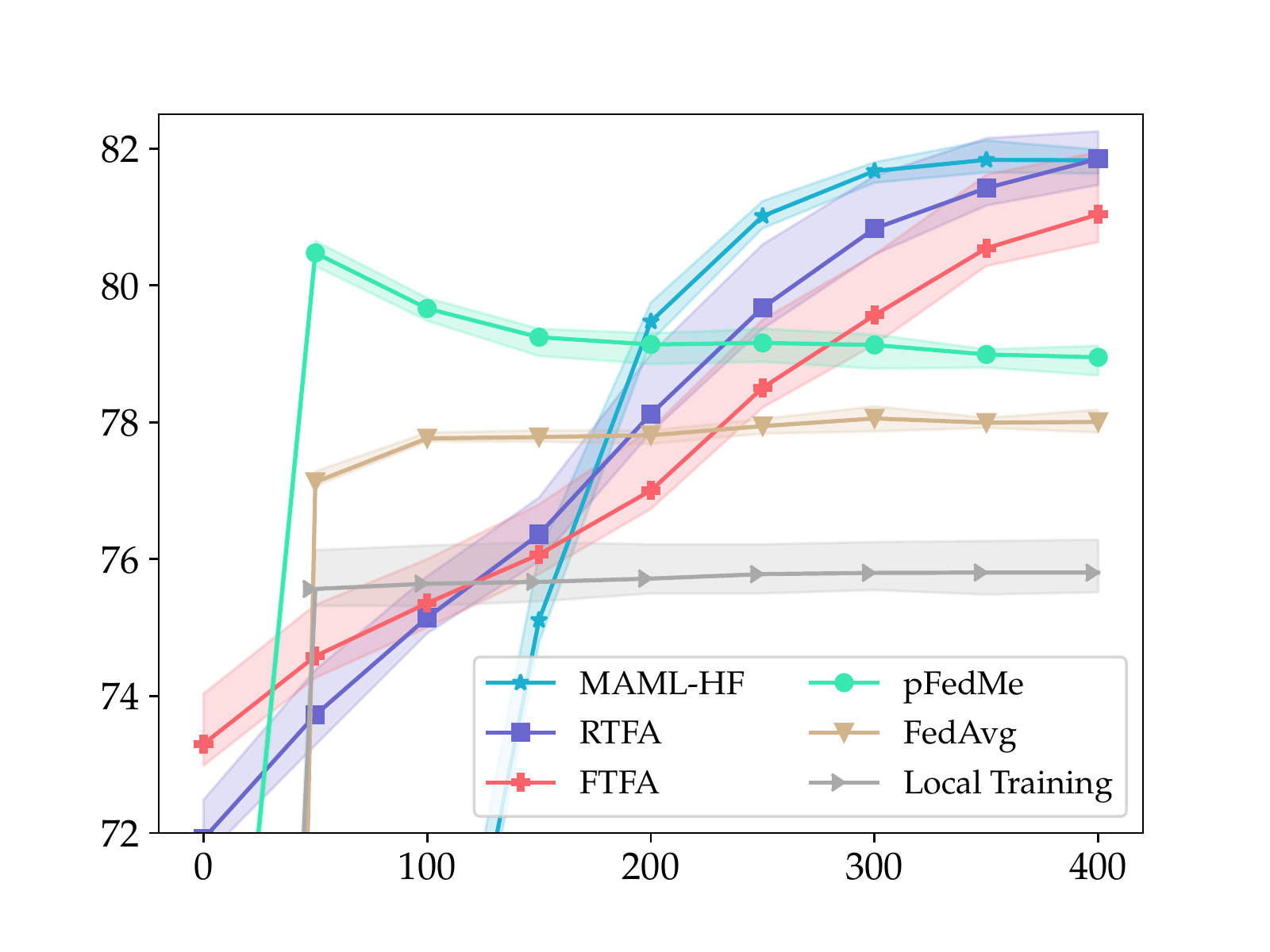}
      		\put(-2,24){
          \rotatebox{90}{{\scriptsize Test Accuracy}}}
        \put(25,0){{\scriptsize Communication Rounds}}
\end{overpic}\caption{\small CIFAR-100. Best-average-worst intervals created from different train-val splits.}
\label{fig:cifar}
\end{minipage}
\hfill
\begin{minipage}[c]{0.3\linewidth}
\begin{overpic}[width=\linewidth]{
      		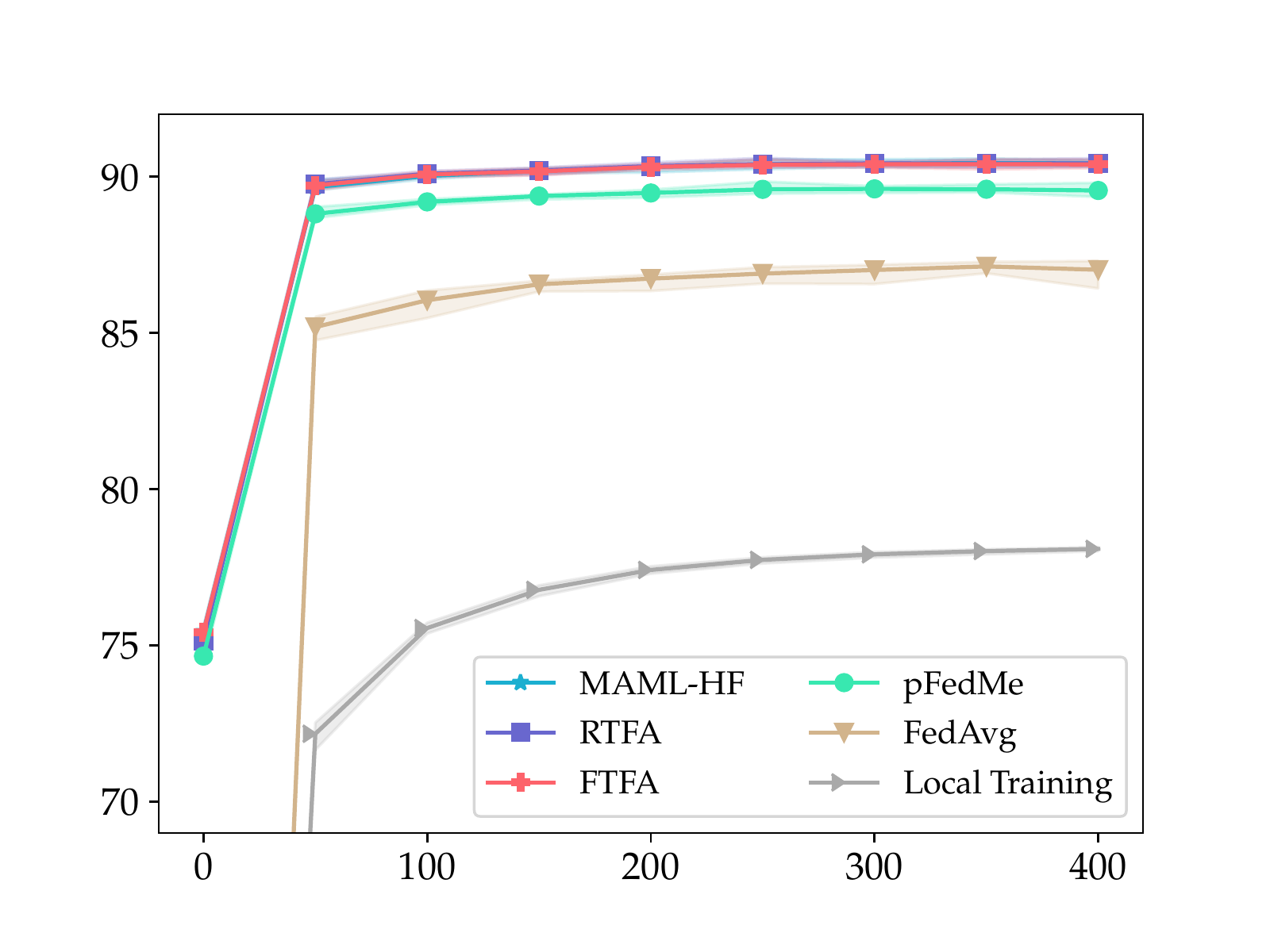}
      		\put(-2,24){
          \rotatebox{90}{{\scriptsize Test Accuracy}}}
        \put(25,0){{\scriptsize Communication Rounds}}
\end{overpic}\caption{EMNIST. Best-average-worst intervals created from different random seeds.}
\label{fig:fedemnist}
\end{minipage}%
\hfill
\begin{minipage}[c]{0.3\linewidth}
\begin{overpic}[width=\linewidth]{
      		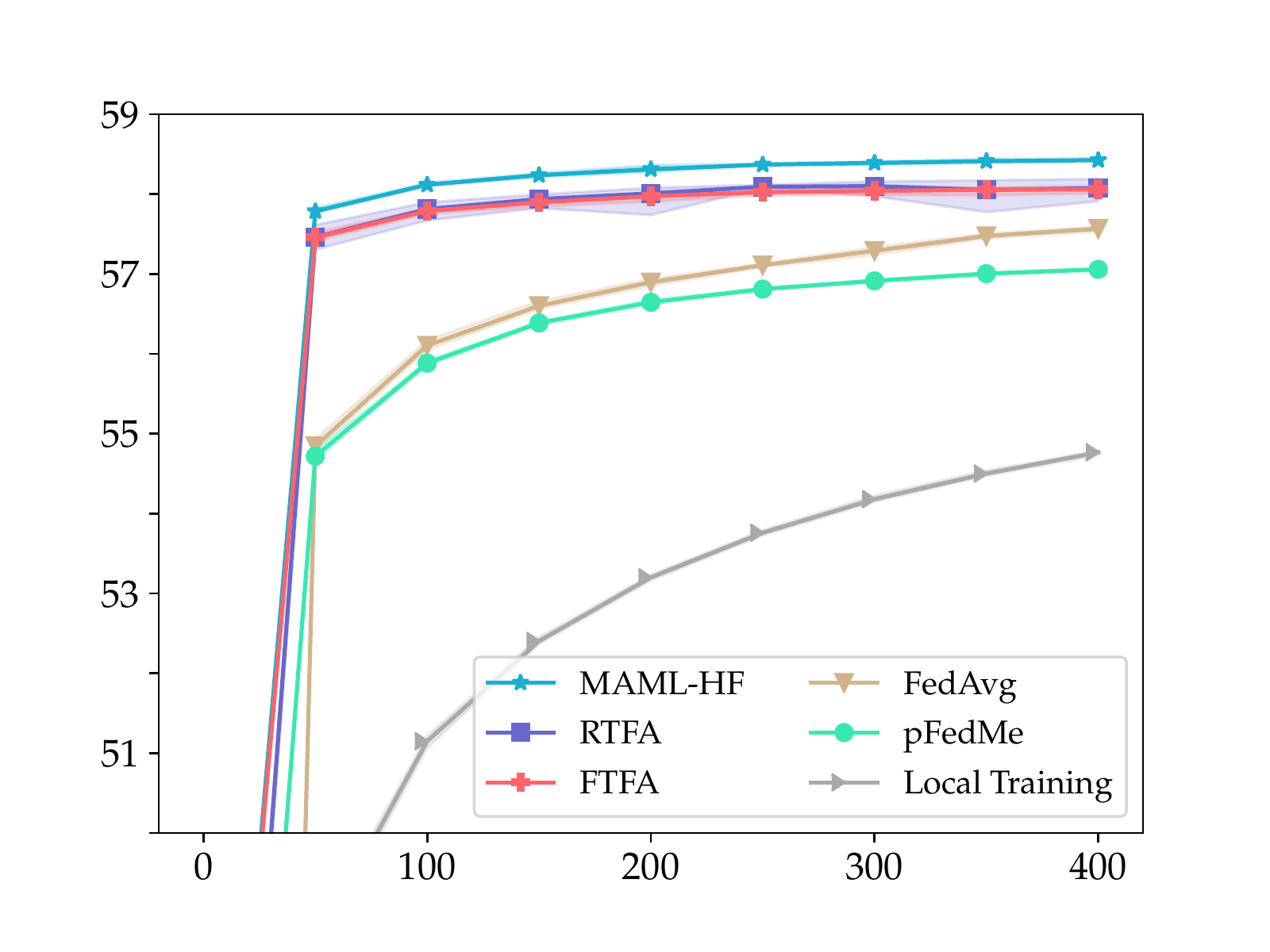}
      		\put(-2,24){
          \rotatebox{90}{{\scriptsize Test Accuracy}}}
        \put(25,0){{\scriptsize Communication Rounds}}
\end{overpic}\caption{Shakespeare. Best-average-worst intervals created from different random seeds.}
\label{fig:shakespeare}
\end{minipage}
\vfill
\begin{minipage}[c]{0.3\linewidth}
\begin{overpic}[width=\linewidth]{
      		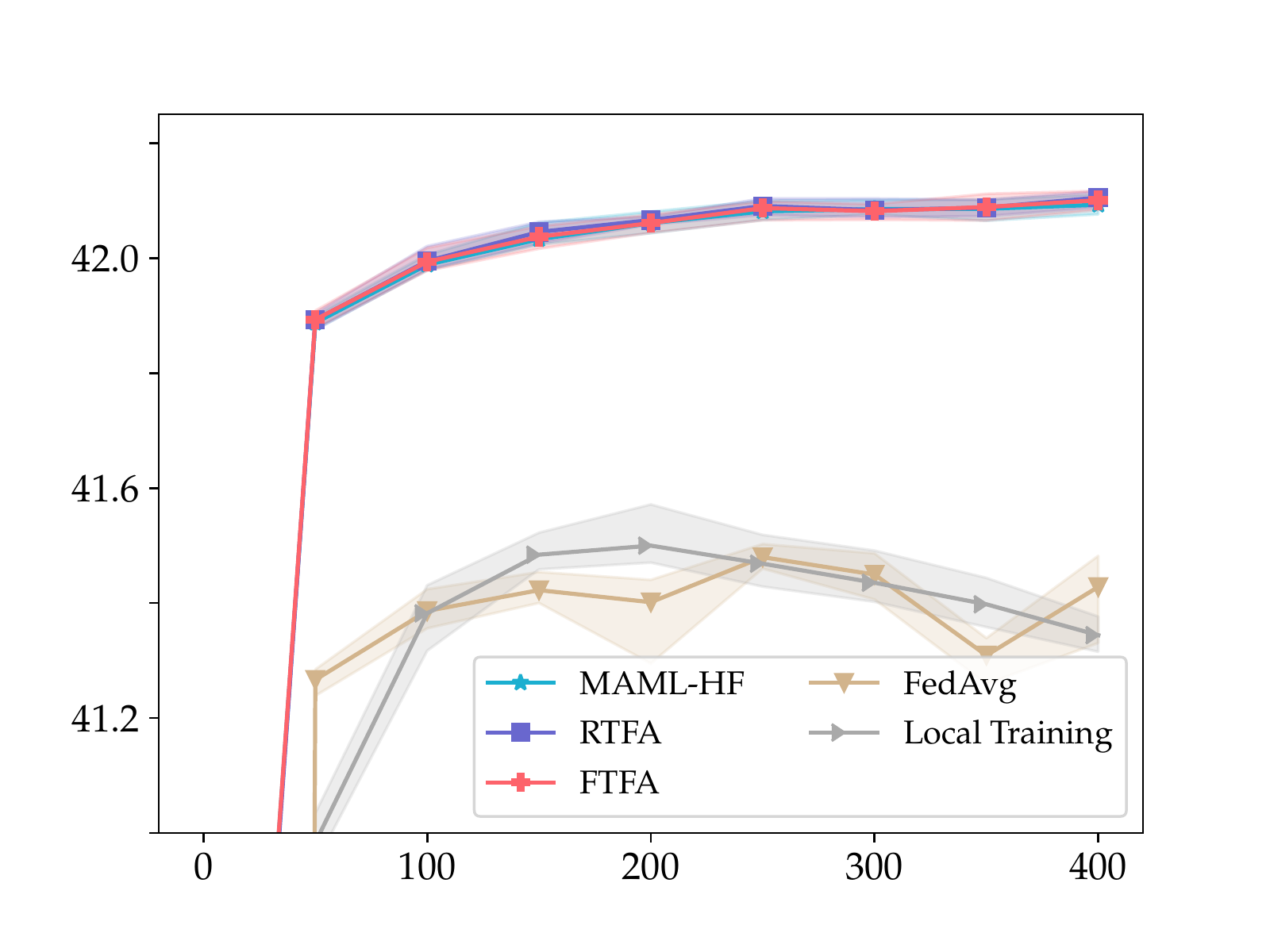}
      		\put(-2,24){
          \rotatebox{90}{{\scriptsize Test Accuracy}}}
        \put(25,0){{\scriptsize Communication Rounds}}
\end{overpic}\caption{Stack Overflow. Best-average-worst intervals created from different train-val splits.}
\label{fig:stackoverflow}
\end{minipage}%
\hfill
\begin{minipage}[c]{0.3\linewidth}
\begin{overpic}[width=\linewidth]{
      		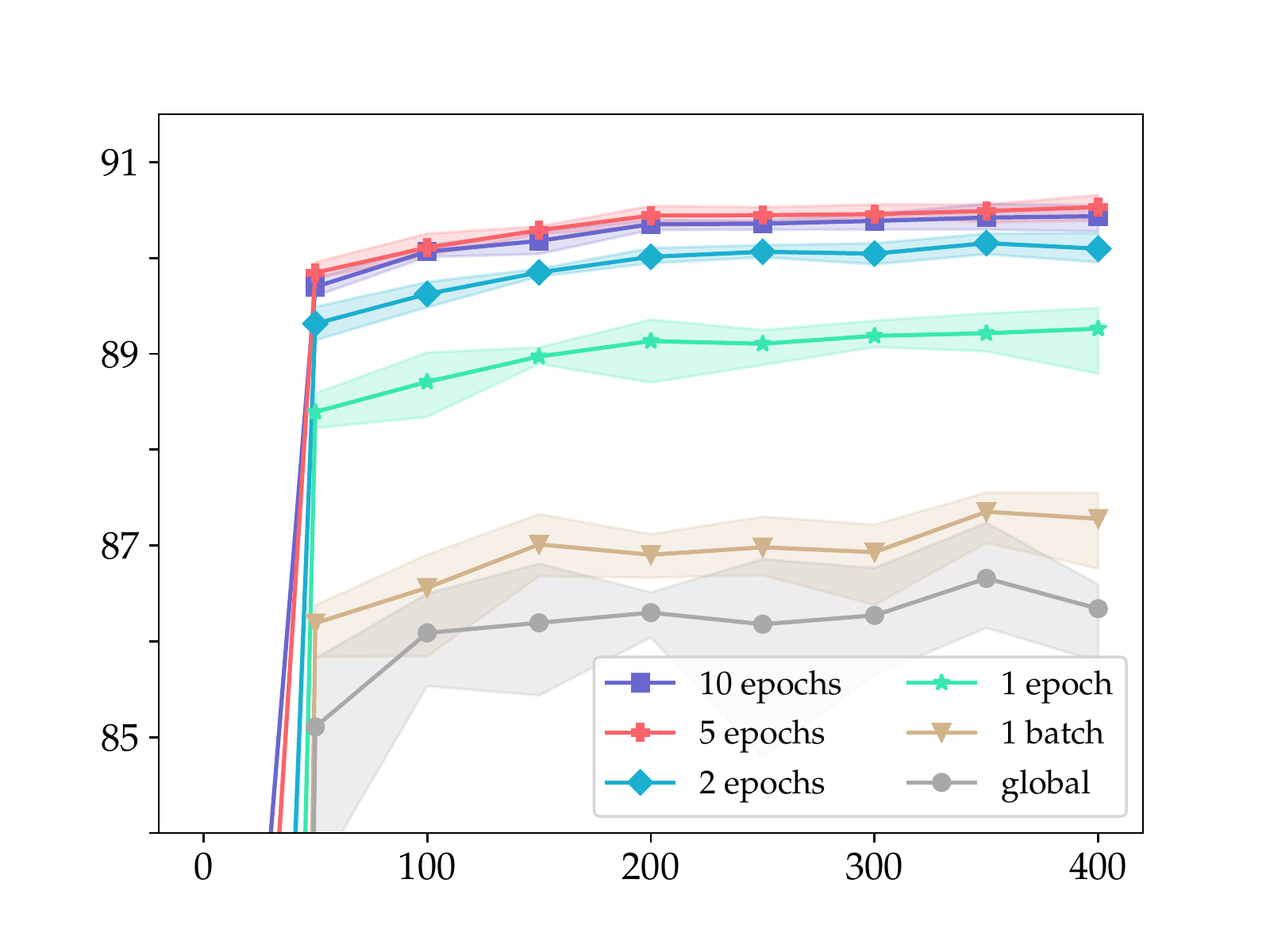}
      		\put(-2,24){
          \rotatebox{90}{{\scriptsize Test Accuracy}}}
        \put(25,0){{\scriptsize Communication Rounds}}
\end{overpic}\caption{EMNIST. Gains of personalization for \ftfa}
\label{fig:pers-femnist}
\end{minipage}%
\hfill
\begin{minipage}[c]{0.3\linewidth}
\begin{overpic}[width=\linewidth]{
      		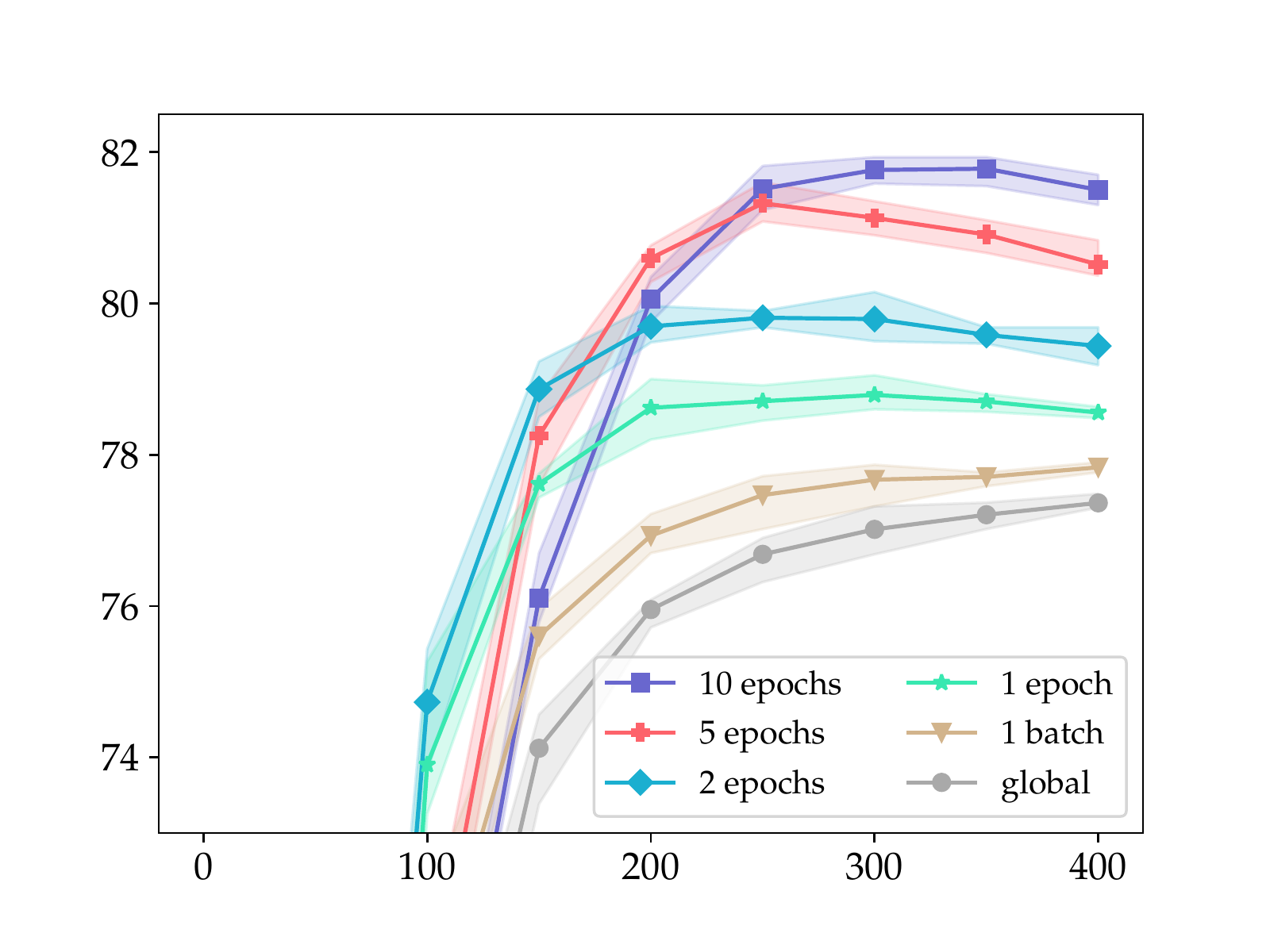}
      		\put(-2,24){
          \rotatebox{90}{{\scriptsize Test Accuracy}}}
        \put(25,0){{\scriptsize Communication Rounds}}
\end{overpic}\caption{CIFAR. Gains of personalization for Hessian free \mamlfl}
\label{fig:pers-cifar}
\end{minipage}
\end{figure*}

While the statistical model we assume in our analytical sections is
stylized and certainly will not fully hold, it suggests
some guidance in practice, and make precise predictions about
the error rates of different methods: that the simpler
fine-tuning methods should exhibit performance comparable to more complex
federated methods, such as \mamlfl and \pfedme. With this in mind, we turn to
several datasets, performing
experiments on federated versions of the Shakespeare
\cite{McMahanMoRaHaAr17}, CIFAR-100 \cite{KrizhevskyHi09}, EMNIST
\cite{CohenAsTaSc17}, and Stack Overflow \cite{TensorflowfederatedSO19}
datasets; dataset statistics and details of how we divide the data to make
effective ``users'' are in \Cref{app:expts}.  For each dataset, we compare
the performance of the following algorithms: Zero Communication (Local
Training), Zero Personalization (FedAvg), \ftfa, \rtfa, \mamlfl, and \pfedme
\cite{DinhTrNg20}. For each classification task, we use each federated
learning algorithm to
train the last layer of a pre-trained neural network.
We run each algorithm for 400 communication rounds, and we compute the test
accuracy (the fraction of correctly classified test data points across
machines) every 50 communication rounds. \ftfa, \rtfa, and \mamlfl\ each
perform 10 epochs of local training for each client before the evaluation of
test accuracy. For each client, \pfedme\ uses the local models to compute
test accuracy.  We first hyperparameter tune each method using training and
validation splits; again, see Appendix~\ref{app:expts} for details. We track
the test accuracy of each tuned method over 11 trials using two different
kinds of randomness:
\begin{enumerate}[leftmargin=0.5cm]
\item \textit{Different seeds}: We run each hyperparameter-tuned method on
  11 different seeds. This captures how different initializations and
  batching affect accuracy.
\item \textit{Different training-validation splits}: We generate 11
  different training / validation splits (same test data) and run each
  hyperparameter-tuned method on each split. This captures how variations in
  user data affect test accuracy.
\end{enumerate}

\paragraph{Experimental setting}
Our experiments are ``semi-synthetic'' in that in each, we re-fit the top
layer of a pre-trained neural network. While this differs from some practice
with experimental work in federated learning, several considerations
motivate our choices to take this tack, and we contend they may be valuable
for other researchers: (i) our (distributed) models are convex, that is, can
be fit via convex optimization. In the context of real-world engineering
problems, it is important to know when a method has converged and, if it
does not, why it has not; in this vein, non-convexity can be a bugaboo, as
it hides the causes of divergent algorithms---is it non-convexity and poor
optimization or engineering issues (e.g.\ communication bugs)? This choice
thus can be valuable even in real, large-scale systems. (ii) In our
experiments, we achieve state-of-the-art or near state-of-the-art results;
using federated approaches to fit full deep models appears to lead to
substantial degradation in performance over a single centralized,
pre-trained model (see, e.g., \citet[Table 1]{ReddiChZaGaRuKoKuMc21}, where
accuracies on CIFAR-10 using a ResNet 18 are at best 78\%, substantially
lower than current state-of-the-art). A question whose answer we do not
know: if a federated learning method provides worse performance than a
downloadable model, what does the FL method's performance tell us about good
methodologies in federated learning?  (iii) Finally, computing with
large-scale distributed models is computationally
expensive: the energy use for fitting large distributed models is substantial
and may be a poor use of resources~\cite{StrubellGaMc19}. In effort to
better approximate the use of a pre-trained model in real federated learning
applications, we use held-out data to pre-train a preliminary network in our Stack Overflow
experiments, doing the experimental training and validation on an
independent dataset.

\vspace{-.2cm}

\paragraph{Results} \Cref{fig:cifar,fig:fedemnist,fig:shakespeare,fig:stackoverflow} plot test accuracy against communication rounds. 
The performance of \mamlfl\ is similar to that of \ftfa\ and \rtfa, and on
the Stack Overflow and EMNIST datasets, where the total dataset size is much
larger than the other datasets, the accuracies of \mamlfl, \ftfa\ and
\rtfa\ are nearly identical. This is consistent with our theoretical claims.
The performances of the naive, zero communication and zero personalization
algorithms are worse than that of \ftfa, \rtfa\ and \mamlfl\ in all
figures. This is also consistent with our theoretical claims.  The
performance of \pfedme\ in \Cref{fig:cifar,fig:fedemnist,fig:shakespeare} is
worse than that of \ftfa, \rtfa\ and \mamlfl.

In \Cref{fig:pers-femnist,fig:pers-cifar}, we plot the test accuracy of
\ftfa\ and \mamlfl\ and vary the number of personalization steps each
algorithm takes. In both plots, the global model performs the worst, and
performance improves monotonically as we increase the number of
personalization steps. As personalization steps are cheap relative to the
centralized training procedure, this suggests benefits for clients to
locally train to convergence.

\newpage

 \small

\bibliography{bib, opl-bib}
\bibliographystyle{alpha}

\normalsize



\newpage
 \newpage
\section{Proofs}
\label{app:proofs}
\subsection{Additional Notation}
To simplify notation, we define some aggregated parameters, $X_i := [\bx_{i, 1}, \ldots, \bx_{i, n_i}]^T \in \R^{n_i \times d} $, $\by_i = [y_{i, 1}, \ldots, y_{i, n_i}]^T\in \R^{n_i}$, $X := [X_1^T, \ldots, X_m^T]^T \in \R^{N \times d}$, and $\by := [\by_1^T, \ldots, \by_m^T]^T\in \R^{N}$. Additionally, we define $\scov_i := X_i^TX_i / n_i \in \R^{d \times d}$. We use the notation $a \lesssim b$ to denote $a \le K b$ for some absolute constant $K$. 
\subsection{Useful Lemmas}

\begin{lemma}\label{lem:vec-khintchine}
 Let $\bx_j$ be vectors in $\R^d$ and let $\zeta_j$ be Rademacher $(\pm1)$ random variables. Then, we have
 \[
 \E\left[\ltwo{\summach \zeta_j \bx_j}^p\right]^{1/p} \le \sqrt{p-1}\left(\summach \ltwo{\bx_j}^2\right)^{1/2},
 \]
 where the expectation is over the Rademacher random variables.
\end{lemma}
\begin{proof}
Using Theorem 1.3.1 of \cite{delaPenaGi99}, we have 
\begin{align*}
\E\left[\ltwo{\summach \zeta_j \bx_j}^p\right]^{1/p} 
&\le \sqrt{p-1}\E\left[ \ltwo{\summach\zeta_j\bx_j}^2\right]^{1/2} \\
&  =  \sqrt{p-1}\E\left[ \sum_{i,j = 1}^{m}\langle\zeta_j\zeta_i\bx_j^T\bx_i\rangle\right]\\
& = \sqrt{p-1}\left(\summach \ltwo{\bx_j}^2\right)^{1/2}
\end{align*}
\end{proof}

\begin{lemma}
\label{lem:scov_moment}
For all clients $j \in [m]$, let the data $\bx_{j,k} \in \R^d$ for $k\in [n]$ be such that $\bx_{j, k} = \Sigma_j^{1/2} \bz_{j, k}$ for some $\Sigma_j$, $\bz_{j,k}$, and $p > 2$ that satisfy \cref{ass:cov}.
Let $(\bx_{j,k})_l \in \R$ denote the $l \in[d]$ entry of the vector $\bx_{j,k} \in \R^d$. 
Define $\scov_j = \frac{1}{n_j}\sum_{k \in [n_j]} {\bx_{j,k} \bx_{j,k}^T}$. 
Then, we have
\begin{align}
    \nonumber \E\left[\opnorm{\scov_j}^{p}\right] \le K (e \log{d})^{p}n_j,
\end{align}
where the inequality holds up to constant factors for sufficiently large $m$.
\end{lemma}
\begin{proof}
We first show a helpful fact that $\E[(\bz_{j,k})_l^{2p}] \leq \moment{p} < \infty$ implies $\E[\ltwo{\bx_{j,k}}^{2p}]^{1/(2p)} \lesssim \sqrt{d}$. For any $j \in [m]$, we have by Jensen's inequality
\begin{align*}
    \E[\ltwo{\bx_{j, k}}^{2p}] \leq \covop^{2p} \E[\ltwo{\bz_{j, k}}^{2p}] = \covop^{2p}d^p\E\left[\left(\frac{1}{d}\sum_{l=1}^d (\bz_{j,k})_l^2\right)^p\right] \leq \covop^{2p} d^p \frac{1}{d}\sum_{l=1}^d \E[(\bz_{j,k})_l^{2p}] \leq \covop^{2p} \kappa_p d^p
\end{align*}
We define some constant $C_4 > \covop^{2p} \moment{p}$. 
With this fact and Theorem A.1 from \cite{ChenGiTr12}, we have
\begin{align*}
    \nonumber \E\left[ \opnorm{\scov_j}^{p}\right] &= \E\left[ \opnorm{\sumsamp \frac{\bx_{j,k}\bx_{j,k}^T}{n}}^{p}\right] \leq 2^{2p-1} \left(\opnorm{\Sigma_j}^{p} +  \frac{(e \log{d})^{p}}{n_j^{p}}\E\left[\max_k \opnorm{\bx_{j,k}\bx_{j,k}^T}^{p}\right]\right)\\
    \nonumber &\leq 2^{2p-1}\left(C +  \frac{(e \log{d})^{p}}{n_j^{p-1}}\E\left[\ltwo{\bx_{j,k}}^{2p}\right]\right) \\
    &\leq 2^{2p-1}\left(C + C_4 \frac{(e \log{d})^{p} d^{p}}{n_j^{p-1}} \right)
\end{align*}
Now, $2^{2p-1}\left(C + C_4 \frac{(e \log{d})^{p} d^{p}}{n_j^{p-1}} \right) \le K (e \log{d})^{p}n_j$ for some absolute constant $K$ since $\frac{d}{n_j} \to \gamma_i$. 
\end{proof}

\begin{lemma}
\label{lem:scov_concentration}
For all clients $j \in [m]$, let the data $\bx_{j,k} \in \R^d$ for $k\in [n]$ be such that $\bx_{j, k} = \Sigma_j^{1/2} \bz_{j, k}$ for some $\Sigma_j$, $\bz_{j,k}$, and $q' > 2$ that satisfy \cref{ass:cov}. Further let $q' = pq$ where $p \ge 1$ and $q\ge 2$. Let $\scov_j = \frac{1}{n_j}\sum_{k \in [n_j]} {\bx_{j,k} \bx_{j,k}^T}$ and $\mu_j = \E[\scov_j^p]$. Additionally assume that $\opnorm{\E[\scov_j^{2p}]} \leq C_3$ for some constant $C_3$. Let $d,n_j$ grow as in \Cref{ass:asym}. Then we have for sufficiently large $m$,

\begin{align*}
    \P\left(\opnorm{\summach p_j \left(\scov_j^p - \mu_j\right)} > t  \right) &\leq
     \frac{2^{q-1}C_2}{t^q} \left[(\log d)^{q/2} \summach p_j^{q/2 + 1}  + (\log d)^{pq + q}\summach p_j^q n_j\right].
\end{align*}
Further supposing that $(\log d)^{pq + q}\summach p_j^q n_j \rightarrow 0$ 
, we get that $\opnorm{\summach p_j\left(\scov_j^p - \mu_j\right)} \cp 0$ .
\end{lemma}
\begin{proof}
Using Markov's inequality, Jensen's inequality, and symmeterization, we have with $\zeta_j$ iid Rademacher
\begin{align*}
     \P\left(\opnorm{\summach p_j\left(\scov_j^p - \mu_j\right)} > t   \right) &\leq \frac{\E\left[\opnorm{\summach p_j(\scov_j^p - \mu_j)}^q\right]}{t^q} 
     \leq 2^{q} \frac{\E\left[\opnorm{\summach p_j\scov_j^p \zeta_j}^q\right]}{t^q}
\end{align*}

We use the second part of Theorem A.1 with $q\geq 2$ from \cite{ChenGiTr12} to bound the RHS.
\begin{align*}
    \E\left[\opnorm{\summach p_j\scov_j^p \zeta_j}^q\right] 
    & \leq \left(\sqrt{e \log d}  \opnorm{\E[\summach p_j^2\scov_j^{2p}]^{1/2}} + (2e \log d) \E\left[\max_j \opnorm{p_j\scov_j^p}^q\right]^{1/q}\right)^q\\
    &\leq 2^{q-1} (e \log d)^{q/2}  \opnorm{\E[\summach p_j^2\scov_j^{2p}]^{1/2}}^q + 2^{q-1}(e \log d)^q \E\left[\max_j p_j^q \opnorm{\scov_j^p}^q\right]\\
    &\leq 2^{q-1} (e \log d)^{q/2} \opnorm{\E\left[\summach p_j^2 \scov_j^{2p}\right]}^{q/2} + 2^{q-1}(e \log d)^q\E\left[ \summach p_j^q\opnorm{\scov_j}^{pq}\right]
\end{align*}

Now we bound the RHS of this quantity using the first part of Theorem A.1. For each $j \in [m]$, we have by \Cref{lem:scov_moment} for sufficiently large $m$,
\begin{align*}
     \E\left[ \opnorm{\scov_j}^{pq}\right] 
     &\le K (e\log d)^{pq} n
     _j,
\end{align*}
for some absolute constant K.
Supposing that $\opnorm{\E\left[\scov_j^{2p}\right]} \leq C_3$ exist for all $j$. Combining all the inequalities, we have for sufficiently large $m$,

\begin{align*}
    \E\left[\opnorm{\summach p_j\scov_j^p \xi_j}^q\right] 
    & \le 2^{q-1} (e \log d)^{q/2}  \left(\summach p_j^{q/2+1} \opnorm{\E\left[\scov_j^{2p}\right]}^{q/2}\right) \\
    & + 2^{q-1} (e \log d)^q \summach p_j^q K (e\log d)^{pq} n_j \\
    & \le C_2\left[(\log d)^{q/2} \summach p_j^{q/2 + 1}  + (\log d)^{pq + q}\summach p_j^q n_j\right],
\end{align*}
where in the first term of the first inequality, we use Jensen's inequality to pull out $\summach p_j$ of the expectation.

To prove the second part of the lemma, we observe that if $(\log d)^{(p+1)q} \summach p_j^q n_j \rightarrow 0$ as $m \rightarrow \infty$ such that $ d/n_i \rightarrow \gamma_i > 1$ for all devices $i \in [m]$, then $(\log d)^{q/2} \summach p_j^{q/2 + 1} \rightarrow 0$. To see this, we first observe
\begin{align*}
    (\log d)^{(p+1)q} \summach p_j^q n_j  \geq (\max_{j \in [m]} p_j (\log d)^{p+1})^q,
\end{align*}
so we know that $\max_{j \in [m]} p_j (\log d)^{p+1} \rightarrow 0$. Further, by Holder's inequality, we know that
\begin{align*}
    (\log d)^{q/2} \summach p_j^{q/2 + 1} \leq (\max_{j \in [m]} p_j \log d)^{q/2}.
\end{align*}
By the continuity of the $q/2$ power, we get the result.
\end{proof}
\begin{lemma}
\label{lem:inverse-concentration}
Let $U \in \R^{d \times d}$ and $V \in \R^{d \times d}$ be positive semidefinite matrices such that $\lambda_{\min{}}(U) \geq \lambda_0$ for some constant $\lambda_0$. Let $d, n_j, m \rightarrow \infty$ as in \Cref{ass:asym}. Suppose $\opnorm{V - U} \cp 0$, then $\opnorm{V\inv - U\inv} \cp 0$.
\end{lemma}
\begin{proof}
For any $t > 0$, we have by Theorem 2.5 (from Section III) of \cite{StewartSu90}
\begin{align*}
    P&\left(\opnorm{V\inv - U\inv} > t\right) \\
    &\leq P\left(\opnorm{V\inv - U\inv} > t \cap \opnorm{V - U} < \frac{1}{\opnorm{U \inv}}\right) + P\left(\opnorm{V - U} \geq \lambda_0\right)\\
    &\leq P\left(\opnorm{U\inv(V - U)} > \frac{t}{t+ \opnorm{U\inv}}\right) + o(1)\\
    &\leq P\left(\opnorm{V - U} > \frac{t \lambda_0}{t+ \lambda_0\inv}\right) + o(1)
\end{align*}
We know this quantity goes to 0 by assumption.
\end{proof}

\subsection{Some useful definitions from previous work}

In this section, we recall some definitions from \cite{HastieMoRoTi19} that will be useful in finding the exact expressions for risk. The expressions for asymptotic risk in high dimensional regression problems (both ridge and ridgeless) are given in an implicit form in \cite{HastieMoRoTi19}. It depends on the geometry of the covariance matrix $\Sigma$ and the true solution to the regression problem $\theta\opt$. Let $\Sigma  = \sum_{i = 1}^d s_i v_i v_i^T$ denote the eigenvalue decomposition of $\Sigma$ with $s_1 \ge s_2 \dots \ge s_d$, and let $(c,\dots,v_d^T\theta\opt)$ denote the inner products of $\theta\opt$ with the eigenvectors. We define two probability distributions which will be useful in giving the expressions for risk:
\begin{align*}
    \esd(s) \coloneqq \frac1d \sum_{i = 1}^{d}\indic{s \ge s_i}, \qquad \wesd(s) \coloneqq \frac{1}{\ltwo{\theta\opt}^2} \sum_{i = 1}^{d}(v_i^T\theta\opt)^2\indic{s \ge s_i}.
\end{align*}
Note that $\wesd$ is a reweighted version of $\esd$ and both have the same support (eigenvalues of $\Sigma$).

\begin{definition}
For $\gamma \in \R^+$, let $c_0 = c_0(\gamma, \esd)$ be the unique non-negative solution of 
\begin{align*}
    1 - \frac{1}{\gamma} = \int \frac{1}{1 + c_0\gamma s}d\esd(s),
\end{align*}
the predicted bias and variance is then defined as
\begin{align}
    \label{eqn:min-bias}\mathscr{B}(\esd,\wesd,\gamma) &\coloneqq \ltwo{\theta\opt}^2\left\{1 + \gamma c_0\frac{\int \frac{s^2}{(1 + c_0\gamma s)}d\esd(s)}{\int \frac{s}{(1 + c_0\gamma s)}d\esd(s)}\right\}\cdot \int \frac{s}{(1 + c_0\gamma s)}d\wesd(s),
    \\
    \label{eqn:min-var}\mathscr{V}(\esd,\gamma) &\coloneqq \sigma^2 \gamma \frac{\int \frac{s^2}{(1 + c_0\gamma s)}d\esd(s)}{\int \frac{s}{(1 + c_0\gamma s)}d\esd(s)}.
\end{align}
\end{definition}
\begin{definition}
For $\gamma \in \R^+$ and $z \in \mathbb{C}_+$, let $m_n(z) = m(z;\esd,\gamma)$ be the unique solution of 
\begin{align*}
    m_{n}(z) \coloneqq \int \frac{1}{s[1 - \gamma - \gamma z m_n(z)] - z}d\esd(s).
\end{align*}
Further, define $m_{n,1}(z) = m_{n,1}(z;\esd,\gamma)$ as 
\begin{align*}
    m_{n,1}(z) \coloneqq \frac{\int \frac{s^2[1 - \gamma - \gamma z m_n(z)]}{[s[1 - \gamma - \gamma z m_n(z)] - z]^2}d\esd(s)}{1 - \gamma \int \frac{zs}{[s[1 - \gamma - \gamma z m_n(z)] - z]^2}d\esd(s)}
\end{align*}
The definitions are extended analytically to $Im(z) = 0$ whenever possible, the predicted bias and variance are then defined by
\begin{align}
    \label{eqn:ridge-bias} \mathscr{B}(\lambda;\esd,\wesd,\gamma) &\coloneqq \lambda^2 \ltwo{\theta\opt}(1 + \gamma m_{n,1}(-\lambda))\int \frac{s}{[\lambda + (1 - \gamma + \gamma \lambda m_n(-\lambda))s]^2}d\wesd(s),\\
    \label{eqn:ridge-var} \mathscr{V}(\lambda;\esd,\gamma) &\coloneqq \sigma^2 \gamma \int \frac{s^2((1 - \gamma + \gamma \lambda m'_n(-\lambda)))}{[\lambda + (1 - \gamma + \gamma \lambda m_n(-\lambda))s]^2}d\esd(s).
\end{align}
\end{definition}



\subsection{Proof of \Cref{thm:perfedavg}}\label{app:proof-perfedavg}

On solving \eqref{eqn:param-fedavg} and \eqref{eqn:param-perfedavg}, the closed form of the estimators $\hparamsolfa{0}$ and $\hparamsolfa{i}$ is given by
\begin{align}
    \nonumber \hparamsolfa{0} &= \argmin_{\param} \summach p_j \frac{1}{2n_j} \ltwo{X_j \param - \by_j}^2 
    =\wscov\inv \summach p_j\frac{X_j^T\by_j}{n_j}\\
    &= \wscov\inv \summach p_j \scov_j \param_j\opt + \wscov\inv \summach p_j \frac{X_j^T \xi_j}{n_j} \label{eqn:paramfa_expanded}
\end{align}

and 
\begin{align*}
    \hparamsolfa{i} &= (I - \scov_i\pinv \scov_i)\hparamsolfa{0} + X_i\pinv \by_i 
    = (I - \scov_i\pinv \scov_i) \hparamsolfa{0} + \scov_i\pinv\scov_i \param_i\opt + \frac{1}{n_i} \scov_i\pinv X_i^T \xi_i\\
    &= \Pi_i \left[ \wscov\inv \summach p_j \scov_j \param_j\opt + \wscov\inv \summach p_j \frac{X_j^T \xi_j}{n_j}\right] + \scov_i\pinv\scov_i \param_i\opt + \frac{1}{n_i} \scov_i\pinv X_i^T \xi_i
\end{align*}

We now calculate the risk by splitting it into two parts as in \eqref{eqn:risk}, and then calculate the asymptotic bias and variance. 

\subsubsection*{Bias:}
\begin{align*}
B_i(\hparamsolfa{i}|X ) & \coloneqq \lsigi{\E[\hparamsolfa{i} | X]  - \param_i\opt}^2 = \lsigi{\Pi_i \left[  \wscov\inv \summach p_j\scov_j (\param_j\opt - \param_i\opt) \right]}^2\\
     &= \ltwo{\Sigma_i^{1/2} \Pi_i \left[\param_0\opt - \param_i\opt + \wscov\inv \summach p_j \scov_j (\param_j\opt - \param_0\opt)  \right]}^2
\end{align*}
The idea is to show that the second term goes to $0$ and use results from \cite{HastieMoRoTi19} to find the asymptotic bias. For simplicity, we let $\Delta_j := \param_j\opt - \param_0\opt$, and we define the event:
\begin{align*}
    B_t &:= \left\{\opnorm{\left(\summach p_j\scov_j\right)\inv - \left(\summach p_j\Sigma_j\right)\inv} > t\right\}\\
    A_t &:= \left\{\lsigi{\wscov\inv \summach p_j \scov_j \Delta_j} > t\right\}
\end{align*}

The proof proceeds in the following steps:

\noindent\fbox{%
    \parbox{\textwidth}{%
    \textbf{Bias Proof Outline}
\begin{enumerate}[label={\bf Step \arabic*.},leftmargin=1.5cm]
    \item We first show for any $t > 0$, the $\P(B_t)\rightarrow 0$ as $d \rightarrow \infty$
    \item Then we show for any $t > 0$, the $\P(A_t)\rightarrow 0$ as $d \rightarrow \infty$
    \item We show that for any $t \in (0, 1]$
    on event $A_t^c$, $B_i(\hparamsolfa{i}|X ) \leq \lsigi{\Pi_i [\param_0\opt - \param_i\opt]}^2 + c t$ and $B_i(\hparamsolfa{i}|X ) \geq \lsigi{\Pi_i [\param_0\opt - \param_i\opt]}^2 - c t$
    \item Show that $\lim_{d \rightarrow \infty}  \P(| B_i(\hparamsolfa{i}|X ) - \lsigi{\Pi_i [\param_0\opt - \param_i\opt]}^2|\leq \varepsilon) = 1$ 
    \item Finally, using the asymptotic limit of  $\lsigi{\Pi_i [\param_0\opt - \param_i\opt]}^2$ from Theorem 1 of \cite{HastieMoRoTi19}, we get the result.
\end{enumerate}
}}
\paragraph{Step 1}
Since we have $\lambda_{\min{}}(\summach p_j\Sigma_j) > 1/ \covop> 0$, it suffices to show by \Cref{lem:inverse-concentration} that the probability of
\begin{align*}
     C_t &:= \left\{\opnorm{\summach p_j \scov_j - \summach p_j\Sigma_j} > t\right\}\\
\end{align*}
goes to 0 as $d, m \rightarrow \infty$ (obeying \Cref{ass:asym}).
Using \Cref{lem:scov_concentration} with $p=1$, we have that
\begin{align*}
    \P(C_t) &\leq  \frac{2^{q-1}C_2}{t^q} \left[(\log d)^{q/2} \summach p_j^{q/2 + 1}  + (\log d)^{2q}\summach p_j^q n_j\right]
\end{align*}
Since $(\log d)^{2q}\summach p_j^q n_j \rightarrow 0$, this quantity goes to 0.
\paragraph{Step 2}
Fix any $t > 0$,
\begin{align*}
    \P(A_t) &\leq \P\left( \left\{\lsigi{\wscov\inv \summach p_j \scov_j \Delta_j} > t\right\} \cap B_{c_1}^c\right) + \P\left(B_{c_1}\right)\\
    &\leq  \P\left( \covop (c_1 + \covop)\ltwo{ \summach p_j \scov_j \Delta_j} > t \right) + \P\left(B_{c_1}\right)
\end{align*}
By Step 1, we know that $\P\left(B_{c_1}\right) \rightarrow 0$. The second inequality comes from $\ltwo{Ax} \leq \opnorm{A}\ltwo{x}$ and triangle inequality. Now to bound the first term, we use Markov and a Khintchine inequality (\Cref{lem:vec-khintchine}). We have that 
\begin{align*}
    \P\left( \covop (c_1 + \covop)\ltwo{\summach p_j\scov_j \Delta_j } > t\right) &\leq \frac{( \covop (c_1 + \covop))^q\E\left[\ltwo{\summach p_j \scov_j \Delta_j }^q\right]}{t^q}\\
    & \leq \frac{\left(2 \covop (c_1 + \covop)\sqrt{q}\right)^q \E\left[\left(\summach \ltwo{p_j\scov_j \Delta_j}^2\right)^{q/2}\right]}{t^q }
\end{align*}
Using Jensen's inequality and the definition of operator norm, we have
\begin{align*}
       \frac{\left(2\covop (c_1+\covop)\sqrt{q}\right)^q \E\left[\left(\summach p_j^2 \ltwo{\scov_j \Delta_j}^2\right)^{q/2}\right]}{t^q} 
       &= \frac{\left(2\covop (c_1+\covop)\sqrt{q}\right)^q \E\left[\left(\summach p_j \cdot p_j \ltwo{\scov_j \Delta_j}^2\right)^{q/2}\right]}{t^q} \\
       &\leq  \frac{\left(2\covop (c_1+\covop)\sqrt{q}\right)^q \summach p_j^{q/2 + 1} \E\left[\ltwo{\scov_j \Delta_j}^{q}\right]}{t^q}  \\
    &\leq \frac{\left(2\covop (c_1+\covop)\sqrt{q}\right)^q \summach p_j^{q/2 + 1}  \E\left[\opnorm{\scov_j}^q\right] \E\left[ \ltwo{\Delta_j}^{q}\right]}{t^q}
\end{align*}

Lastly, we can bound this using \Cref{lem:scov_moment} as follows.

\begin{align*}
    \P\left(\covop (c_1+\covop)\ltwo{\summach p_j \scov_j \Delta_j } > t\right) &\le \frac{K  \left(2\covop (c_1+\covop)\sqrt{q}\right)^q (e\log d)^q \summach n_j p_j^{q/2 + 1} \E[ \ltwo{\Delta_j}^{q}]}{t^q} \rightarrow 0,
\end{align*}
using  $(\log d)^q\summach p_j^{q/2 + 1} n_j r_j^q  \rightarrow 0$.

\paragraph{Step 3} 
For any $t \in (0, 1]$, on the event $A_t^c$, we have that 
\[B(\hparamsolfa{i}| X) = \lsigi{\Pi_i[\param_0\opt - \param_i\opt + E]}^2\]
for some vector $E$ where we know $\ltwo{E} \leq t$(which means $\ltwo{E} \leq t \sqrt{\covop}$).

Thus, we have
\begin{align*}
\lsigi{\Pi_i[\param_0\opt - \param_i\opt + E]}^2 &\leq  \lsigi{\Pi_i[\param_0\opt - \param_i\opt]}^2 + \lsigi{\Pi_i E}^2 + 2\lsigi{\Pi_i E} \lsigi{\Pi_i[\param_0\opt - \param_i\opt]}\\
&\leq \lsigi{\Pi_i[\param_0\opt - \param_i\opt]}^2 + \covop^2 t^2 + 2t \covop^{3/2} r_i^2\\
\lsigi{\Pi_i[\param_0\opt - \param_i\opt + E]}^2 &\geq  \lsigi{\Pi_i[\param_0\opt - \param_i\opt]}^2 + \lsigi{\Pi_i E}^2 - 2\ltwo{\Pi_i E} \lsigi{\Pi_i[\param_0\opt - \param_i\opt]}\\
&\geq \lsigi{\Pi_i[\param_0\opt - \param_i\opt]}^2 - 2t \covop^2 r_i^2
\end{align*}

Since $t\in (0, 1]$, we have that $t^2 \leq t$ and thus we can choose $c = M^2 + 2\covop^{3/2} r_i^2$

\paragraph{Step 4} Reparameterizing $\varepsilon \coloneqq c t$, we have that for any $\varepsilon > 0$
    \begin{align*}
        \lim_{n \rightarrow \infty} \P(| B_i(\hparamsolfa{i}|X ) - \lsigi{\Pi_i [\param_0\opt - \param_i\opt]}^2|\leq \varepsilon) &\geq \lim_{n \rightarrow \infty} \P(| B_i(\hparamsolfa{i}|X ) - \lsigi{\Pi_i [\param_0\opt - \param_i\opt]}^2|\leq \varepsilon \wedge c) \\
        \geq \lim_{n \rightarrow \infty} \P(A_{\frac{\varepsilon}{c} \wedge 1}^c) = 1
    \end{align*}

\paragraph{Step 5} Using Theorem 3 of \cite{HastieMoRoTi19}, as  $d \rightarrow \infty$, such that $\frac{d}{n_i} \rightarrow \gamma_i > 1$, we know that the limit of  $\lsigi{\Pi_i [\param_0\opt - \param_i\opt]}^2$ is given by \eqref{eqn:min-bias} with $\gamma = \gamma_i$ and $\esd,\wesd$ be the empirical spectral distribution and weighted empirical spectral distribution of $\Sigma_i$ respectively.

In the case when $\Sigma_i = I$, using Theorem 1 of \cite{HastieMoRoTi19} we have $B_i(\hparamsolfa{i}|X ) = \ltwo{\Pi_i [\param_0\opt - \param_i\opt]}^2 \rightarrow r_i^2\left(1 - \frac{1}{\gamma_i}\right)$.

\subsubsection*{Variance:}
We let $\xi_i = [\xi_{i,1},\dots,\xi_{i,n}]$ denote the vector of noise.
\begin{align*}
     V_i(\hparamsolfa{i}|X) & = \tr(\cov(\hparamsolfa{i}|X) \Sigma_i) = \E\left[\lsigi{\hparamsolfa{i} -  \E\left[\hparamsolfa{i} | X\right]}^2 | X\right]\\
    \nonumber    &= \E\left[\lsigi{\Pi_i \left[ \wscov\inv \summach p_j \frac{X_j^T \xi_j}{n_j}\right]  + \frac{1}{n_i} \scov_i\pinv X_i^T \xi_i }^2|X\right] \\
     &=  \underbrace{\summach \frac{p_j^2}{n_j} \tr\left(\Pi_i \Sigma_i \Pi_i \wscov\inv \scov_j \wscov\inv\right) \sigma_j^2}_{(i)} + \\ 
    & \underbrace{2 \tr\left(\Sigma_i \scov_i\pinv \frac{X_i^TX_i}{n_i^2} \wscov\inv \Pi_i\right) \sigma_i^2}_{(ii)} +  \underbrace{\frac{1}{n_i^2} \tr\left(\scov_i\pinv X_i^TX_i \scov_i\pinv\Sigma_i \right)\sigma_i^2}_{(iii)}
\end{align*}
We now study the asymptotic behavior of each of the terms $(i),(ii)$ and $(iii)$ separately.

\paragraph{(i)}
Using the Cauchy Schwartz inequality on Schatten $p-$norms and using the fact that the nuclear norm of a projection matrix is at most $d$, we get
\begin{align}
    \nonumber & \summach \frac{p_j^2\sigma_j^2}{n_j} \tr\left(\Pi_i \Sigma_i \Pi_i\wscov\inv \scov_j \wscov\inv\right)\\ 
    \nonumber &\le \summach \frac{p_j^2\sigma_j^2}{n_j} \matrixnorms{\Pi_i}_{1} \opnorm{\Sigma_i} \opnorm{\Pi_i} \opnorm{\wscov\inv} \opnorm{\scov_j} \opnorm{\wscov\inv} \\
    \label{eqn:ftfa-var-(i)-cal} & \le C_3 \sigma_{\max}^2\gammamax \left(\summach p_j^2\opnorm{\scov_j}\right),
\end{align}
where the last inequality holds with probability going to $1$ for some constant $C_3$ because $\P(B_t) \rightarrow 0$.
Lastly, we show that $\P\left(\summach p_j^2\opnorm{\scov_j} > t\right) \rightarrow 0$. Using Markov's and Jensen's inequality, we have
\begin{align*}
    \P\left(\summach p^2_j\opnorm{ \scov_j} > t\right) 
    & \le \frac{\E\left[\summach p_j^2\opnorm{\scov_j}\right]^q}{t^q} 
    \le  \frac{\summach p_j^{q + 1} \E\left[\opnorm{\scov_j}^q\right]}{t^q}
\end{align*}
Using \Cref{lem:scov_moment}, we have
\begin{align*}
    \P\left(\summach p^2_j\opnorm{ \scov_j} > t\right) 
    & \le K\frac{\summach p_j^{q + 1} (e\log d)^qn_j }{t^q}
\end{align*}
Finally, since we know that $\summach p_j^{q + 1} (e\log d)^qn_j  \rightarrow 0$, we have $\summach p^2_j\opnorm{ \scov_j} \cp 0$. Thus, 
\[
\summach \frac{p_j^2\sigma_j^2}{n_j} \tr\left(\Pi_i \Sigma_i \Pi_i \wscov\inv \scov_j \wscov\inv\right) \cp 0
\]

\paragraph{(ii)}

Using the Cauchy Schwartz inequality on Schatten $p-$norms and using the fact that the nuclear norm of a projection matrix is $d-n$, we get
\begin{align*}
    \frac{2p_i\sigma^2}{n_i} \tr\left(\Pi_i\Sigma_i \scov_i\pinv \scov_i \wscov\inv\right) & \le  \frac{2p_i\sigma^2 }{n_i}\matrixnorms{\Pi_i}_1 \opnorm{\Sigma_i} \opnorm{\scov_i\pinv\scov_i} \opnorm{\wscov\inv} \\
    & \leq C_4p_i,
\end{align*}
where the last inequality holds with probability going to $1$ for some constant $C_4$ because $\P(B_t) \rightarrow 0$ and using \Cref{ass:cov}. Since $p_i \to 0$, we have

\begin{align*}
    \frac{2p_i\sigma^2}{n_i} \tr\left(\Pi_i\Sigma_i \scov_i\pinv \scov_i \wscov\inv\right) \to 0
\end{align*}

\paragraph{(iii)}
\[
\frac{1}{n_i^2} \tr(\scov_i\pinv X_i^TX_i \scov_i\pinv \Sigma_i)\sigma_i^2 = \frac{1}{n_i} \tr(\scov_i\pinv \Sigma_i)\sigma_i^2
\]
Using Theorem 3 of \cite{HastieMoRoTi19}, as  $d \rightarrow \infty$, such that $\frac{d}{n_i} \rightarrow \gamma_i > 1$, we know that the limit of  $\frac{\sigma_i^2}{n_i} \tr(\scov_i\pinv\Sigma_i)$ is given by \eqref{eqn:min-var} with $\gamma = \gamma_i$ and $\esd,\wesd$ be the empirical spectral distribution and weighted empirical spectral distribution of $\Sigma_i$ respectively.

In the case when $\Sigma_i = I$, using Theorem 1 of \cite{HastieMoRoTi19} we have $V_i(\hparamsolfa{i}|X) = \frac{\sigma_i^2}{n_i} \tr(\scov_i\pinv) \rightarrow \frac{\sigma_i^2}{\gamma_i - 1}$.



\subsection{Proof of \Cref{thm:perridge}} \label{app:proof-ridge}

We use the global model from \eqref{eqn:param-fedavg} and the personalized model from \eqref{eqn:param-perridge}. The closed form of the estimators $\hparamsolfa{0}$ and $\hparamsolrt{i}$ is given by
\begin{align*}
    \hparamsolfa{0} &= \argmin_{\param} \summach p_j \frac{1}{2n_j} \ltwo{X_j \param - y_j} 
    =\wscov\inv \summach p_j\frac{X_j^Ty_j}{n_j}\\
    &= \wscov\inv \summach p_j \scov_j \param_j\opt + \wscov\inv \summach p_j \frac{X_j^T \xi_j}{n_j}
\end{align*}

and 
\begin{align*}
        \hparamsolrt{i} &= \argmin_{\param} \frac{1}{2n_i}\ltwo{X_i \param - y_i}^2 +\frac{\lambda}{2}\ltwo{\hparamsolfa{0} - \param}^2\\
        &= (\scov_i + \lambda I)\inv \left(\lambda \hparam_{FA} + \scov_i \param_i\opt + \frac{1}{n_i} X_i^T \xi_i\right)
\end{align*}

We now calculate the risk by splitting it into two parts as in \eqref{eqn:risk}, and then calculate the asymptotic bias and variance. 

\subsubsection*{Bias:}
\begin{align*}
     B(\hparamsolrt{i}| X) &:= \lsigi{\E[\hparamsolrt{i} | X]  - \param_i\opt}^2 
     = \lambda^2 \lsigi{(\scov_i + \lambda I)\inv \left[ \wscov\inv \summach p_j\scov_j (\param_j\opt - \param_i\opt) \right]}^2\\
     &=\lambda^2 \lsigi{(\scov_i + \lambda I)\inv \left[\param_0\opt - \param_i\opt +  \wscov\inv \summach p_j\scov_j (\param_j\opt - \param_0\opt)  \right]}^2
\end{align*}

The idea is to show that the second term goes to $0$ and use results from \cite{HastieMoRoTi19} to find the asymptotic bias. For simplicity, we let $\Delta_j := \param_j\opt - \param_0\opt$, and we define the event:
\begin{align*}
    B_t &:= \left\{\opnorm{\left(\summach p_j\scov_j\right)\inv - \left(\summach p_j\Sigma_j\right)\inv} > t\right\}\\
    A_t &\coloneqq \left\{\lsigi{\wscov\inv \summach p_j\scov_j \Delta_j} > t\right\}
\end{align*}

The proof proceeds in the following steps:

\noindent\fbox{%
    \parbox{\textwidth}{%
    \textbf{Bias Proof Outline}
\begin{enumerate}[label={\bf Step \arabic*.},leftmargin=1.5cm]
    \item We first show for any $t > 0$, the $\P(B_t)\rightarrow 0$ as $d \rightarrow \infty$
    \item We show for any $t > 0$, the $\P(A_t)\rightarrow 0$ as $d \rightarrow \infty$
    \item We show that for any $t \in (0, 1]$ on event $A_t^c$, $B(\hparamsolrt{i}| X) \leq \lambda^2\lsigi{(\scov_i + \lambda I)\inv[\param_0\opt - \param_i\opt]}^2 + c t$ and $B(\hparamsolrt{i}| X) \geq \lambda^2\lsigi{(\scov_i + \lambda I)\inv[\param_0\opt - \param_i\opt]}^2 - c t$
    \item Show that $\lim_{d \rightarrow \infty} \P(| B(\hparamsolrt{i}| X) - \lambda^2\lsigi{(\scov_i + \lambda I)\inv[\param_0\opt - \param_i\opt]}^2|\leq \varepsilon) = 1$ 
    \item Finally, using the asymptotic limit of  $\lambda^2\lsigi{(\scov_i + \lambda I)\inv[\param_0\opt - \param_i\opt]}^2$ from Corollary 5 of \cite{HastieMoRoTi19}, we get the result.
\end{enumerate}
    }%
}

\textbf{Step 1} and {\bf Step 2} follow from {\bf Step 1} and {\bf Step 2} of proof of \Cref{thm:perfedavg}.

\paragraph{Step 3} 
For any $t \in (0, 1]$, on the event $A_t^c$ where $T_i\inv = (\scov_i + \lambda I)\inv$, we have that 
\[B(\hparamsolrt{i}| X) = \lambda^2\lsigi{T_i\inv[\param_0\opt - \param_i\opt + E]}^2\]
for some vector $E$ where we know $\lsigi{E} \leq t$ (which means $\ltwo{E} \leq t \sqrt{\covop}$).

We can form the bounds 
\begin{align*}
\lsigi{T_i\inv[\param_0\opt - \param_i\opt + E]}^2 &\leq  \lsigi{T_i\inv[\param_0\opt - \param_i\opt]}^2 + \lsigi{T_i\inv E}^2 + 2\lsigi{T_i\inv E} \lsigi{T_i\inv[\param_0\opt - \param_i\opt]}\\
&\leq \lsigi{T_i\inv[\param_0\opt - \param_i\opt]}^2 + \covop^2\lambda^{-2}t^2 + 2\covop^{3/2} t \lambda^{-2}r_i^2\\
\lsigi{T_i\inv[\param_0\opt - \param_i\opt + E]}^2 &\geq  \lsigi{T_i\inv[\param_0\opt - \param_i\opt]}^2 + \lsigi{T_i\inv E}^2 - 2\lsigi{T_i\inv E} \lsigi{T_i\inv[\param_0\opt - \param_i\opt]}\\
&\geq \lsigi{T_i\inv[\param_0\opt - \param_i\opt]}^2 - 2\covop^2 t\lambda^{-2} r_i^2.
\end{align*}

Since $t\in (0, 1]$, we have that $t^2 \leq t$ and thus we can choose $c = \lambda^{-2}(M^2 + 2M^{3/2}r_i^2)$.

\paragraph{Step 4} 
Reparameterizing $\varepsilon := c t$, we have that for any $\varepsilon > 0$
    \begin{align*}
        & \lim_{n \rightarrow \infty} \P(| B(\hparamsolrt{i}| X) - \lambda^2\lsigi{(\scov_i + \lambda I)\inv[\param_0\opt - \param_i\opt]}^2| \leq \varepsilon) \\
        &\geq \lim_{n \rightarrow \infty} \P(| B(\hparamsolrt{i}| X) - \lambda^2\lsigi{(\scov_i + \lambda I)\inv [\param_0\opt - \param_i\opt]}^2|\leq \varepsilon \wedge c) \\
        & \geq \lim_{n \rightarrow \infty} \P(A_{\frac{\varepsilon}{c} \wedge 1}^c) = 1.
    \end{align*}

\paragraph{Step 5}
Using Theorem 6 of \cite{HastieMoRoTi19}, as  $d \rightarrow \infty$, such that $\frac{d}{n_i} \rightarrow \gamma_i > 1$, we know that the limit of  $\lambda^2\lsigi{(\scov_i + \lambda I)\inv[\param_0\opt - \param_i\opt]}^2$ is given by \eqref{eqn:ridge-bias} with $\gamma = \gamma_i$ and $\esd,\wesd$ be the empirical spectral distribution and weighted empirical spectral distribution of $\Sigma_i$ respectively.

In the case when $\Sigma_i = I$, using Corollary 5 of \cite{HastieMoRoTi19} we have $B_i(\hparamsolrt{i}|X ) = \ltwo{\Pi_i [\param_0\opt - \param_i\opt]}^2 \rightarrow r_i^2\lambda^2 m'_i(-\lambda)$.

\subsubsection*{Variance:}

We let $\xi_i = [\xi_{i,1},\dots,\xi_{i,n}]$ denote the vector of noise.
Substituting in the variance formula and using $\E[\xi_i\xi_j^T] = 0$ and $\E[\xi_i\xi_i^T] = \sigma^2 I$, we get
\begin{align*}
    \var(\hparamsolrt{i} | X) &=\E\left[ \lsigi{(\scov_i + \lambda I)^{-1} \left(\frac{1}{n_i}X_i^T \xi_i + \lambda\wscov\inv \summach p_j\frac{X_j^T \xi_j}{n_j}\right)}^2 \bigg| X \right] \\
    &=  \underbrace{\summach \frac{\lambda^2 p_j^2}{n_j} \tr\left((\scov_i + \lambda I)^{-1}\Sigma (\scov_i + \lambda I)^{-1} \wscov\inv \scov_j \wscov\inv\right) \sigma_j^2}_{(i)} \\
    & + \underbrace{2\lambda p_i \tr\left( \frac{X_i^TX_i}{n_i^2} \wscov\inv (\scov_i + \lambda I)^{-1}\Sigma (\scov_i + \lambda I)^{-1} \right) \sigma_i^2}_{(ii)} 
    +   \underbrace{\tr\left((\scov_i + \lambda I)^{-1}\Sigma (\scov_i + \lambda I)^{-1}  \scov_i\right) \frac{\sigma_i^2}{n_i}}_{(iii)}
\end{align*}

We now study the asymptotic behavior of each of the terms $(i),(ii)$ and $(iii)$ separately.
\paragraph{(i)}
Using the Cauchy Schwartz inequality on Schatten $p-$norms, we get
\begin{align*}
    & \summach \frac{p_j^2\lambda^2\sigma_j^2}{n_j} \tr\left((\scov_i + \lambda I)^{-1}\Sigma (\scov_i + \lambda I)^{-1} \wscov\inv \scov_j \wscov\inv\right)\\ 
    &\le \summach \frac{\lambda^2p_j^2\sigma_j^2}{n_j} \matrixnorms{(\scov_i + \lambda I)^{-1}}_{1} \opnorm{\Sigma}\opnorm{(\scov_i + \lambda I)^{-1}} \opnorm{\wscov\inv} \opnorm{\scov_j} \opnorm{\wscov\inv} \\
    & \le C_5 \sigma_{\max{}}^2\gammamax \left(\summach p_j^2\opnorm{\scov_j}\right),
\end{align*}
where the last inequality holds with probability going to $1$ for some constant $C_5$ because $\P(B_t) \rightarrow 0$. Note that this expression is same as \eqref{eqn:ftfa-var-(i)-cal} and hence the rest of the analysis for this term is same as the one in the proof of \ftfa\ (\Cref{app:proof-perfedavg}).

\paragraph{(ii)}
Using the Cauchy Schwartz inequality on Schatten $p-$norms, we get
\begin{align*}
    & \frac{2p_i\lambda\sigma_i^2}{n_i} \tr\left( (\scov_i + \lambda I)^{-1}\scov_i \wscov\inv(\scov_i + \lambda I)^{-1}\Sigma\right)\\
    & \le  \frac{2p_i\lambda\sigma_i^2 }{n_i}\matrixnorms{(\scov_i + \lambda I)^{-1}}_1 \opnorm{ \scov_i}\opnorm{\wscov\inv}\opnorm{\Sigma} \opnorm{(\scov_i + \lambda I)^{-1}}  \\
    & \le \frac{C_6\sigma_i^2 d p_i}{\lambda n_i},
\end{align*}
where $C_6$ is an absolute constant which captures an upper bound on the operator norm of the sample covariance matrix $\scov_i$ using Bai Yin Theorem \cite{BaiYi93}, and an upper bound on the operator norm of $\wscov\inv$, which follows from $\P(B_t) \to 0$.
Since $p_i \rightarrow 0$, we have \\
$ \frac{2p_i\lambda\sigma_i^2}{n_i} \tr\left( (\scov_i + \lambda I)^{-1}\scov_i \wscov\inv(\scov_i + \lambda I)^{-1}\Sigma\right) \cp 0$

\paragraph{(iii)}

Using Theorem 3 of \cite{HastieMoRoTi19}, as  $d \rightarrow \infty$, such that $\frac{d}{n_i} \rightarrow \gamma_i > 1$, we know that the limit of  $\tr((\scov_i + \lambda I)^{-2}  \scov_i\Sigma_i) \frac{\sigma_i^2}{n_i}$ is given by \eqref{eqn:ridge-var} with $\gamma = \gamma_i$ and $\esd,\wesd$ be the empirical spectral distribution and weighted empirical spectral distribution of $\Sigma_i$ respectively.

In the case when $\Sigma_i = I$, using Theorem 1 of \cite{HastieMoRoTi19} we have $V_i(\hparamsolrt{i} |X) = \frac{\sigma_i^2}{n_i} \tr((\scov_i + \lambda I)^{-2}\scov_i\pinv\Sigma_i) \cp \frac{\sigma_i^2}{\gamma_i - 1}$.
\subsection{Proof of \Cref{thm:permaml}}\label{app:proof-maml}

On solving \eqref{eqn:param-maml} and \eqref{eqn:param-permaml}, the closed form of the estimators $\hparamsolmaml{0}$ and $\hparamsolmaml{i}$ is given by

   \begin{align*}
    \hparamsolmaml{0} &:= \argmin_\param \summach  \frac{p_j}{2n_j} \ltwo{X_j\left[\param - \frac{\alpha}{n_j} X_j^T(X_j \param - y_j)\right] -y_j}^2\\
    &= \argmin_\param \summach \frac{p_j}{2n_j}\ltwo{\left(I_n - \frac{\alpha}{n} X_j X_j^T\right)(X_j\param - y_j)}^2\\ 
    &= \mamlcov\inv \summach\frac{p_j}{n_j} X_j^TW_j^2 y_j
\end{align*}
where $W_j := I - \frac{\alpha}{n_j}X_jX_j^T$ and 
 \begin{align*}
    \hparamsolmaml{i} &:= \argmin_{\param} \ltwo{\hparamsolmaml{0} - \param} \quad s.t.\quad X_i \param = y_i\\
    &= (I - \scov_i\pinv \scov_i) \hparamsolmaml{0} + \scov_i\pinv\scov_i \param_i\opt + \frac{1}{n_i} \scov_i\pinv X_i^T \xi_i
\end{align*}

We now calculate the risk by splitting it into two parts as in \eqref{eqn:risk}, and then calculate the asymptotic bias and variance. 

\subsubsection*{Bias:}
\begin{align*}
    B(\hparamsolmaml{i}| X) &:= \lsigi{\E[\hparamsolmaml{i} | X]  - \param_i\opt}^2 = \ltwo{\Pi_i \left[ \mamlcov\inv \summach\frac{p_j}{n_j} X_j^TW_j^2 X_j (\param_j\opt - \param_i \opt) \right]}^2\\
    &= \lsigi{\Pi_i \left[\param_0\opt - \param_i\opt + \mamlcov\inv \summach\frac{p_j}{n_j} X_j^TW_j^2 X_j (\param_j\opt - \param_0 \opt) \right]}^2
\end{align*}

For simplicity, we let $\Delta_j := \param_j\opt - \param_0\opt$, and we define the events:
\begin{align}
    B_t &:= \left\{ \opnorm{\mamlcov\inv- \E\left[\summach p_j \frac{1}{n_j}X_j^T W_j^2 X_j\right]\inv} > t \right\}\\
    A_t &:= \left\{\lsigi{\mamlcov\inv \summach \frac{p_j}{n_j}X_j^TW_j^2 X_j \Delta_j} > t\right\}
\end{align}

The proof proceeds in the following steps:

\noindent\fbox{%
    \parbox{\textwidth}{%
    \textbf{Bias Proof Outline}
\begin{enumerate}[label={\bf Step \arabic*.},leftmargin=1.5cm]
    \item We first show for any $t > 0$, the $\P(B_t)\rightarrow 0$ as $d \rightarrow \infty$
    \item Then, we show for any $t > 0$, the $\P(A_t)\rightarrow 0$ as $d \rightarrow \infty$
    \item We show that for any $t \in (0, 1]$\footnote{this restriction is not necessary; its just to simplify the accounting} on event $A_t^c$, $B(\hparamsolmaml{i}| X) \leq \lsigi{\Pi_i [\param_0\opt - \param_i\opt]}^2 + c t$ and $B(\hparamsolmaml{i}| X) \geq \lsigi{\Pi_i [\param_0\opt - \param_i\opt]}^2 - c t$
    \item Show that $\lim_{d \rightarrow \infty} \P(| B(\hparamsolmaml{i}| X) - \lsigi{\Pi_i [\param_0\opt - \param_i\opt]}^2|\leq \varepsilon) = 1$ 
    \item Finally, using the asymptotic limit of  $\lsigi{\Pi_i [\param_0\opt - \param_i\opt]}^2$ from Theorem 1 of \cite{HastieMoRoTi19}, we get the result.
\end{enumerate}
}}

We now give the detailed proof:
\paragraph{Step 1}

Since  $\lambda_{\min{}} ( \E\left[\frac{1}{n_j}X_j^T W_j^2 X_j\right]) \geq \lambda_0$, it suffices to show by \Cref{lem:inverse-concentration} that the probability of
\begin{align*}
     C_t &:= \left\{ \opnorm{\summach p_j\left(\frac{1}{n_j}X_j^T W_j^2 X_j- \E\left[\frac{1}{n_j}X_j^T W_j^2 X_j\right]\right)} > t \right\}
\end{align*}
goes to 0 as $d, m \rightarrow \infty$ under \Cref{ass:asym}.
\begin{align}
    \nonumber \P(C_t) &= \P\left(\opnorm{\summach p_j\left[ \left[\scov_j - 2\alpha^2 \scov_j^2 + \alpha^2 \scov_j^3 \right] - \E\left[\frac{1}{n_j}X_j^T W_j^2 X_j\right]\right]} > t \right) \\
    \nonumber &\leq \P\left(\opnorm{\summach p_j\left(\scov_j - \mu_{1, j}\right)} > t/3\right) \\
    & + \P\left(\opnorm{2\alpha\summach p_j\left(\scov_j^2 - \mu_{2, j}\right)} > t/3\right) + \P\left(\opnorm{\alpha^2 \summach p_j\left(\scov_j^3 - \mu_{3, j}\right)} > t/3\right), \label{eqn:mamlbiasstep1}
\end{align}
where $\mu_{p, j} := \E[\scov_j^p]$. We repeatedly apply \Cref{lem:scov_concentration} for $p = 1,2,3$ to bound each of these three terms. It is clear that if $\E[\ltwo{\bx_{j, k}}^{6q}]^{1/(6q)} \lesssim \sqrt{d}$, $\opnorm{\E[\scov_j^{6}]} \leq C_3$ for some constant $C_3$, $(\log d)^{4q} \summach p_j^q n_j \rightarrow 0$, and $(\log d)^{q/2} \summach p_j^{q/2 + 1} \rightarrow 0$, then \eqref{eqn:mamlbiasstep1} goes to 0.

\paragraph{Step 2} 
\begin{align*}
    \P(A_t) &\leq \P(A_t \cap B_{c_1}^c) + \P(B_{c_1})\\
    &\leq \P\left(M\left(c_1 + \frac{1}{\lambda_0}\right) \ltwo{\summach \frac{p_j}{n_j}X_j^TW_j^2 X_j \Delta_j} > t \right) + \P(B_{c_1}) 
\end{align*}
From Step 1, we know that $\lim_{n\rightarrow\infty} \P(B_{c_1}) = 0$. The second inequality comes from the fact that $\ltwo{Ax} \leq \opnorm{A} \ltwo{x}$. To handle the first term, we 
use Markov's inequality.
\begin{align*}
    \P\left(c_2 \ltwo{\summach \frac{p_j}{n_j}X_j^TW_j^2 X_j \Delta_j} > t \right) 
    &\leq \frac{c_2^q}{t^q} \E\left[\ltwo{\summach \frac{p_j}{n_j}X_j^TW_j^2 X_j \Delta_j}^q\right] \\
    & \leq  \frac{(2c_2\sqrt{q})^q}{t^q} \E\left[\left(\summach \ltwo{\frac{p_j}{n_j}X_j^TW_j^2 X_j \Delta_j}^2\right)^{q/2}\right]\\
    &\leq  \frac{(2c_2\sqrt{q})^q}{t^q} \summach p_j\E\left[\left(p_j \ltwo{\frac{1}{n_j}X_j^TW_j^2 X_j \Delta_j}^2\right)^{q/2}\right]\\
    &= \frac{(2c_2\sqrt{q})^q}{t^q} \summach p_j^{q/2 + 1}\E\left[\ltwo{\scov_j(I - \alpha \scov_j)^2 \Delta_j}^{q}\right]\\
    & \leq \frac{(8c_2\sqrt{q})^q}{2t^q} \summach p_j^{q/2 + 1}\E\left[\opnorm{\scov_j}^q + \alpha^2\opnorm{ \scov_j}^{3q}\right] \E[\ltwo{\Delta_j}^{q}]\\
    &\le \frac{(8c_2\sqrt{q})^q}{2t^q} \summach p_j^{q/2 + 1} \left[\alpha^2 K (e\log d)^{3q} n
     _j\right] r_j^q,
\end{align*}
where the last step follows from \Cref{lem:scov_moment} and the final expression goes to 0 since $(\log d)^{3q}\summach p_j^{q/2 + 1} n_jr_j^q \rightarrow 0$.
\paragraph{Step 3, 4 and 5} are same as the bias calculation of proof of \Cref{thm:perfedavg}.

\subsubsection*{Variance:}
We let $\xi_i = [\xi_{i,1},\dots,\xi_{i,n}]$ denote the vector of noise.
\begin{align*}
    \nonumber V_i(\hparamsolmaml{i}; \param_i\opt|X) & = \tr(\cov(\hparamsolmaml{i}|X) \Sigma) = \E[\lsigi{\hparamsolmaml{i} -  \E[\hparamsolmaml{i} | X]}^2|X]\\
    \nonumber & = \E[\lsigi{\Pi_i \left[  \mamlcov\inv \summach \frac{p_j}{n_j} X_j^TW_j^2 \xi_j\right]  + \frac{1}{n_i} \scov_i\pinv X_i^T \xi_i }^2|X] \\
    &=  \underbrace{\summach \frac{p_j^2}{n^2_j} \tr\left(\Pi_i\Sigma_i \Pi_i \mamlcov\inv X_j^TW_j^4X_j \mamlcov\inv\right) \sigma_j^2}_{(i)} + \\ 
    & \underbrace{2 \tr\left(\scov_i\pinv \frac{X_i^TW_j^2X_i}{n_i^2} \mamlcov\inv \Pi_i\Sigma_i\right) \sigma_i^2}_{(ii)} +  \underbrace{\frac{1}{n_i^2} \tr(\scov_i\pinv X_i^TX_i \scov_i\pinv\Sigma_i)\sigma_i^2}_{(iii)}
\end{align*}

We now study the asymptotic behavior of each of the terms $(i),(ii)$ and $(iii)$ separately.

\paragraph{(i)}
Using the Cauchy Schwartz inequality on Schatten $p-$norms and using the fact that the nuclear norm of a projection matrix is at most $d$, we get
\begin{align*}
    & \summach \frac{p_j^2\sigma_j^2}{n_j} \tr\left(\Pi_i \Sigma_i \Pi_i \mamlcov\inv \scov_j(I - \alpha \scov_j)^4 \mamlcov\inv\right)\\ 
    &\le \summach \frac{p_j^2\sigma_j^2}{n_j} \matrixnorms{\Pi_i}_{1} \opnorm{\Sigma_i}\opnorm{\Pi_i}\opnorm{\mamlcov\inv} \opnorm{\scov_j(I - \alpha \scov_j)^4} \opnorm{\mamlcov\inv} \\
    & \le C_7 \sigma_{\max{}}^2\gammamax \left(\summach p_j^2\opnorm{\scov_j(I - \alpha \scov_j)^4}\right),
\end{align*}
where the last inequality holds with probability going to $1$ for some constant $C_7$ because $\P(C_t) \rightarrow 0$.
Lastly, we show that $\P\left(\summach p_j^2\opnorm{\scov_j(I - \alpha \scov_j)^4} > t\right) \rightarrow 0$. Using Markov's and Jensen's inequality, we have
\begin{align*}
    \P\left(\summach p^2_j\opnorm{ \scov_j(I - \scov_j)^4} > t\right) 
    & \le \frac{\E\left[\summach p_j^2\opnorm{\scov_j(I - \alpha \scov_j)^4}\right]^q}{t^q} \\
    & \le  \frac{\summach p_j^{q + 1} \E\left[\opnorm{\scov_j(I - \alpha \scov_j)^4}^q\right]}{t^q}\\
    & \le  \frac{\summach p_j^{q + 1} \E\left[\opnorm{\scov_j}\opnorm{(I - \alpha \scov_j)}^{4q}\right]}{t^q}\\
    & \le \frac{\summach p_j^{q + 1} \E\left[\opnorm{\scov_j}\left(\opnorm{I} +\alpha \opnorm{\scov_j}\right)^{4q}\right]}{t^q}\\
    & \le 2^{4q-1} \frac{\summach p_j^{q + 1} \E\left[\opnorm{\scov_j} +\alpha^4 \opnorm{\scov_j}^{5q}\right]}{t^q}
\end{align*}
Using \Cref{lem:scov_moment} and Markov's inequality, we have
\begin{align*}
    \P\left(\summach p^2_j\opnorm{\scov_j} > t\right) 
    & \le 2^{4q-1} \left(K_1\frac{\summach p_j^{q + 1} (e\log d)n_j}{t^q} + K_2\alpha ^4 \frac{\summach p_j^{q + 1} (e\log d)^{5q}n_j}{t^q}\right).
\end{align*}
Finally, since we know that $\summach p_j^{q + 1} (\log d)^{5q}n_j \rightarrow 0$, we have $\P\left(\summach p^2_j\opnorm{ \scov_j(I - \alpha \scov_j)^4} > t\right) \rightarrow 0$.

\paragraph{(ii)}
Using the Cauchy Schwartz inequality on Schatten $p-$norms and using the fact that the nuclear norm of a projection matrix is $d-n$, we get

\begin{align*}
    & 2 \tr\left(\scov_i\pinv \frac{X_i^TW_i^2X_i}{n_i^2} \mamlcov\inv \Pi_i\Sigma_i\right) \sigma_i^2\\
    &=  2 \tr\left(\Pi_i \Sigma_i \scov_i\pinv \frac{\scov_i(I - \scov_i)^2}{n_i} \mamlcov\inv\right) \sigma_i^2\\
    & \le  \frac{2p_i\sigma^2 }{n_i}\matrixnorms{\Pi_i}_1 \opnorm{\Sigma_i} \opnorm{\scov_i\pinv\scov_i} \opnorm{(I - \scov_i)^2} \opnorm{\mamlcov\inv}  \\
    & \leq C_4p_i,
\end{align*}
where the last inequality holds with probability going to $1$ for some constant $C_4$ because $\P(B_t) \rightarrow 0$ and using \Cref{ass:cov}. Since $p_i \to 0$, we have

\begin{align*}
    2 \tr\left(\scov_i\pinv \frac{X_i^TW_i^2X_i}{n_i^2} \mamlcov\inv \Pi_i\Sigma_i\right) \sigma_i^2 \to 0
\end{align*}

\paragraph{(iii)}
\[
\frac{1}{n_i^2} \tr(\scov_i\pinv X_i^TX_i \scov_i\pinv\Sigma_i)\sigma_i^2 = \frac{1}{n_i} \tr(\scov_i\pinv\Sigma_i)\sigma_i^2
\]
Using Theorem 3 of \cite{HastieMoRoTi19}, as  $d \rightarrow \infty$, such that $\frac{d}{n_i} \rightarrow \gamma_i > 1$, we know that the limit of  $\frac{\sigma_i^2}{n_i} \tr(\scov_i\pinv\Sigma_i)$ is given by \eqref{eqn:min-var} with $\gamma = \gamma_i$ and $\esd,\wesd$ be the empirical spectral distribution and weighted empirical spectral distribution of $\Sigma_i$ respectively.

In the case when $\Sigma_i = I$, using Theorem 1 of \cite{HastieMoRoTi19} we have $V_i(\hparamsolmaml{i}|X) = \frac{\sigma_i^2}{n_i} \tr(\scov_i\pinv) \rightarrow \frac{\sigma_i^2}{\gamma_i - 1}$.

\subsection{Proof of \Cref{thm:prox}}\label{app:proof-prox}

The solution to this minimization problem in \eqref{eqn:param-prox} is given by
\begin{align*}
    \hparamsolprox{0} = \param_0\opt + Q\inv\left(\summach p_jT_j\inv\scov_j \Delta_j + \summach p_jT_j\inv\frac{1}{n_j}X_j^T\xi_j\right),
\end{align*}

where $\Delta_j = \param_j\opt - \param_0\opt$, $T_j = \scov_j + \lambda I$ and $Q = I - \lambda\summach p_jT_j^{-1}$.
The personalized solutions are then given by 
\begin{align*}
    \hparamsolprox{i} = T_i\inv\left(\lambda \hparamsolprox{0} + \scov_i \param_i\opt + \frac{1}{n_i}X_i^T\xi_i\right)
\end{align*}

We now calculate the risk by splitting it into two parts as in \eqref{eqn:risk}, and then calculate the asymptotic bias and variance. 

\subsubsection*{Bias:}

Let $\Delta_j := \param_j\opt - \param_0\opt$, then we have

\begin{align*}
    B(\hparamsolprox{i}| X) := \lsigi{T_i\inv \left( \lambda \param_0\opt - \lambda \param_i\opt + \lambda Q\inv \left[ \summach p_jT_j\inv\scov_j\Delta_j\right]\right)}^2
\end{align*}

The idea is to show that the second term goes to $0$ and use results from \cite{HastieMoRoTi19} to find the asymptotic bias. To do this, we first define the events:

\begin{align*}
    C_t &:= \left\{ \opnorm{\summach p_j(T_j\inv - \E[T_j \inv])} > t \right\}\\
    A_t &:= \left\{\lsigi{ Q\inv \left[ \summach p_j T_j \inv\scov_j\Delta_j\right]} > t\right\}
\end{align*}

The proof proceeds in the following steps:

\noindent\fbox{%
    \parbox{\textwidth}{%
    \textbf{Bias Proof Outline}
\begin{enumerate}[label={\bf Step \arabic*.},leftmargin=1.5cm]
    \item We first show for any $t > 0$, the $\P(C_t)\rightarrow 0$ as $d \rightarrow \infty$
    \item Then, we show for any $t > 0$, the $\P(A_t)\rightarrow 0$ as $d \rightarrow \infty$.
    \item We show that for any $t \in (0, 1]$, $B(\hparamsolprox{i}| X) \leq \ltwo{T_i\inv [\lambda\param_0\opt - \lambda\param_i\opt]}^2 + c t$ and $B(\param, X) \geq \ltwo{T_i\inv [\lambda\param_0\opt - \lambda\param_i\opt]}^2 - c t$
    \item Show that $\lim_{d \rightarrow \infty} \P(| B(\hparamsolprox{i}| X) - \ltwo{T_i\inv [\lambda\param_0\opt - \lambda\param_i\opt]}^2|\leq \varepsilon) = 1$ 
    \item Finally, using the asymptotic limit of  $\ltwo{T_i\inv [\lambda\param_0\opt - \lambda\param_i\opt]}^2$ from Corollary 5 of \cite{HastieMoRoTi19}, we get the result.
\end{enumerate}
}}

We now give the detailed proof:
\paragraph{Step 1}
\begin{align*}
    \P(C_t) &= \P\left(\opnorm{\summach p_j(T_j\inv - \E[T_j \inv])} > t \right)
    \leq \frac{2^q\E\left[\opnorm{\summach \xi_j p_jT_j\inv}^q\right]}{t^q},
\end{align*}
We use Theorem A.1 from \cite{ChenGiTr12} to bound this object.
\begin{align*}
    \E\left[\opnorm{ \summach \xi_j T_j\inv}^q\right] &\leq \left[\sqrt{e \log d} \opnorm{\left(\summach p_j^2\E[T_j\inv]\right)^{1/2}} + (e\log d) (\E \max_j \opnorm{ p_jT_j\inv}^q)^{1/q}\right]^q\\
    &\leq 2^{q-1}\left(\sqrt{e \log d}^q \opnorm{\left(\summach p_j^2\E[(T_j\inv)^2]\right)^{1/2}}^q +  (e\log d)^q (\E \max_j \opnorm{ p_jT_j\inv}^q)\right)\\
    &\leq 2^{q-1}\left(\sqrt{e \log d}^q \opnorm{\summach p_j^2\E[(T_j\inv)^2]}^{q/2} + \frac{(e \log d)^q \max_j p_j^q }{\lambda^q}\right)\\
    &\leq 2^{q-1}\left(\sqrt{e \log d}^q \summach p_j^{q/2 + 1}\opnorm{(\E[(T_j\inv)^2])^{q/2}} + \frac{1}{\lambda^q}(e \log d)^q \summach p_j^q\right)\\
    &\leq \frac{2^{q-1}}{\lambda^q}\left((e\log d)^{q/2} \summach p_j^{q/2 + 1} + (e \log d)^q \summach p_j^q\right), 
\end{align*}
where we use the fact that $\opnorm{T_j\inv}  = \opnorm{(\scov_j + \lambda I)\inv} \le \frac{1}{\lambda}$ since $\scov_j$ is always positive semidefinite. Since $(\log d)^{q/2} \summach p_j^{q/2 + 1}$ and $(\log d)^q \summach p_j^q$, we get that $\P(C_t) \rightarrow 0$ for all $t > 0$.

\paragraph{Step 2} To prove this step, we will first use a helpful lemma,
\begin{lemma}\label{lem:prox-eig}
Suppose that $\Sigma = \E[\scov] \in \R^{d, d}$ has a spectrum supported on $[a , b]$ where $0 < a <b < \infty$. Further suppose that $\E\left[\opnorm{\scov^2}\right] \leq \tau$ and there exists an $R \geq b$ such that $\P(\lambda_{\max{}}(\scov) > R) \leq \frac{a^2}{8 \tau}$, then
\begin{align*}
    \opnorm{\E[(\scov + \lambda I) \inv]}  \leq \frac{1}{\lambda} \left(1 - \frac{a^3}{16 \tau(R + \lambda)} \right) \leq \frac{1}{\lambda}
\end{align*}
\end{lemma}
\begin{proof}
Fix an arbitrary vector $u \in \R^d$ with unit $\ell_2$ norm. We fix $\delta = a / 2 > 0$, we define the event $A \coloneqq \{ u^T \scov u \geq \delta\}$ and $B \coloneqq \{\lambda_{\max{}}(\scov) \leq R \}$
\begin{align*}
    u^T \E[(\scov + \lambda I) \inv] u \leq \E[\bindic{A \cap B}u^T (\scov + \lambda I) \inv u ] + \frac{1}{\lambda}(1 - \P(A \cap B))
\end{align*}

Let $\sigma_i^2$ and $v_i$ denote the $i$th eigenvalue and eigenvector of $\scov$ respectively sorted in descending order with respect to eigenvalue ($\sigma_1^2 \geq \sigma_2^2 \geq \ldots \geq \sigma_d^2$). 
On the event $A$, we have that $u^T (\scov + \lambda I)\inv u$ has value no larger than 
\begin{align*}
   \max_{\alpha\in \R^d}&\ \sum_{i=1}^d \frac{1}{\sigma_i^2 + \lambda} \alpha_i\\
   \text{s.t.   }& \alpha \geq 0\\
   & \ones^T \alpha = 1\\
   & \sum_{i=1}^d \sigma_i^2 \alpha_i \geq \delta
\end{align*}

The dual of this problem is
\begin{align*}
    \min_{\theta }&\ \max_{j \in [d]}\left\{\theta \sigma_j^2 + \frac{1}{\sigma_j^2 +\lambda }\right\} - \theta \delta\\
    \text{s.t.   }& \theta \geq 0
\end{align*}

It suffices to demonstrate that there exists a $\theta$ which satisfies the constraints of the dual and has objective value less than $\frac{1}{\lambda}$. We can verify that selecting $\theta = \frac{1}{\lambda(\sigma_1^2 + \lambda)}$ has an objective value of 
\begin{align*}
    \frac{1}{\lambda} - \frac{\delta}{\lambda(\sigma_1^2 + \lambda)}
\end{align*}
which is less than the desired $\frac{1}{\lambda}$. All that remains is to lower bound $\P(A \cap B) \geq \P(A) - \P(B^c)$. We know by Paley-Zygmund 
\begin{align*}
    \P(A) \geq \P\left(u^T \scov u \geq \frac{\delta}{a} u^T \Sigma u\right) \geq \P\left(u^T \scov u \geq \frac{1}{2} u^T \Sigma u\right) \geq \frac{(u^T\Sigma u)^2}{4 \E[(u^T \scov u)^2]} \geq \frac{a^2}{4\tau}
\end{align*}

Note that $a^2/4\tau <1$ because the second moment of a random variable is no smaller than the first moment squared of the random variable. Moreover, by construction, $R$ is large enough such that $\P(B^c) \leq \P(A) / 2$, thus,

\begin{align*}
    u^T \E[(\scov + \lambda I)\inv]u &\leq \frac{1}{\lambda}\left(1 - \frac{a}{2(R + \lambda)}\right) \frac{a^2}{8\tau} + \frac{1}{\lambda}\left(1 - \frac{a^2}{8\tau}\right)\\
    &= \frac{1}{\lambda} \left(1 - \frac{a^3}{16 \tau(R + \lambda)} \right)
\end{align*}

\end{proof}

Recall that we have the assumptions that for sufficiently large $m$, for all $j \in [m]$ we have $\Sigma_j$ has a spectrum supported on $[a , b]$ where $a = 1/M$ and $b = M$ and $\E\left[\opnorm{\scov_j^2}\right] \leq \tau_3$. Moreover, since we have the assumption that there exists an $R \geq b$ such that $\limsup_{m \rightarrow \infty} \sup_{j \in [m]} \P(\lambda_{\max{}}(\scov_j) > R) \leq \frac{a^2}{16 \tau_3}$, by \Cref{lem:prox-eig} there exists and $1 > \varepsilon>0$ such that for sufficiently large $m$, for all $j\in[m]$, $\opnorm{\E[(\scov_j + \lambda I)\inv]} \leq \frac{1 - \varepsilon}{\lambda}$.
\begin{align*}
    \P(A_t) \leq \P(A_t \cap C_{c_1}^c) + \P(C_{c_1})
\end{align*}
Since we know $\P(C_{c_1}) \rightarrow 0$, it suffices to bound the first term.
\begin{align*}
    \P(A_t \cap C_{c_1}^c) &\leq \P\left( \sqrt{\covop}\opnorm{Q\inv}\ltwo{\left[\summach p_j T_j\inv\scov_j\Delta_j\right]} >t \cap C_{c_1}^c \right)\\
    &= \P\left(\sqrt{\covop}\left(1 - \opnorm{ \lambda \summach p_jT_j\inv}\right)\inv \ltwo{\left[\summach p_j T_j\inv\scov_j\Delta_j\right]} >t \cap C_{c_1}^c \right)\\
    &\leq \P\left( \sqrt{\covop}\left(1 - \opnorm{\lambda\summach p_j\E[T_j\inv] + E_{c_1}}\right)\inv\ltwo{\left[\summach p_j T_j\inv\scov_j\Delta_j\right]} >t \right)\\
    &\leq \P\left( \sqrt{\covop}\left(1 - \lambda\summach p_j\opnorm{\E[T_j\inv]} - c_1\right)\inv\ltwo{\left[\summach p_j T_j\inv\scov_j\Delta_j\right]} > t \right)
\end{align*}
where we used Jensen's inequality in the last step. $E_{c_1}$ is a matrix error term which on the event $C_{c_1}^c$ has operator norm bounded by $c_1$. As discussed, we have that $\summach p_j\opnorm{\E[T_j\inv]}$ is less than $\frac{1 - \varepsilon}{\lambda}$, which shows there exists a constant $c_2$, such that $\sqrt{\covop}(1 - \lambda\summach p_j\opnorm{\E[T_j\inv]} - c_1)\inv < c_2$. Now, we have, using \Cref{lem:vec-khintchine},
\begin{align*}
    \P\left(c_2\ltwo{\summach p_jT_j\inv\scov_j \Delta_j } > t\right) &\leq \frac{c_2^q\E\left[\ltwo{\summach p_j T_j\inv\scov_j \Delta_j }^q\right]}{t^q} \leq \frac{(2c_2\sqrt{q})^q \E\left[\left(\summach \ltwo{p_jT_j\inv\scov_j \Delta_j}^2\right)^{q/2}\right]}{t^q }
\end{align*}
Using Jensen's inequality and the definition of operator norm, we have
\begin{align*}
       \frac{(2c_2\sqrt{q})^q \E\left[\left(\summach p_j^2 \ltwo{T_j\inv\scov_j \Delta_j}^2\right)^{q/2}\right]}{t^q} &= \frac{(2c_2\sqrt{q})^q \E\left[\left(\summach p_j \cdot p_j \ltwo{T_j\inv\scov_j \Delta_j}^2\right)^{q/2}\right]}{t^q}\\
       & \leq  \frac{(2c_2\sqrt{q})^q \summach p_j^{q/2 + 1} \E\left[\ltwo{T_j\inv\scov_j \Delta_j}^{q}\right]}{t^q}  \\
    &\leq \frac{(2c_2\sqrt{q})^q \summach p_j^{q/2 + 1}  \E\left[\opnorm{T_j\inv}^q\right]\E\left[\opnorm{\scov_j}^q\right] \E\left[ \ltwo{\Delta_j}^{q}\right]}{t^q}
\end{align*}

Lastly, we can bound this using \Cref{lem:scov_moment} as follows and using the fact that $\opnorm{T_j\inv} \le \frac{1}{\lambda}$.

\begin{align*}
   \P\left(c_2\ltwo{\summach p_jT_j\inv\scov_j \Delta_j } > t\right) &\le\frac{(2c_2\sqrt{q})^q \summach p_j^{q/2 + 1}  K n_j (e\log d)^q r_j^q}{\lambda^q t^q} \rightarrow 0,
\end{align*}
using  $(\log d)^q\summach n_j p_j^{q/2 + 1}  \rightarrow 0$

\paragraph{Step 3} 
For any $t \in (0, 1]$, on the event $A_t^c$, we have that 
\[B(\hparamsolprox{i}| X) = \lsigi{T_i\inv[\lambda\param_0\opt - \lambda\param_i\opt + E]}^2\]
for some vector $E$ where we know $\lsigi{E} \leq t$ (which means $\ltwo{E} \leq t \sqrt{\covop}$).

We can form the bounds 
\begin{align*}
\lsigi{T_i\inv[\lambda\param_0\opt - \lambda\param_i\opt + E]}^2 &\leq  \lambda^2\lsigi{T_i\inv[\param_0\opt - \param_i\opt]}^2 + \lsigi{T_i\inv E}^2 + 2\lambda\lsigi{T_i\inv E} \lsigi{T_i\inv[\param_0\opt - \param_i\opt]}\\
&\leq\lambda^2 \lsigi{T_i\inv[\param_0\opt - \param_i\opt]}^2 + \lambda^{-2}t^2 \covop^2 + 2t \lambda^{-1}r_i^2 \covop^{3/2}\\
\lsigi{T_i\inv[\lambda\param_0\opt -\lambda \param_i\opt + E]}^2 &\geq  \lambda^2\lsigi{T_i\inv[\param_0\opt - \param_i\opt]}^2 + \lsigi{T_i\inv E}^2 - 2\lambda\lsigi{T_i\inv E} \lsigi{T_i\inv[\param_0\opt - \param_i\opt]}\\
&\geq \lambda^2\lsigi{T_i\inv[\param_0\opt - \param_i\opt]}^2 - 2t\lambda^{-1} r_i^2 \covop^{3/2}
\end{align*}

Since $t\in (0, 1]$, we have that $t^2 \leq t$ and thus we can choose $c = \lambda^{-2}M^2 + 2r_i^2 \lambda\inv M^{3/2}$

\paragraph{Step 4} Reparameterizing $\varepsilon := c t$, we have that for any $\varepsilon > 0$
    \begin{align*}
        \lim_{n \rightarrow \infty} \P(| B(\hparamsolprox{i}, X) - \lsigi{T_i\inv [\param_0\opt - \param_i\opt]}^2|\leq \varepsilon) &\geq \lim_{n \rightarrow \infty} \P(| B(\hparamsolprox{i}| X) - \lsigi{T_i\inv [\param_0\opt - \param_i\opt]}^2|\leq \varepsilon \wedge c) \\
        \geq \lim_{n \rightarrow \infty} \P(A_{\frac{\varepsilon}{c} \wedge 1}^c) = 1
    \end{align*}
    
\paragraph{Step 5}  Using Theorem 6 of \cite{HastieMoRoTi19}, as  $d \rightarrow \infty$, such that $\frac{d}{n_i} \rightarrow \gamma_i > 1$, we know that the limit of  $\lambda^2\lsigi{(\scov_i + \lambda I)\inv[\param_0\opt - \param_i\opt]}^2$ is given by \eqref{eqn:ridge-bias} with $\gamma = \gamma_i$ and $\esd,\wesd$ be the empirical spectral distribution and weighted empirical spectral distribution of $\Sigma_i$ respectively.

In the case when $\Sigma_i = I$, using Corollary 5 of \cite{HastieMoRoTi19} we have $B_i(\hparamsolprox{i}|X ) = \ltwo{\Pi_i [\param_0\opt - \param_i\opt]}^2 \rightarrow r_i^2\lambda^2 m'_i(-\lambda)$.

\subsubsection*{Variance}

\begin{align*}
    \var(\hparamsolprox{i} | X) &= \E\left[\lsigi{T_i^{-1}\left(\lambda Q^{-1}\left[\summach p_jT_j^{-1} \frac{1}{n}X_j^T\xi_j\right] + \frac{1}{n_i}X_i^T\xi_i\right)}^2\right] \\
    & =  \underbrace{\summach \frac{\lambda^2 p_j^2}{n_j} \tr\left(T_i^{-1}\Sigma_i T_i^{-1} Q\inv T_j\scov_j T_jQ\inv\right) \sigma_j^2 }_{(i)}\\  
    & +\underbrace{\frac{2\lambda\sigma_i^2p_i}{n_i}2\tr\left(\Sigma_iT_i^{-1}Q^{-1} T_i\inv \scov_iT_i^{-1}\right)}_{(ii)} \\  
    & + \underbrace{\tr\left(T_i^{-1}\Sigma_iT_i^{-1}  \scov_i\right) \frac{\sigma_i^2}{n_i}}_{(iii)} 
\end{align*}
We now study the asymptotic behavior of each of the terms $(i),(ii)$ and $(iii)$ separately. In these steps, we will have to bound $\opnorm{Q\inv}$. To do this, we observe that there exists a sufficiently large constant $t$ such that the following statement is true.
\begin{align*}
    \P(\opnorm{Q\inv} > t) &= \P(\opnorm{Q\inv} > t \cap C_{c_1}^c) + \P(C_{c_1})\\ 
    &= \P\left(\left(1 - \opnorm{ \lambda \summach p_jT_j\inv}\right)\inv >t \cap C_{c_1}^c \right) + o(1)\\
    &\leq \P\left(\left(1 - \opnorm{\lambda\summach p_j\E[T_j\inv] + E_{c_1}}\right)\inv >t \right)+ o(1)\\
    &\leq \P\left( \left(1 - \lambda\summach p_j\opnorm{\E[T_j\inv]} - c_1\right)\inv > t \right)+ o(1)\\
    &\leq o(1).
\end{align*}
This is true because of \Cref{lem:prox-eig}.

\paragraph{(i)} Using the Cauchy Schwartz inequality on Schatten $p-$norms and using the high probability bounds from the bias proof, we get that for some constant $C_8$, the following holds with probability going to 1.
\begin{align*}
    &\summach \frac{\lambda^2 p_j^2}{n_j} \tr\left(T_i^{-1}\Sigma_iT_i^{-1} Q\inv T_j\scov_j T_jQ\inv\right) \sigma_j^2\\
    & \le \summach \frac{\sigma_j^2\lambda^2 p_j^2}{n_j} \matrixnorms{T_i^{-1}}_1 \opnorm{T_i^{-1}} \opnorm{\Sigma_i} \opnorm{Q\inv} \opnorm{T_j}\opnorm{\scov_j} \opnorm{T_j}\opnorm{Q\inv} \\
    & \le \summach \frac{\covop C_8\sigma_j^2\lambda^2 p_j^2d}{n_j} \opnorm{\scov_j}\\
    & \le \gamma_{\max}\covop C_8\sigma_{\max{}}^2\lambda^2\summach p_j^2\opnorm{\scov_j}
\end{align*}
$\opnorm{Q\inv}$ is upper bounded by some constant as shown above. We use the same technique as in proof of the variance of \Cref{thm:perfedavg} calculation from \Cref{app:proof-perfedavg} to show that $\summach p_j^2\opnorm{\scov_j} \cp 0$.

\paragraph{(ii)}
Using the Cauchy Schwartz inequality on Schatten $p-$norms and using the high probability bounds from the bias proof, we get that for some constant $C_9$, the following holds with probability going to 1.
\begin{align*}
    \frac{2\lambda\sigma_i^2p_i}{n_i}2\tr\left(\Sigma_i T_i^{-1}Q^{-1} T_i\inv \scov_iT_i^{-1}\right) & \le  \frac{2p_i\lambda\sigma_i^2 }{n_i}\matrixnorms{T_i^{-1}}_1 \opnorm{T_i\inv} \opnorm{\Sigma_i}  \opnorm{Q\inv}\opnorm{T_i\inv}\opnorm{\scov_i} \\
    & \le \frac{\covop C_9\sigma_i^2 d p_i}{\lambda n_i}.
\end{align*}
$\opnorm{Q\inv}$ is upper bounded by some constant as shown above. Moreover, since $p_i \rightarrow 0$, we have $ \frac{2\lambda\sigma^2p_i}{n_i}2\tr\left(T_i^{-1}Q^{-1} T_i\inv \scov_iT_i^{-1}\right) \cp 0$

\paragraph{(iii)}

Using Theorem 3 of \cite{HastieMoRoTi19}, as  $d \rightarrow \infty$, such that $\frac{d}{n_i} \rightarrow \gamma_i > 1$, we know that the limit of  $\tr((\scov_i + \lambda I)^{-2}  \scov_i\Sigma_i) \frac{\sigma_i^2}{n_i}$ is given by \eqref{eqn:ridge-var} with $\gamma = \gamma_i$ and $\esd,\wesd$ be the empirical spectral distribution and weighted empirical spectral distribution of $\Sigma_i$ respectively.

In the case when $\Sigma_i = I$, using Theorem 1 of \cite{HastieMoRoTi19} we have $V_i(\hparamsolprox{i} |X) = \frac{\sigma_i^2}{n_i} \tr((\scov_i + \lambda I)^{-2}\scov_i\pinv\Sigma_i) \cp \frac{\sigma_i^2}{\gamma_i - 1}$.
\subsection{Proof of \Cref{cor:naive}}\label{app:proof-naive}

We first prove that $\rho_i \geq r_i$. If we let $\omega \in \Omega$ be the probability space associated with $\param_i\opt$, we claim that $\<\param_0\opt, \param_i\opt(\omega) - \param_0\opt \> = 0$ for all (not just a.s.) $\omega \in \Omega$. For the sake of contradiction, suppose that there exists $\omega'$ such that $\<\param_0\opt, \param_i\opt(\omega') - \param_0\opt \> > 0$. If there exists $\omega'' \neq \omega'$ such that $\<\param_0\opt, \param_i\opt(\omega'') - \param_0\opt \> \neq \<\param_0\opt, \param_i\opt(\omega') - \param_0\opt \>$, then observe that the following equalities are true:
\begin{align*}
    \ltwo{\param_i\opt(\omega') - \param_0\opt}^2 + \ltwo{\param_0\opt}^2 &+ 2\< \param_0\opt, \param_i\opt(\omega') - \param_0\opt \> =  \ltwo{\param_i\opt(\omega')}^2
    = \rho_i^2 = \ltwo{\param_i\opt(\omega'')}^2\\
    &= \ltwo{\param_i\opt(\omega'') - \param_0\opt}^2 + \ltwo{\param_0\opt}^2 + 2\< \param_0\opt, \param_i\opt(\omega'') - \param_0\opt \>
\end{align*}

In looking at the first and last term, we see that this implies $\< \param_0\opt, \param_i\opt(\omega') - \param_0\opt \> = \< \param_0\opt, \param_i\opt(\omega'') - \param_0\opt \>$, which is a contradiction. Consider the other possibility that for all $\omega'' \in \Omega$, $\<\param_0\opt, \param_i\opt(\omega'') - \param_0\opt \> = \<\param_0\opt, \param_i\opt(\omega') - \param_0\opt \> > 0$. This would imply that $\E[\<\param_0\opt, \param_i\opt(\omega) - \param_0\opt \>] > 0$ which is a contradiction, since $\E[\param_i\opt(\omega)] = \param_0\opt$. Now since $\<\param_0\opt, \param_i\opt(\omega) - \param_0\opt \> = 0$, we have that $\rho_i^2 = r_i^2 + \ltwo{\param_0\opt}^2 \geq r_i^2$.


To show the result for $\hparamsolfa{0}$, recall \cref{eqn:paramfa_expanded}. The bias associated with this estimator is 
\begin{align*}
    B_i(\hparamsolfa{0}|X ) = \ltwo{\param_0\opt - \param_i\opt + (\summach p_j \scov_j)\inv \summach p_j \scov_j (\param_j\opt - \param_0\opt)}^2
\end{align*}
Using the bias proof of \Cref{thm:perfedavg}, treating $\Pi_i$ as the identity, we know that this quantity converges in probability to $r_i^2$.
Showing the variance of $\hparamsol{0}{FA}{0}$ of goes to $0$ follows directly from part $(i)$ of the variance proof of \Cref{thm:perfedavg}, again treating $\Pi_i$ as the identity.

The result regarding the estimator $\hparam^N_i$ is a direct consequence of Theorem 1 from \cite{HastieMoRoTi19}. The result regarding the estimator $\hparamsol{i}{N}{\lambda}$ is a direct consequence of Corollary 5 from \cite{HastieMoRoTi19}.

\subsection{Proof that \rtfa\ has lower risk than FedAvg}\label{sec:rtfabetterfedavg}
We show that \rtfa\ with optimal hyperparameter has lower risk than FedAvg by using the fact that $(1 - 1/\gamma)^2 \leq (1 + 1/\gamma)^2$ for $\gamma \geq 1$ and completing the square:
\begin{align*}
    \risk_i(\hparam_i^R(\lambda\opt); \param_i\opt | X) &= \frac{1}{2}\left[r_i^2\left(1 - \frac{1}{\gamma}\right) - \sigma_i^2 + \sqrt{r_i^4 \left(1 - \frac{1}{\gamma}\right)^2 + \sigma_i^4 + 2\sigma_i^2 r_i^2\left(1 + \frac{1}{\gamma}\right) } \right] \\
    &\leq r_i^2 = \risk_i(\hparamsolfa{0}; \param_i\opt|X).
\end{align*}
 \newpage
 \section{Algorithm implementations} \label{app:algs}

In this section, we give all steps of the exact algorithms used to implement all algorithms in the experiments section. 

\begin{algorithm}[tb]
   \caption{Naive local training}
   \label{alg:local}
\begin{algorithmic}[1]
    \REQUIRE m: number of users, K: epochs
    \FOR{$i\gets1$ to $m$}
    \STATE Each client runs K epochs of SGM with personal stepsize $\alpha$
    \ENDFOR
\end{algorithmic}
\end{algorithm}

\begin{algorithm}[tb]
   \caption{Federated Averaging \cite{McMahanMoRaHaAr17}}
   \label{alg:fedavg}
\begin{algorithmic}[1]
    \REQUIRE $\rounds$: Communication Rounds, $\sampuser$: Number of users sampled each round, $\localupdates$: Number of local update steps, $\hparamsolfa{0,0}$: Initial iterate for global model
    \FOR{$r\gets0$ to $\rounds - 1$}
    \STATE Server samples a subset of clients $\mc{S}_r$ uniformly at random such that $|\mc{S}_r| = \sampuser$
    \STATE Server sends $\hparamsolfa{0,r}$ to all clients in $\mc{S}_r$
    \FOR{$i \in \mc{S}_r$}
    \STATE Set $\hparamsolfa{i,r+1,0} \gets \hparamsolfa{0,r}$
    \FOR{$k \gets 1$ to $\localupdates$}
    \STATE Sample a batch $\mc{D}_k^i$ of size $\batchsize$ from user $i$'s data $\mc{D}_i$
    \STATE Compute Stochastic Gradient $g(\hparamsolfa{i,r+1,k-1};\mc{D}_k^i) = \frac{1}{B}\sum_{S \in \mc{D}_k^i} \grad F(\hparamsolfa{i,r+1,k-1};S)$\\
    Set $\hparamsolfa{i,r+1,k} \gets \hparamsolfa{i,r+1,k-1} - \eta g(\hparamsolfa{i,r+1,k-1};\mc{D}_k^i)$
    \ENDFOR
    \STATE Client $i$ sends $\hparamsolfa{i,r+1,K}$ back to the server.
    \ENDFOR
    \STATE Server updates the central model using $\hparamsolfa{0,r+1} = \sum_{j = 1}^D \frac{n_j}{\sum_{j = 1}^D n_j}\hparamsolfa{i,r+1,K}$.
    \ENDFOR
    \STATE {\bf return} {$\hparamsolfa{0,R}$}
\end{algorithmic}
\end{algorithm}

\begin{algorithm}[tb]
   \caption{\ftfa}
   \label{alg:ftfa-detail}
\begin{algorithmic}[1]
    \REQUIRE $\persiter$: Personalization iterations
    \STATE Server sends $\hparamsolfa{0} = \hparamsolfa{0,R}$ (using \Cref{alg:fedavg} with stepsize $\eta$) to all clients
    \FOR{$i\gets1$ to $m$}
    \STATE Run $P$ steps of SGM on $\riskhat_i(\cdot)$ using $\hparamsolfa{0}$ as initial point with learning rate $\alpha$ and output $\hparamsolfa{i,P}$ 
    \ENDFOR
    \STATE {\bf return} $\hparamsolfa{i,P}$ 
\end{algorithmic}
\end{algorithm}

\begin{algorithm}[tb]
   \caption{\rtfa}
   \label{alg:rtfa}
\begin{algorithmic}[1]
    \REQUIRE $\persiter$: Personalization iterations
    \STATE Server sends $\hparamsolfa{0} = \hparamsolfa{0,R}$ (using \Cref{alg:fedavg} with stepsize $\eta$) to all clients
    \FOR{$i\gets1$ to $m$}
    \STATE Run $P$ steps of SGM on $\riskhat_i(\param) + \frac{\lambda}{2}\ltwo{\param - \hparamsolfa{0}}^2$ with learning rate $\alpha$ and output $\hparamsolfa{i,P}$
    \ENDFOR
    \STATE {\bf return} $\hparamsolfa{i,P}$ 
\end{algorithmic}
\end{algorithm}

\begin{algorithm}[tb]
   \caption{\mamlfl-HF \cite{FallahMoOz20}}
   \label{alg:maml-hf-app}
\begin{algorithmic}[1]
    \REQUIRE $\rounds$: Communication Rounds, $\sampuser$: Number of users sampled each round, $\localupdates$: Number of local update steps, $\hparamsolmaml{0,0}$: Initial iterate for global model
    \FOR{$r\gets0$ to $\rounds - 1$}
    \STATE Server samples a subset of clients $\mc{S}_r$ uniformly at random such that $|\mc{S}_r| = \sampuser$
    \STATE Server sends $\hparamsolmaml{0,r}$ to all clients in $\mc{S}_r$
    \FOR{$i \in \mc{S}_r$}
    \STATE Set $\hparamsolmaml{i,r+1,0} \gets \hparamsolmaml{0,r}$
    \FOR{$k \gets 1$ to $\localupdates$}
    \STATE Sample a batch $\mc{D}_k^i$ of size $\batchsize$ from user $i$'s data $\mc{D}_i$
    \STATE Compute Stochastic Gradient $g(\hparamsolmaml{i,r+1,k-1};\mc{D}_k^i) = \frac{1}{B}\sum_{S \in \mc{D}_k^i} \grad F(\hparamsolmaml{i,r+1,k-1};S)$\\
    Set $\underline{\hparamsolmaml{i,r+1,k}} \gets \hparamsolmaml{i,r+1,k-1} - \alpha g(\hparamsolmaml{i,r+1,k-1};\mc{D}_k^i)$
    \ENDFOR
    \STATE Client $i$ sends $\hparamsolmaml{i,r+1,K}$ back to the server.
    \ENDFOR
    \STATE Server updates the central model using $\hparamsolmaml{0,r+1} = \sum_{j = 1}^D \frac{n_j}{\sum_{j = 1}^D n_j}\hparamsolmaml{i,r+1,K}$.
    \ENDFOR
    \STATE Server sends $\hparamsolmaml{0,R}$ to all clients
    \FOR{$i \gets 1$ to $m$}
    \STATE Run $P$ steps of SGM on $\riskhat_i(\cdot)$ using $\hparamsolmaml{0}$ as initial point with learning rate $\alpha$ and output $\hparamsolmaml{i,P}$ 
    \ENDFOR
    \STATE {\bf return} {$\hparamsolmaml{0,R}$}
\end{algorithmic}
\end{algorithm}

\begin{algorithm}[tb]
   \caption{pFedMe \cite{DinhTrNg20}}
   \label{alg:pfedme}
\begin{algorithmic}[1]
    \REQUIRE $\rounds$: Communication Rounds, $\sampuser$: Number of users sampled each round, $\localupdates$: Number of local update steps, $\hparamsolprox{0,0}$: Initial iterate for global model
    \FOR{$r\gets0$ to $\rounds - 1$}
    \STATE Server samples a subset of clients $\mc{S}_r$ uniformly at random such that $|\mc{S}_r| = \sampuser$
    \STATE Server sends $\hparamsolprox{0,r}$ to all clients in $\mc{S}_r$
    \FOR{$i \in \mc{S}_r$}
    \STATE Set $\hparamsolprox{i,r+1,0} \gets \hparamsolprox{0,r}$
    \FOR{$k \gets 1$ to $\localupdates$}
    \STATE Sample a batch $\mc{D}_k^i$ of size $\batchsize$ from user $i$'s data $\mc{D}_i$
    \STATE Compute $\theta_i(\hparamsolprox{i,r+1,k-1}) = \argmin_\param \frac{1}{B}\sum_{S \in \mc{D}_k^i} \grad F(\param;S) + \frac{\lambda}{2}\ltwo{\param - \hparamsolprox{i,r+1,k-1}}^2 $
    \STATE Set $\hparamsolprox{i,r+1,k} \gets \hparamsolprox{i,r+1,k-1} - \eta \lambda(\hparamsolprox{i,r+1,k-1} - \theta_i(\hparamsolprox{i,r+1,k-1}))$
    \ENDFOR
    \STATE Client $i$ sends $\hparamsolprox{i,r+1,K}$ back to the server.
    \ENDFOR
    \STATE Server updates the central model using $\hparamsolprox{0,r+1} = (1 - \beta)\hparamsolprox{0,r} + \beta\sum_{j = 1}^D \frac{n_j}{\sum_{j = 1}^D n_j}\hparamsolprox{i,r+1,K}$.
    \ENDFOR
    \STATE {\bf return} {$\hparamsolprox{0,R}$}
\end{algorithmic}
\end{algorithm}

\newpage
 \newpage
 \section{Experimental Details} \label{app:expts}
\subsection{Dataset Details}

In this section, we provide detailed descriptions on datasets and how they were divided into users. We perform experiments on federated versions of the Shakespeare \cite{McMahanMoRaHaAr17}, CIFAR-100 \cite{KrizhevskyHi09}, EMNIST \cite{CohenAsTaSc17}, and Stack Overflow \cite{TensorflowfederatedSO19} datasets. We download all datasets using FedML APIs \cite{HeLiSoWaWaVeSi20} which in turn get their datasets from \cite{TensorflowfederatedSO19}. For each dataset, for each client, we divide their data into train, validation and test sets with roughly a $80\%, 10\%, 10\%$ split.  The information regarding the number of users in each dataset, dimension of the model used, and the division of all samples into train, validation and test sets is given in \Cref{tbl:data}.

\begin{table}[H]
\begin{center}
\begin{tabular}{ |c|c|c|c|c|c|c| } 
\hline
 Dataset & Users & Dimension & Train & Validation  & Test & Total Samples  \\ 
 \hline
 CIFAR 100 & 600 & 51200 & 48000 & 6000 & 6000 & 60000\\ 
 Shakespeare & 669 & 23040 & 33244 & 4494 & 5288 & 43026 \\
 EMNIST & 3400 & 31744 & 595523 & 76062 & 77483 & 749068 \\
 Stackoverflow-nwp & 300 & 960384 & 155702 & 19341 & 19736 & 194779\\ 
 \hline
\end{tabular}
\end{center}
\caption{Dataset Information}
\label{tbl:data}
\end{table}

\paragraph{Shakespeare}
Shakespeare is a language modeling dataset built using the works of William Shakespeare and the clients correspond to a speaking role with at least two lines. The task here is next character prediction. The way lines are split into sequences of length 80, and the description of the vocabulary size is same as \cite{ReddiChZaGaRuKoKuMc21} (Appendix C.3). Additionally, we filtered out clients with less than 3 sequences of data, so as to have a train-validation-test split for all the clients. This brought the number of clients down to 669. More information on sample sizes can be found in \Cref{tbl:data}. The models trained on this dataset are trained on two Tesla P100-PCIE-12GB GPUs.

\paragraph{CIFAR-100}
CIFAR-100 is an image classification dataset with 100 classes and each image consisting of 3 channels of 32x32 pixels. We use the clients created in the Tensorflow Federated framework \cite{TensorflowfederatedSO19} ---  client division is described in Appendix F of \cite{ReddiChZaGaRuKoKuMc21}. Instead of using 500 clients for training and 100 clients for testing as in \cite{ReddiChZaGaRuKoKuMc21}, we divided each clients' dataset into train, validation and test sets and use all the clients' corresponding data for training, validation and testing respectively. The models trained on this dataset are trained on two Titan Xp GPUs. 

\paragraph{EMNIST}
EMNIST contains images of upper and lower characters of the English language along with images of digits, with total 62 classes. The federated version of EMNIST partitions images by their author providing the dataset natural heterogenity according to the writing style of each person. The task is to classify images into the 62 classes. As in other datasets, we divide each clients' data into train, validation and test sets randomly. The models trained on this dataset are trained on two Tesla P100-PCIE-12GB GPUs. 

\paragraph{Stack Overflow}

Stack Overflow is a language model consisting of questions and answers from the StackOverflow website. The task we focus on is next word prediction. As described in Appendix C.4 of \cite{ReddiChZaGaRuKoKuMc21}, we also restrict to the 10000 most frequently used words, and perform padding/truncation to ensure each sentence to have 20 words. Additionally, due to scalability issues, we use only a sample of 300 clients from the original dataset from \cite{TensorflowfederatedSO19} and for each client, we divide their data into train, validation and test sets randomly. The models trained on this dataset are trained on two Titan Xp GPUs.

\subsection{Hyperparameter Tuning Details}
\subsubsection{Pretrained Model}
We now describe how we obtain our pretrained models. First, we train and hyperparameter tune a neural net classifier on the train and validation sets in a non-federated manner. The details of the hyperparameter sweep are as follows:

\paragraph{Shakespeare} For this dataset we use the same neural network architecture as used for Shakespeare in \cite{McMahanMoRaHaAr17}. It has an embedding layer, an LSTM layer and a fully connected layer. We use the \href{https://pytorch.org/docs/stable/_modules/torch/optim/lr_scheduler.html#StepLR}{StepLR} learning rate scheduler of PyTorch , and we hyperparameter tune over the step size [0.0001, 0.001, 0.01, 0.1, 1] and the learning rate decay gamma [0.1, 0.3, 0.5] for 25 epochs with a batch size of 128.

\paragraph{CIFAR-100} For this dataset we use the Res-Net18 architecture \cite{HeZhReSu16}. We perform the standard preprocessing for CIFAR datasets for train, validation and test data. For training images, we perform a random crop to shape $(32,32,3)$ with padding size $4$, followed by a horizontal random flip. For all training, validationn and testing images, we normalize each image according to their mean and standard deviation. We use the hyperparameters specified by \cite{weiaicunzai20} to train our nets for 200 epochs.

\paragraph{EMNIST} For this dataset, the architecture we use is similar to that found in \cite{ReddiChZaGaRuKoKuMc21}; the exact architecture can be found in our code. We use the \href{https://pytorch.org/docs/stable/_modules/torch/optim/lr_scheduler.html#StepLR}{StepLR} learning rate scheduler of PyTorch, and we hyperparameter tune over the step size [0.0001, 0.001, 0.01, 0.1, 1] and the learning rate decay gamma [0.1, 0.3, 0.5] for 25 epochs with a batch size of 128.

\paragraph{Stackoverflow} For this dataset we use the same neural network architecture as used for Stack Overflow next word prediction task in \cite{ReddiChZaGaRuKoKuMc21}. We use the \href{https://pytorch.org/docs/stable/_modules/torch/optim/lr_scheduler.html#StepLR}{StepLR} learning rate scheduler of PyTorch, and we hyperparameter tune over the step size [0.0001, 0.001, 0.01, 0.1, 1] and the learning rate decay gamma [0.1, 0.3, 0.5] for 25 epochs with a batch size of 128.

\subsubsection{Federated Last Layer Training}
After selecting the best hyperparameters for each net, we pass our data through said net and store their representations (i.e., output from penultimate layer). These representations are the data we operate on in our federated experiments. 

Using these representations, we do multi-class logistic regression with each of the federated learning algorithms we test; we adapt and extend this \href{https://github.com/CharlieDinh/pFedMe}{code base} \cite{DinhTrNg20} to do our experiments. 
For all of our algorithms, the number of global iterations $R$ is set to 400, and the number of local iterations $K$ is set to 20. 
The number of users sampled at global iteration $r$, $D$, is set to 20. The batch size per local iteration, $B$, is 32. The random seed is set to 1. For algorithms \ftfa, \rtfa, \mamlfl-FO, and \mamlfl-HF, we set the number of personalization epochs $P$ to be 10. We fix some hyperparameters due to computational resource restrictions and to avoid conflating variables; we choose to fix these ones out of precedence, see experimental details of \cite{ReddiChZaGaRuKoKuMc21}. We now describe what parameters we hyperparameter tune over for each algorithm.

\paragraph{Naive Local Training}
This algorithm is described in \Cref{alg:local}. We hyperparameter tune over the step size $\alpha$ [0.0001, 0.001, 0.01, 0.1, 1, 10]. 

\paragraph{FedAvg}
This algorithm is described in \Cref{alg:fedavg}. We hyperparameter tune over the step size $\eta$ [0.0001, 0.001, 0.01, 0.1, 1, 10].

\paragraph{\ftfa} This algorithm is described in \Cref{alg:ftfa}. We hyperparameter tune over the step size of FedAvg $\eta$ [0.0001, 0.001, 0.01, 0.1, 1], and the step size of the personalization SGM steps $\alpha$ [0.0001, 0.001, 0.01, 0.1, 1].

\paragraph{\rtfa} This algorithm is described in \Cref{alg:rtfa}. We hyperparameter tune over the step size of FedAvg $\eta$ [0.0001, 0.001, 0.01, 0.1, 1], the step size of the personalization SGM steps $\alpha$ [0.0001, 0.001, 0.01, 0.1, 1], and the ridge parameter $\lambda$ [0.001, 0.01, 0.1, 1, 10].

\paragraph{\mamlfl-HF} This is the hessian free version of the algorithm, i.e., the hessian term is approximated via finite differences (details can be found in \cite{FallahMoOz20}). This algorithm is described in \Cref{alg:maml-hf-app}. We hyperparameter tune over the step size $\eta$ [0.0001, 0.001, 0.01, 0.1, 1], the step size of the personalization SGM steps $\alpha$ [0.0001, 0.001, 0.01, 0.1, 1], and the hessian finite-difference-approximation parameter $\delta$ [0.001, 0.00001]. We used only two different values of  $\delta$ because the results of preliminary experiments suggested little change in accuracy with changing $\delta$.

\paragraph{\mamlfl-FO} This is the first order version of the algorithm, i.e., the hessian term is set to 0 (details can be found in \cite{FallahMoOz20}). This algorithm is described in \Cref{alg:maml-hf-app}. We hyperparameter tune over the step size $\eta$ [0.0001, 0.001, 0.01, 0.1, 1], the step size of the personalization SGM steps $\alpha$ [0.0001, 0.001, 0.01, 0.1, 1].

\paragraph{\pfedme} This algorithm is described in \Cref{alg:pfedme}. We hyperparameter tune over the step size $\eta$ [0.0005, 0.005, 0.05], and the weight $\beta$ [1, 2]. The proximal optimization step size, hyperparameter $K$, and prox-regularizer $\lambda$ associated with approximately solving the prox problem is set to 0.05, 5, and 15 respectively.
We chose these hyperparameters based on the suggestions from \cite{DinhTrNg20}.
We were unable to hyperparameter tune \pfedme\ as much as, for example, \rtfa\ because each run of \pfedme\ takes significantly longer to run. Additionally, for this same reason, we were unable to run \pfedme on the Stack Overflow dataset. While we do not have wall clock comparisons (due to multiple jobs running on the same gpu), we have observed that pFedMe, with the hyperparameters we specified, takes approximately 20x the compute time to complete relative to \ftfa, \rtfa, and \mamlfl-FO.

The ideal hyperparameters for each dataset can be found in the tables below:
\begin{table}[H]
\begin{center}
\begin{tabular}{ |c|c|c|c|c|c|c| } 
\hline
 Algorithm & $\eta$ & $\alpha$ & $\lambda$ & $\delta$ & $\beta$ \\
 \hline
 Naive Local & - & 0.1 & - & -  & - \\ 
 FedAvg & 0.1 & - & - & -  & -\\
 \ftfa & 1 & 0.1 & - & -   & -\\
 \rtfa & 1 & 0.1 & 0.1 & -  & - \\ 
 \mamlfl-HF & 1 & 0.1 & - & 0.00001 & -  \\ 
 \mamlfl-FO & 1 & 0.1 & - & -  & - \\ 
 \pfedme & 0.05 & - & - & -  & 2  \\ 
 \hline
\end{tabular}
\end{center}
\caption{Shakespeare Best Hyperparameters}
\label{tbl:shakespeare-hyp}
\end{table}

\begin{table}[H]
\begin{center}
\begin{tabular}{ |c|c|c|c|c|c|c| } 
\hline
 Algorithm & $\eta$ & $\alpha$ & $\lambda$ & $\delta$ & $\beta$ \\
 \hline
 Naive Local & - & 0.1 & - & -  & - \\ 
 FedAvg & 0.01 & - & - & -  & -\\
 \ftfa & 0.001 & 0.1 & - & -   & -\\
 \rtfa & 0.001 & 0.1 & 0.1 & -  & - \\ 
 \mamlfl-HF & 0.001 & 0.01 & - & 0.001 & -  \\ 
 \mamlfl-FO & 0.001 & 0.01 & - & -  & - \\ 
 \pfedme & 0.05 & - & - & -  & 1  \\ 
 \hline
\end{tabular}
\end{center}
\caption{CIFAR-100 Best Hyperparameters}
\label{tbl:cifar-hyp}
\end{table}

\begin{table}[H]
\begin{center}
\begin{tabular}{ |c|c|c|c|c|c|c| } 
\hline
 Algorithm & $\eta$ & $\alpha$ & $\lambda$ & $\delta$ & $\beta$ \\
 \hline
 Naive Local & - & 0.001 & - & -  & - \\ 
 FedAvg & 0.01 & - & - & -  & -\\
 \ftfa & 0.1 & 0.01 & - & -   & -\\
 \rtfa & 0.1 & 0.01 & 0.1 & -  & - \\ 
 \mamlfl-HF & 0.1 & 0.01 & - & 0.00001 & -  \\ 
 \mamlfl-FO & 0.1 & 0.01 & - & -  & - \\ 
 \pfedme & 0.05 & - & - & -  & 2 \\ 
 \hline
\end{tabular}
\end{center}
\caption{EMNIST Best Hyperparameters}
\label{tbl:emnist-hyp}
\end{table}

\begin{table}[H]
\begin{center}
\begin{tabular}{ |c|c|c|c|c|c|c| } 
\hline
 Algorithm & $\eta$ & $\alpha$ & $\lambda$ & $\delta$ & $\beta$ \\
 \hline
 Naive Local & - & 0.1 & - & -  & - \\ 
 FedAvg & 1 & - & - & -  & -\\
 \ftfa & 1 & 0.1 & - & -   & -\\
 \rtfa & 1 & 0.1 & 0.001 & -  & - \\ 
 \mamlfl-HF & 1 & 0.1 & - & 0.00001 & -  \\ 
 \mamlfl-FO & 1 & 0.1 & - & -  & - \\ 
 \hline
\end{tabular}
\end{center}
\caption{Stack Overflow Best Hyperparameters}
\label{tbl:stackoverflow-hyp}
\end{table}

\subsection{Additional Results}
In this section, we add additional plots from the experiments we conducted, which were omitted from the main paper due to length constraints. In essence, these plots only strengthen the claims made in the experiments section in the main body of the paper.

\begin{figure}[H]
\begin{minipage}[c]{0.5\linewidth}
\begin{overpic}[width=\linewidth]{
      		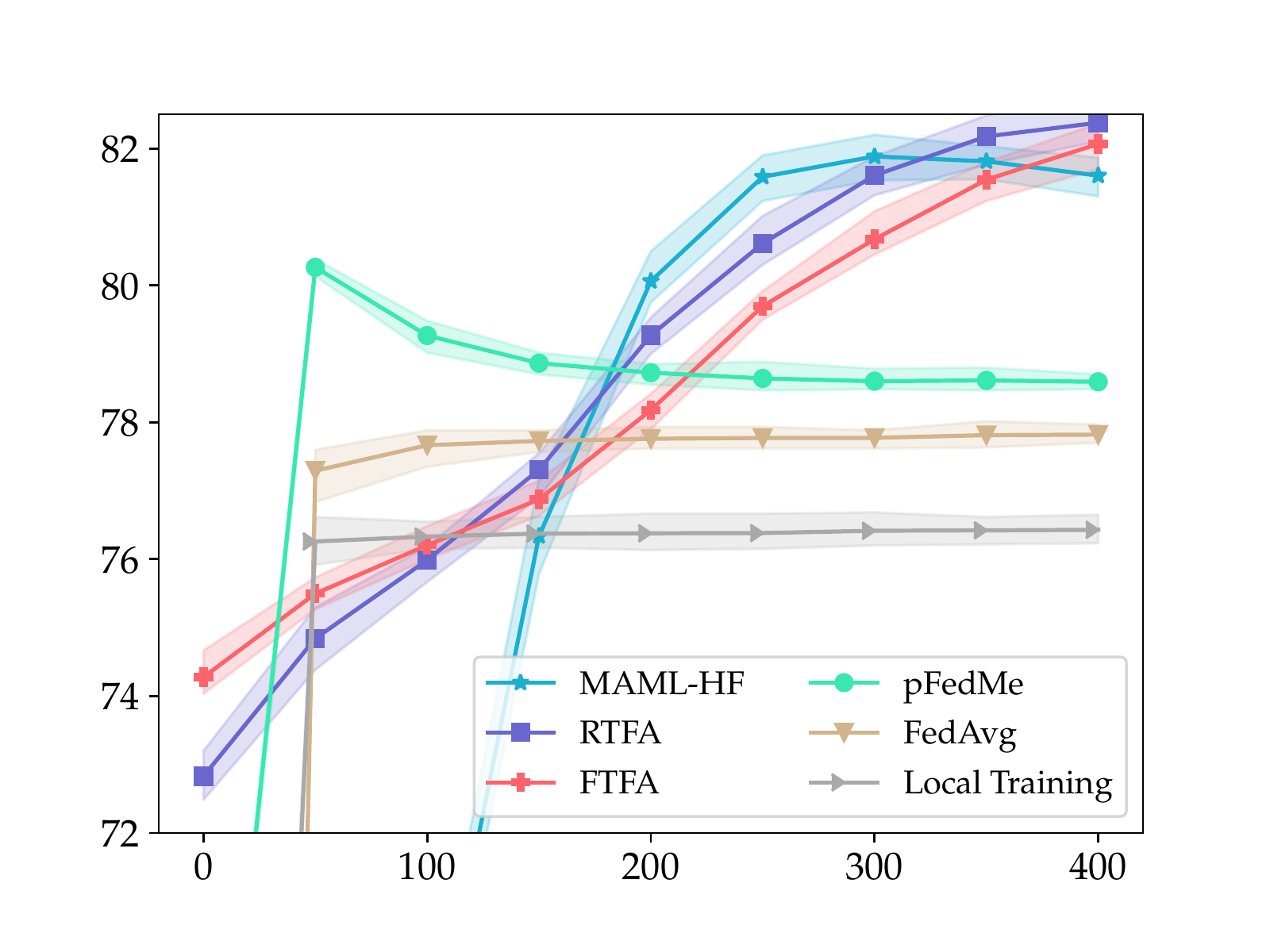}
      		\put(-4,28){
          \rotatebox{90}{{\small Test Accuracy}}}
        \put(30,0){{\small Communication Rounds}}
\end{overpic}\caption{CIFAR-100. Best-average-worst intervals created from different random seeds.}
\label{fig:cifar2}
\end{minipage}
\hfill
\begin{minipage}[c]{0.5\linewidth}
\begin{overpic}[width=\linewidth]{
      		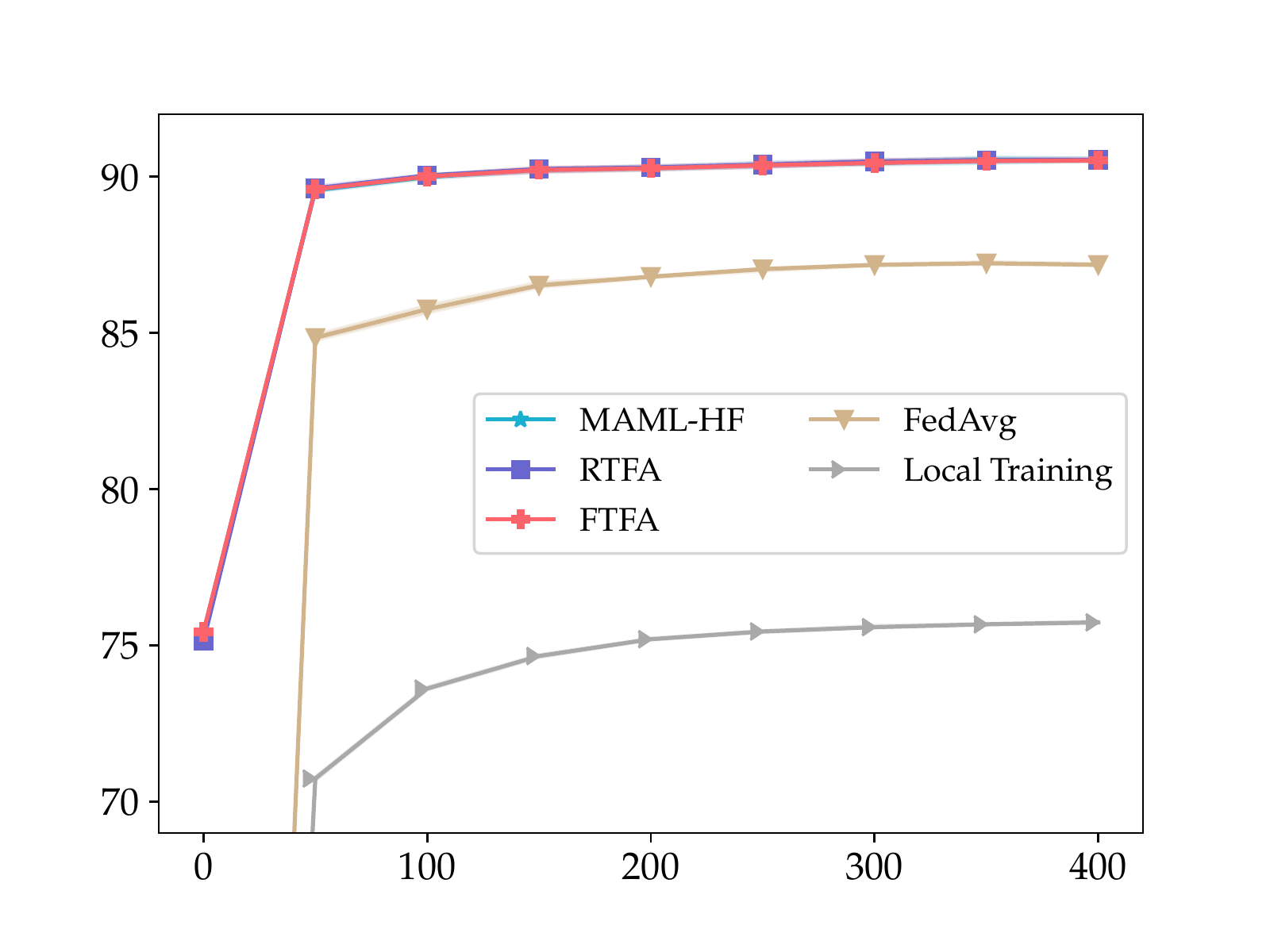}
      		\put(-4,28){
          \rotatebox{90}{{\small Test Accuracy}}}
        \put(30,0){{\small Communication Rounds}}
\end{overpic}\caption{EMNIST. Best-average-worst intervals created from different train-val splits.}
\label{fig:fedemnist2}
\end{minipage}%
\end{figure}

\begin{figure}[H]
\begin{center}
\begin{overpic}[width=0.5\linewidth]{
      		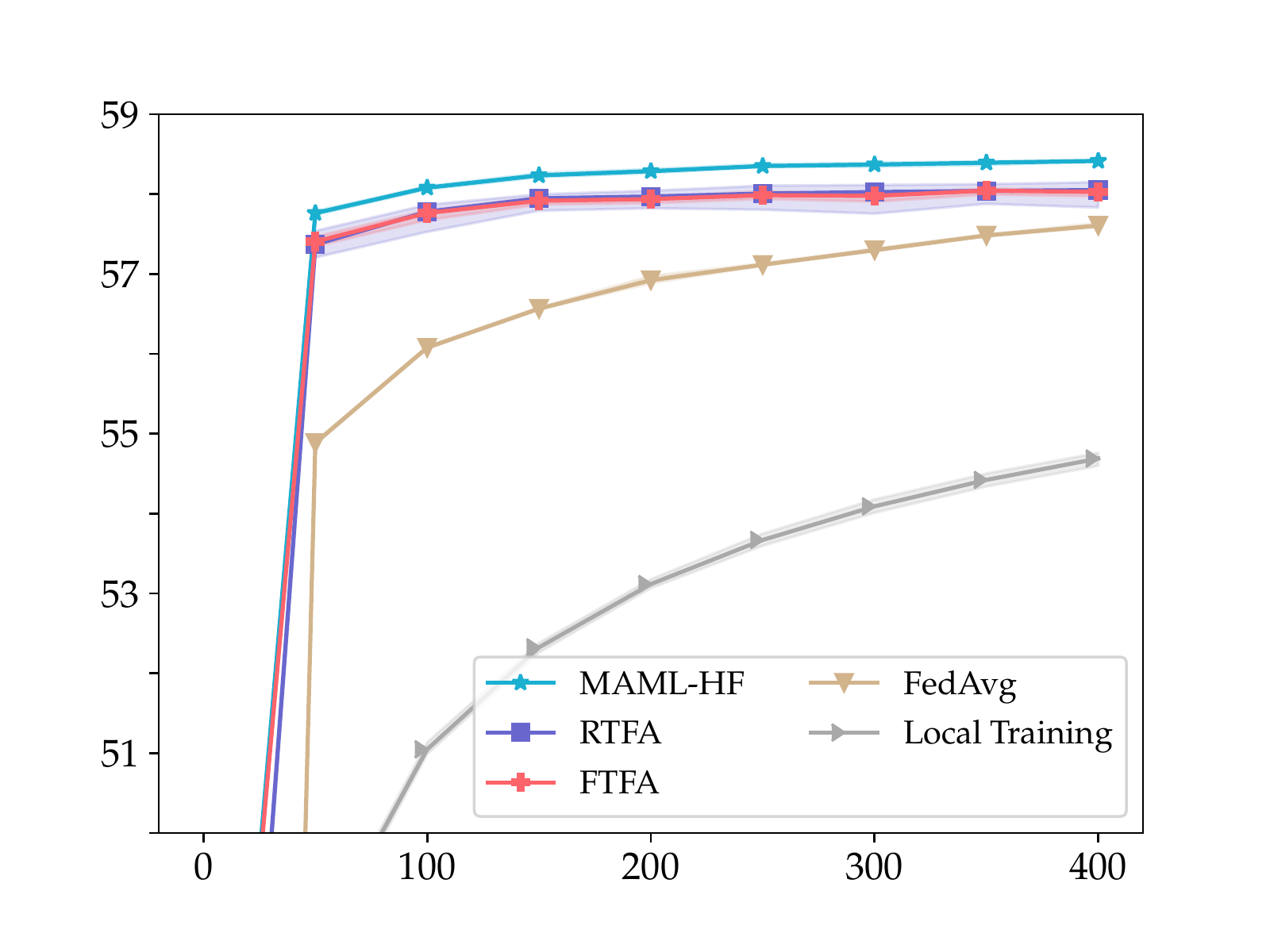}
      		\put(-4,28){
          \rotatebox{90}{{\small Test Accuracy}}}
        \put(30,0){{\small Communication Rounds}}
\end{overpic}\caption{Shakespeare. Best-average-worst intervals created from different random train-val splits.}
\label{fig:shakespeare2}
\end{center}
\end{figure}

\begin{figure}[H]
\begin{minipage}[c]{0.5\linewidth}
\begin{overpic}[width=\linewidth]{
      		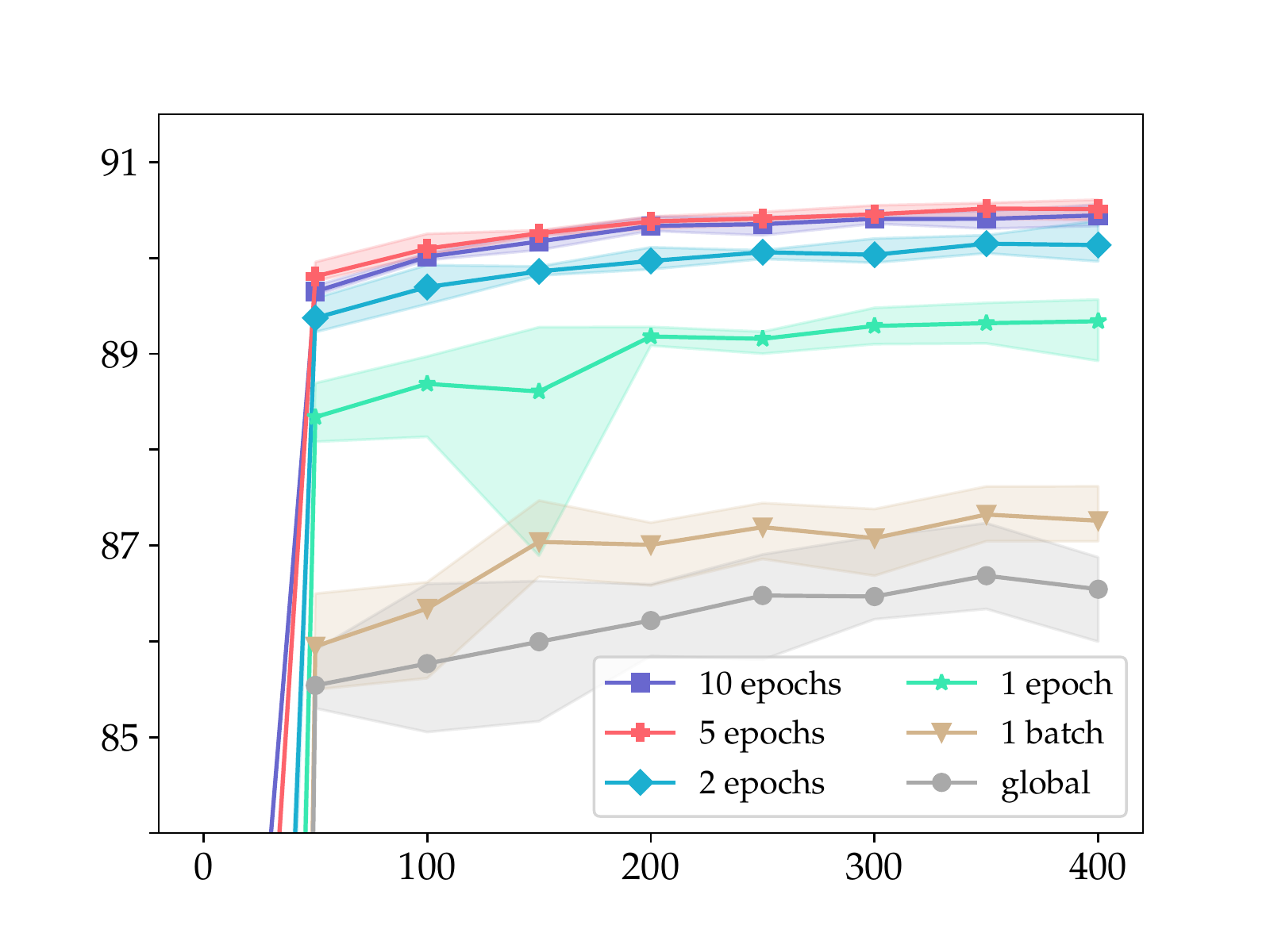}
      		\put(-4,28){
          \rotatebox{90}{{\small Accuracy}}}
        \put(30,0){{\small Communication Rounds}}
\end{overpic}\caption{EMNIST. Gains of personalization for FO-\mamlfl}
\label{fig:pers-femnist2}
\end{minipage}%
\hfill
\begin{minipage}[c]{0.5\linewidth}
\begin{overpic}[width=\linewidth]{
      		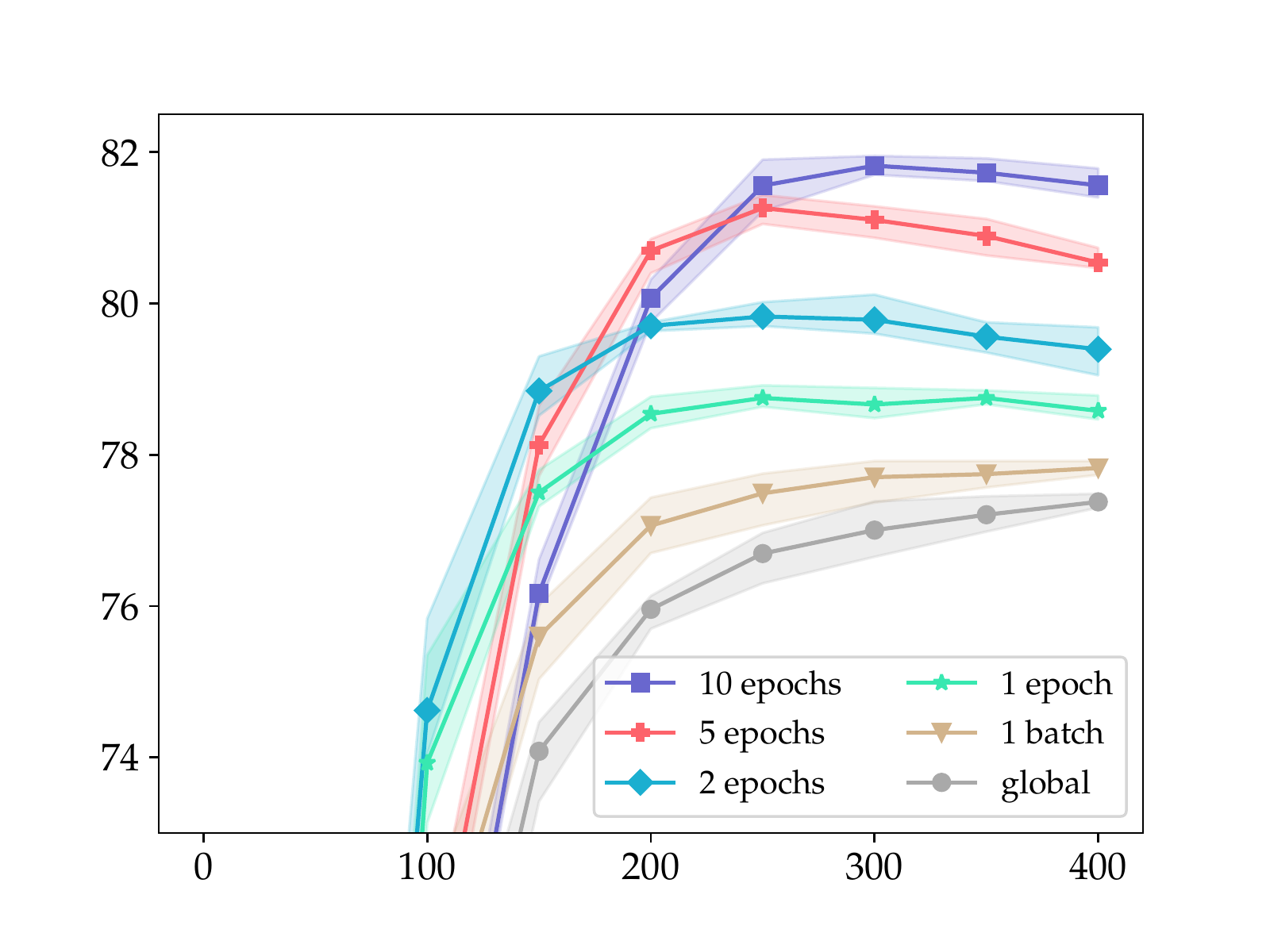}
      		\put(-4,28){
          \rotatebox{90}{{\small Accuracy}}}
        \put(30,0){{\small Communication Rounds}}
\end{overpic}\caption{CIFAR. Gains of personalization for FO-\mamlfl}
\label{fig:pers-cifar2}
\end{minipage}
\end{figure}

\begin{figure}[H]
\begin{minipage}[c]{0.5\linewidth}
\begin{overpic}[width=\linewidth]{
      		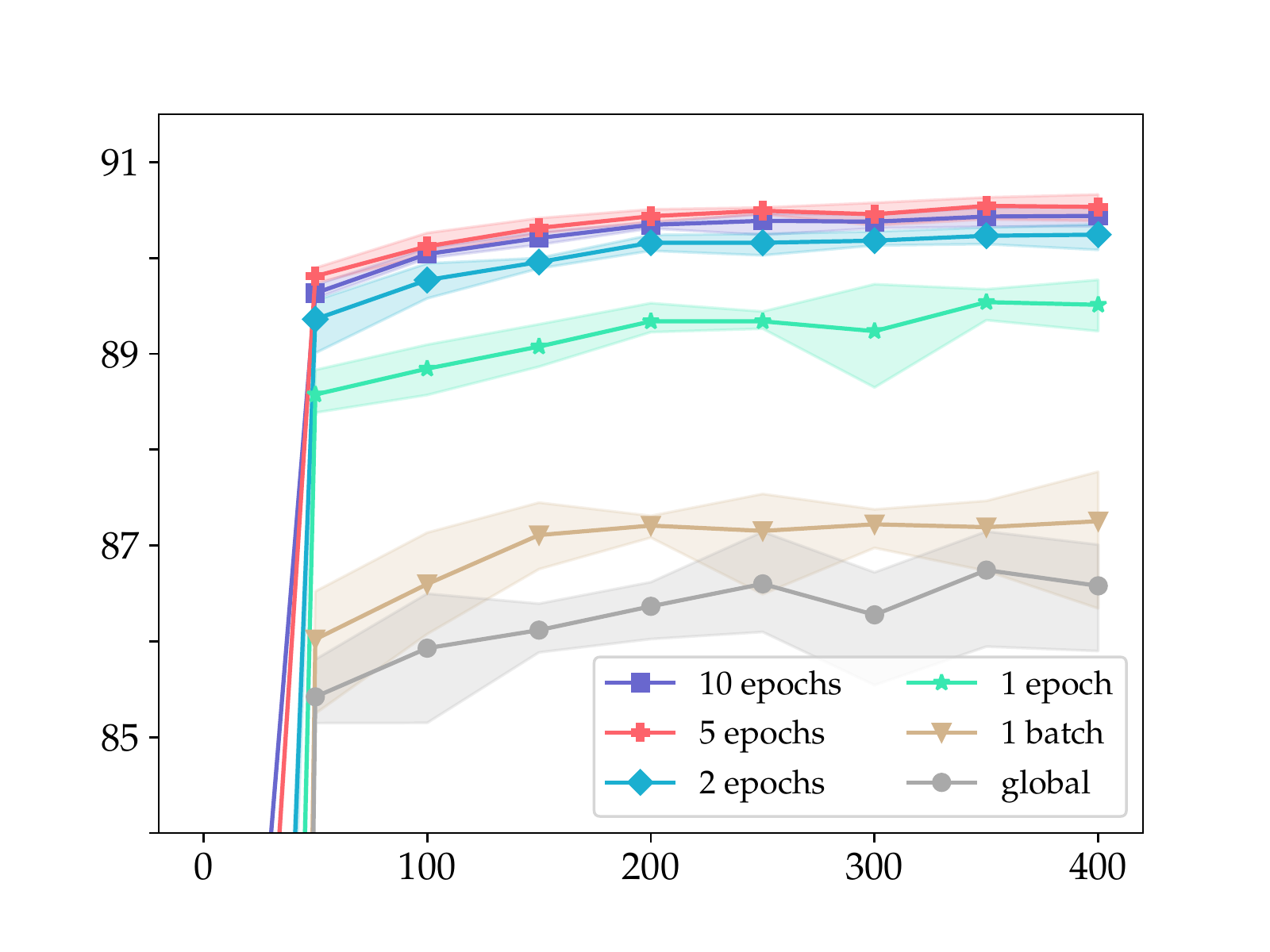}
      		\put(-4,28){
          \rotatebox{90}{{\small Accuracy}}}
        \put(30,0){{\small Communication Rounds}}
\end{overpic}\caption{EMNIST. Gains of personalization for HF-\mamlfl}
\label{fig:pers-femnist3}
\end{minipage}%
\hfill
\begin{minipage}[c]{0.5\linewidth}
\begin{overpic}[width=\linewidth]{
      		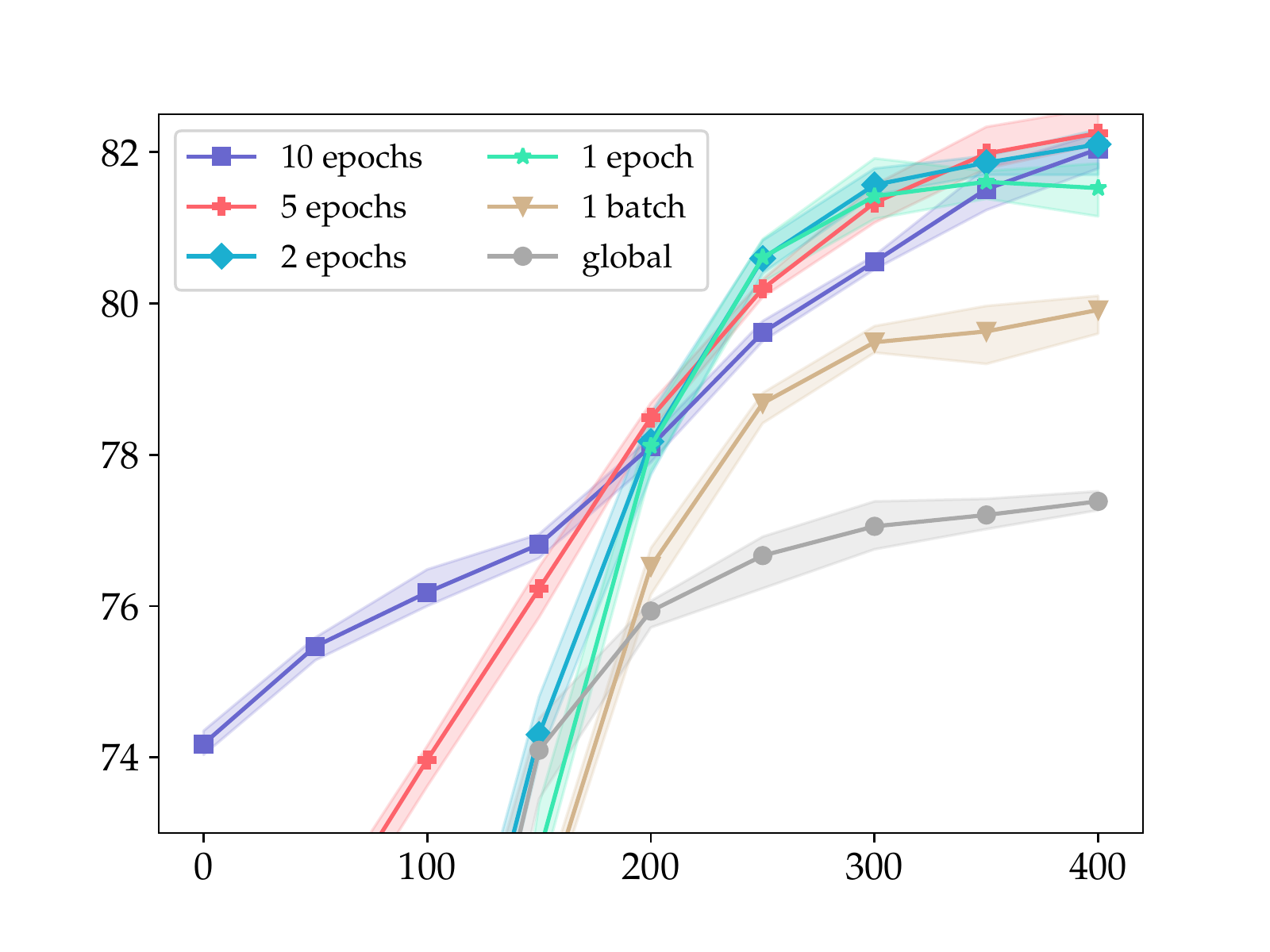}
      		\put(-4,28){
          \rotatebox{90}{{\small Accuracy}}}
        \put(30,0){{\small Communication Rounds}}
\end{overpic}\caption{CIFAR. Gains of personalization for \ftfa}
\label{fig:pers-cifar3}
\end{minipage}
\end{figure}

\end{document}